\def\year{2022}\relax
\documentclass[letterpaper]{article} 
\usepackage{aaai22}  
\usepackage{times}  
\usepackage{helvet} 
\usepackage{courier}  
\usepackage[hyphens]{url}  
\usepackage{graphicx} 
\usepackage{natbib}  
\usepackage{caption} 
\usepackage[utf8]{inputenc} 
\usepackage[T1]{fontenc}    
\usepackage{url}            
\usepackage{booktabs}       
\usepackage{nicefrac}       
\usepackage{microtype}      
\usepackage{latexsym,amsmath,amsfonts,amssymb}
\usepackage{stmaryrd}
\usepackage{mathabx}
\usepackage{graphicx}
\usepackage{algorithmic}
\usepackage{algorithm}
\usepackage{enumitem,nicefrac}
\usepackage{subcaption}
\usepackage{alphalph}
\usepackage{mathtools}
\usepackage{amsthm}
\usepackage{rotating}

\usepackage{longtable} 
\usepackage{pdflscape} 

\usepackage{xcolor}

\newcommand{\blinducb}{BlindUCB}
\newcommand{\biascorrecteducb}{BClinUCB}

\newcommand{\algorithms}{\mathcal{A}}
\newcommand{\instances}{\mathcal{I}}

\newcommand{\expectation}[1]{\mathbb{E}\left[#1\right]}
\newcommand*{\defeq}{\mathrel{\vcenter{\baselineskip0.5ex \lineskiplimit0pt
			\hbox{\footnotesize.}\hbox{\footnotesize.}}}%
	=}
\renewcommand{\vec}[1]{\boldsymbol{#1}}
\newcommand{\R}{\mathbb{R}}

\nocopyright

\newtheorem{Thm}{Theorem}

\newtheorem{Lem}[Thm]{Lemma}

\usepackage{ifthen}
\newboolean{blind}
\setboolean{blind}{false}

\title{Machine Learning for Online Algorithm Selection under Censored Feedback}
\ifthenelse{\boolean{blind}}
{
\author{
    Anonymous authors \textsuperscript{\rm 1}\\
}
\affiliations{
    \textsuperscript{\rm 1}Anonymous institution\\
    anon@ymo.us

}}
{
\author{
    Alexander Tornede, \textsuperscript{\rm 1}
    Viktor Bengs, \textsuperscript{\rm 1}
    Eyke Hüllermeier \textsuperscript{\rm 2} \\
}
\affiliations{
    \textsuperscript{\rm 1}Department of Computer Science, Paderborn University\\
    \textsuperscript{\rm 2}Institute for Informatics, LMU Munich\\
    alexander.tornede@upb.de, viktor.bengs@upb.de, eyke@ifi.lmu.de

}
}

\begin{document}

\maketitle

\begin{abstract}
In online algorithm selection (OAS), instances of an algorithmic problem class are presented to an agent one after another, and the agent has to quickly select a presumably best algorithm from a fixed set of candidate algorithms. For decision problems such as satisfiability (SAT), quality typically refers to the algorithm's runtime. As the latter is known to exhibit a heavy-tail distribution, an algorithm is normally stopped when exceeding a predefined upper time limit. As a consequence, machine learning methods used to optimize an algorithm selection strategy in a data-driven manner need to deal with right-censored samples, a problem that has received little attention in the literature so far. 
In this work, we revisit multi-armed bandit algorithms for OAS and discuss their capability of dealing with the problem. Moreover, we adapt them towards runtime-oriented losses, allowing for partially censored data while keeping a space- and time-complexity independent of the time horizon. In an extensive experimental evaluation on an adapted version of the ASlib benchmark, we demonstrate that theoretically well-founded methods based on Thompson sampling perform specifically strong and improve in comparison to existing methods.
\end{abstract}

\section{Introduction}
Algorithm selection (AS) considers the automatic selection of an algorithm most suitable for solving an instance of an algorithmic problem class, such as the boolean satisfiability problem (SAT) or the traveling salesperson problem (TSP). Suitability is often quantified in terms of a loss function (or performance measure), such as solution quality or the algorithm's runtime\,---\,we shall focus on the latter in the remainder of this work. AS is motivated by the phenomenon of performance complementarity, roughly stating that the best algorithm within a pool of candidate algorithms will typically vary from instance to instance \cite{no_free_lunch_wolpert1997no,kerschke2019automated}.

Considering AS in an online setting, i.e., \textit{online} algorithm selection (OAS), the problem can be posed as an iterative decision making problem, where instances arrive and decisions have to be made in an online manner. This setting suggests modeling the task as a contextual multi-armed bandit problem \cite{chuLRS11_linucb}. In contrast to standard AS, one usually does not assume an upfront training dataset, consisting of loss function measurements for some of the algorithms on some training instances. Instead, in the online setting, the algorithm is provided feedback in the form of an evaluation of the loss function for the selected algorithm on the current instance, after the selection has been made. 

In practice, algorithm runtime distributions often exhibit a heavy-tail property \cite{heavy-tailed_distributions_in_combinatorial_search_gomesSC97}, meaning that some algorithms can take extremely long to terminate on some instances. To avoid an online AS system from stalling, it is common to run algorithms with a timeout after which the algorithm is forcefully terminated if it did not solve the instance before. Correspondingly, one may end up with unsolved instances and \textit{right-censored} runtimes \cite{survial_analysis_kleinbaum2010}: although the true runtime is not known precisely, it is known to exceed the timeout. Needless to say, learning algorithms should try to exploit such censored data, which nevertheless provide a kind of ``weak supervision''. Although different forms of the OAS problem has been studied quite intensively in the literature (cf.\ Section \ref{sec:related_work}), censoring has hardly been considered so far. 

To alleviate this situation, we revisit well-known linear contextual bandit algorithms for the OAS problem under censored data and and discuss their weaknesses. Thereupon, we present theoretically grounded adaptations of these algorithms, incorporating improvements toward learning from censored data and employing selection criteria tailored toward runtime-based loss functions from the field of AS. In an extensive experimental study, we show that our approaches improve in terms of performance upon comparable existing approaches while featuring a time- and space-complexity independent of the time horizon.


\section{The Online Algorithm Selection Problem}
\label{sec:from_offline_to_online}
The online algorithm selection problem is an iterative decision problem, which comprises a problem instance space $\instances$ featuring instances of an algorithmic problem class, and a set of candidate algorithms $\algorithms$, which can solve such instances. The instances arrive iteratively over time, so that, at each time step $t$, an algorithm $s(h_t,i_t) = a_t \in \algorithms$ has to be chosen by the \emph{algorithm selector}
$s: \mathcal{H} \times \instances \rightarrow \algorithms$ decides about the algorithm used to solve the instance $i_t \in \instances$. Here, $h_t$ denotes the history of the selection process, consisting of triplets $\{(i_k, a_k, l_k)\}_{k=1}^{t-1}$ informing about the problem instances encountered so far, the algorithms that have been applied, and the corresponding evaluations in terms of losses $l_k = l(i_k,a_k)$ determined by a loss function $l: \instances \times \algorithms \rightarrow \mathbb{R}$. This process either continues ad infinitum (unlimited horizon) or ends at some final time step $T$ (final horizon). Consequently, the goal is to minimize the (average) loss achieved over the course of time, which, for the final horizon case, means to minimize $\mathcal{L}(s) = T^{-1} \sum\nolimits_{t=1}^T l(i_t, s(h_t, i_t))$.
In light of this, the optimal selector, called oracle or \emph{virtual best solver} (VBS), is defined as
\begin{equation}\label{eq:oracle}
    s^*(h_t,i_t) \defeq \arg\min\nolimits_{a \in \algorithms} \mathbb{E}[l(i_t,a) \, | \, h_t] \, ,
\end{equation} where the expectation accounts for possible randomness in the application of algorithm $a$.
To approximate the VBS, any intelligent selection strategy has to leverage the historical data and perform some kind of \textit{online learning}. This is a challenging task, especially as it involves the well-known exploration-exploitation dilemma of sequential decision making: The learner is constantly faced with the question of whether to acquire new information about the effectiveness of hitherto less explored algorithms, possibly improving but also taking the risk of large losses, or rather to exploit the current knowledge and choose algorithms that are likely to yield reasonably good results. 
%

One way of constructing an OAS $s$ is to learn, based on the history $h,$ a (cheap-to-evaluate) surrogate loss function $\widehat{l}_h: \instances \times \algorithms \rightarrow \mathbb{R}$ and invoke the principle of ``estimated'' loss minimization:
\begin{equation}\label{eq:shi}
s(h,i) \defeq \arg\min\nolimits_{a \in \algorithms} \widehat{l}_h(i,a) \, .
\end{equation}
For this purpose, we assume instances $i \in \instances$ to be representable in terms of $d$-dimensional feature vectors $f(i) \in \mathbb{R}^d$, where $f: \instances \rightarrow \mathbb{R}^d$ is a feature map. In the case of SAT, for example, features could be the number of clauses, the number of variables, etc.

Due to the online nature of the problem, it is desirable that OAS approaches have a time- and space-complexity independent of the time-horizon. In particular, memorizing all instances (i.e., entire histories $h_t$) and constantly retraining in batch mode is no option. Moreover, from a practical point of view, decisions should be taken quickly to avoid stalling an AS system.

\subsection{Censored Runtimes}\label{sec:censored_data}
A loss function of specific practical relevance is the runtime of an algorithm, i.e., the time until a solution to a problem instance is found. However, in domains like combinatorial optimization, runtimes may exhibit a heavy-tail distribution, i.e., some algorithms may run unacceptably long on some instances \cite{heavy-tailed_distributions_in_combinatorial_search_gomesSC97}. This is why algorithms are usually executed under time constraints in the form of an upper bound on the runtime (a ``cutoff'') $C \in \mathbb{R}$. If an algorithm does not terminate until $C$, it is forcefully terminated, so as to avoid blocking the system. The instance is then treated as unsolved. 
%
%
%

To account for unsolved instances, the loss is augmented by a penalty function $\mathcal{P}: \mathbb{R} \rightarrow \mathbb{R}$: 
\begin{equation}\label{eq:loss_at_t}
    l(i,a) = m(i,a) 1_{ \llbracket m(i,a) \leq C \rrbracket} +  \mathcal{P}(C) 1_{\llbracket m(i,a) > C \rrbracket},
\end{equation} where $1_{\llbracket \cdot \rrbracket}$ is the indicator function and $m: \mathcal{I} \times \mathcal{A} \rightarrow \mathbb{R}$ returns the true (stochastic) runtime of an algorithm $a$ on an instance $i$.
Formally, when selecting algorithm $a$, either the runtime $m(i_t,a)$ is observed, if $m(i_t,a) \leq C,$ or a penalty $\mathcal{P}(C)$ due to a \textit{right-censored} sample \cite{survial_analysis_kleinbaum2010}, i.e.\ $m(i_t,a) > C$, while the true runtime $m(i_t,a)$ remains unknown. With $\mathcal{P}(C) = 10C$, (\ref{eq:loss_at_t}) yields the well-known \textit{PAR10} score.

Previous work has shown that treating such right-censored data correctly is important in the context of standard (offline) AS \cite{xu2007satzilla,tornedeWWMH20_run2survive,hanselle2021algorithm} and algorithm configuration (AC) \cite{bo_with_censored_data_hutterHL11,benchmarking_algorithm_configurators_eggenspergerLHH18,neural_model_based_opt_with_censored_obs_eggenspergerEHMLH20}. 
In the online setting, this might be even more critical, because the data does not only influences the learning but also the (active) sampling strategy of the AS.

The simplest approach for dealing with censored samples is to ignore them all together, which, however, causes an undesirable loss of information. Another simple strategy is imputation. For example, in the case of the \textit{PAR10} score, censored samples are commonly replaced by the cutoff time $C$ or its multiplicity $10\, C$. Obviously, such imputations of constant values may produce a strong bias. For example, imputation of $C$ leads to a systematic underestimation of the true runtime, and so does the ignorance of the censored (and hence long) runtimes \cite{tornedeWWMH20_run2survive}. 

A more advanced technique for imputation of right-censored data developed by \citet{regression_with_censored_data_schmee1979simple} leverages sampling from a truncated normal distribution, which has been frequently used in the context of AS and AC in the past (e.g. \cite{xu2007satzilla, benchmarking_algorithm_configurators_eggenspergerLHH18}), but not necessarily improves upon the na\"ive imputation schemes previously discussed \cite{tornedeWWMH20_run2survive}.

Although recent work \cite{tornedeWWMH20_run2survive} has shown that classical parameter-free survival analysis methods can perform very well in the offline AS setting, these methods cannot be easily transformed into online variants. For example, the well-known Cox proportional hazard model 
\cite{david1972regression} relies on estimating the baseline survival function through algorithms like the Breslow estimator  (in its parameter-free version), which inherently requires the storage of all data in the form of so-called risk-sets \cite{breslow1972contribution}. In principle, approximation techniques from the field of learning on data streams are conceivable \cite{shaker2014survival}. Yet, in this work, we will focus on veritable online algorithms that do not require any approximation.

\subsection{OAS as a Bandit Problem}
OAS can be cast as a contextual multi-armed bandit (MAB) problem comprising a set of arms/algorithms $\algorithms$ to choose from. In each round $t$, the learner is presented a context, i.e., an instance $i_t \in \mathcal{I}$ and its features $f(i_t) \in \mathbb{R}^d$, and is requested to select one of the algorithms for this instance, i.e., pull one of the arms, which will be denoted by $a_t$. As a result, the learner will suffer a loss as defined in \eqref{eq:loss_at_t}.
%

In the stochastic variant of the contextual MAB problem, the losses are generated at random according to underlying distributions, one for each arm, which are unknown to the learner.
Thus, the expected loss $\expectation{l(i_t,a_t) \vert f(i_t)}$ suffered at time step $t$ is taken with respect to the unknown distribution of the chosen algorithm's runtime $m(i_t,a_t)$ and depends on the instance (features) $f(i_t).$ 
Ideally, the learner should pick in each time step $t$ an arm having the smallest expected loss for the current problem instance.
Formally, 
%
\begin{equation} \label{defi_optimal_arm}
    a_t^* \in \arg\min\nolimits_{a \in \algorithms} \expectation{l(i_t,a_t) \vert f(i_t)},
\end{equation} 
suggesting the optimal strategy to be $s^*_t(h_t,i_t) = a_t^*.$
Needless to say, this optimal strategy can only serve as a benchmark, since the runtime distributions are unknown.
Nevertheless, having an appropriate model or estimate for the expected losses, one could mimic the choice mechanism in \eqref{defi_optimal_arm}, giving rise to a suitable online algorithm selector (\ref{eq:shi}).
%
%

For convenience, we 
shall write from now on $\vec{f}_i,$ $l_{t,a}$ or $m_{i,a}$ instead of $f(i),$ $l(i_t,a)$ or $m(i,a)$ for any $i,i_t\in \mathcal{I},$ $a \in \algorithms.$
Moreover, we write $\|\vec{x}\|_A \defeq \sqrt{\vec{x}^\intercal A^{-1} \vec{x}}$ for any $\vec{x}\in \R^d$ and semi-positive definite matrix $A\in \R^{d\times d}$, and $\|\vec{x}\| \defeq \sqrt{\vec{x}^\intercal \vec{x}}.$
In Section \ref{sec_appendix_symbols} of the appendix, we provide a list of frequently used symbols for quick reference.

\section{Modeling Runtimes} \label{sec_runtime_model}
As hinted at earlier, online algorithm selectors based on a bandit approach can be reasonably designed through the estimation of the expected losses occurring in \eqref{defi_optimal_arm}.
%
To this end, we make the following assumption regarding the runtime of an algorithm $a \in \mathcal{A}$ on problem instance $i \in \mathcal{I}$ with features $\vec{f}_i \in \mathbb{R}^d:$ 
\begin{equation} \label{defi_exp_runtime_model}
    m_{i,a} = \exp(\vec{f}_i^\intercal\vec{\theta}^*_a ) \exp(  \epsilon_{i,a}),
\end{equation} 
where $\vec{\theta}^*_a \in \mathbb{R}^d$ is some unknown weight vector for each algorithm $a \in \mathcal{A}$, and $ \epsilon_{i,a}$ is a stochastic noise variable with zero mean.
The motivation for \eqref{defi_exp_runtime_model} is twofold. First, theoretical properties such as positivity  of the runtimes and heavy-tail properties of their distribution (by appropriate choice of the noise variables) are ensured.   
Second, we obtain a convenient linear model for the logarithmic runtime $y_{i,a}$ of an algorithm $a$ on instance $i,$ namely
\begin{equation} \label{defi_runtime_model}
    y_{i,a} = \log(m_{i,a}) =  \vec{f}_i^\intercal\vec{\theta}^*_a  +  \epsilon_{i,a} \, .
\end{equation} 
It is important to realize the two main implications coming with those assumption. First, up to the stochastic noise term, the (logarithmic) runtime of an algorithm depends linearly on the features of the corresponding instance. This is not a big restriction, because the feature map $\vec{f}_i$ may include nonlinear transformations of ``basic'' features and play the role of a \emph{linearization} \cite{scho_lw}\,---\,the practical usefulness of non-linear models has recently been shown ,for example, by \citet{tornedeWWMH20_run2survive}. Moreover, previous work on AS has also considered logarithmic runtimes for model estimation \cite{xu2007satzilla}.
Second, the model (\ref{defi_runtime_model}) suggests to estimate $\vec{\theta}_a^*$ separately for each algorithm, which is common practice but excludes the possibility to exploit (certainly existing) similarities between the algorithms. 
In principle, it might hence be advantageous to estimate joint models  \cite{tornede2019algorithm,tornedeWH20_extreme}.

Additionally, we assume that (i) $\exp(f_i^\intercal\vec{\theta}^*_a) \leq C$ for any $a \in \mathcal{A},$  $i \in \mathcal{I}$ and (ii) $\epsilon_{i,a}$ is normally distributed with zero mean and standard deviation $\sigma>0$. While the first assumption is merely used for technical reasons, namely to derive theoretically valid confidence bounds for estimates of the weight vectors, the second assumption implies that $\exp(\epsilon_{i,a})$ is log-normal, which is a heavy-tail distribution yielding a heavy-tail distribution for the complete runtime, thus adhering to practical observations discussed earlier.

\section{Stochastic Linear Bandits Approaches} \label{sec_lin_bandits}
%
As \eqref{defi_runtime_model} implies $\expectation{\log(m_{i,a}) \vert \vec{f}_i} = \vec{f}_i^\intercal\vec{\theta}^*_a$,
it is tempting to apply a straightforward contextualized MAB learner designed for expected loss minimization,
 in which the expected losses are linear with respect to the context vector, viewing the logarithmic runtimes as the losses of the arms.
This special case of contextualized MABs, also known as the \emph{stochastic linear bandit} problem, has received much attention in the recent past (cf.\ Chap.\ 19 in \citet{lattimore2020bandit}). 
Generally speaking, such a learner tends to choose an arm having a low expected log-runtime for the given context (i.e., instance features), which in turn has a small expected loss.
A prominent learner in stochastic linear bandits is LinUCB \cite{chuLRS11_linucb}, a combination of online linear regression and the UCB \cite{auerCF02_ucb} algorithm. UCB implements the principle of optimism in the face of uncertainty and solves the exploration-exploitation trade-off by constructing confidence intervals around the estimated mean losses of each arm, and choosing the most optimistic arm according to the intervals.

Under the runtime assumption \eqref{defi_runtime_model}, the basic LinUCB variant (which we call BLindUCB) disregards censored observations in the OAS setting, and therefore considers the ridge-regression (RR) estimator for each algorithm $a \in \mathcal{A}$ only on the non-censored observations.
Formally, in each time step $t$, the RR estimate $\widehat{\vec{\theta}}_{t,a}$ for the weight parameter $\vec{\theta}_a^*$ is 
\begin{equation} \label{defi_ols_na\"ive}
     \arg\min\limits_{\vec{\theta} \in \mathbb{R}^d} \sum\nolimits_{j=1}^t 1_{\llbracket a_j = a, m_{i_j,a} \leq C \rrbracket} \big( \vec{f}_{i_j}^\intercal\vec{\theta} - y_{i_j,a}\big)^2 + \lambda\|\vec{\theta}\|^2 \, ,
\end{equation} 
where  $\lambda \geq 0$ is a regularization parameter.
The resulting selection strategy for choosing algorithm $a_t$ at time $t$ is then
\begin{equation} \label{eq_choice_linucb}
    a_t = \arg\min\limits_{a \in \mathcal{A}} \ \vec{f}_{i_t}^\intercal \widehat{\vec{\theta}}_{t,a} - \alpha \cdot w_{t,a}(\vec{f}_{i_t}) \, ,
\end{equation} 
where  $\alpha \in \mathbb{R}_{>0}$ is some parameter controlling the exploration-exploitation trade-off, and $w_{t,a}(\vec{f}_{i_t}) = \| \vec{f}_{i_t} \|_{A_{t,a}}$
%
%
the confidence width for the prediction of the (logarithmic) runtime of algorithm $a$ for problem instance $i_t$ based on the estimate \eqref{defi_ols_na\"ive}.
%
Here, $A_{t,a}= \lambda I_{d} + X_{t,a}^\intercal X_{t,a}  $ is the (regularized) Gram matrix, with $I_{d}$ the $d\times d$ identity and $X_{t,a}$ denoting the design matrix at time step $t$ associated with algorithm $a,$ i.e., the matrix that stacks all the features row by row whenever $a$ has been chosen.

%
The great appeal of this approach is the existence of a closed form expression of the RR estimate, which can be updated sequentially with time- and space-complexity depending only on the feature dimension but independent of the time horizon \cite{lattimore2020bandit}.
%
%
However, as already mentioned in Section \ref{sec:censored_data}, disregarding the right-censoring of the data often yields a rather poor performance of a regression-based learner in offline AS problems, so it might be advantageous to adapt this method to that end.

%

\subsection{Imputation-based Upper Confidence Bounds} \label{sec_bclinUCB}
A simple approach to include right-censored data into \blinducb{} is to impute the corresponding samples by the cut-off time $C$ as discussed in Sec.\ \ref{sec:censored_data}, giving rise to the RR estimate
%
\begin{equation} \label{defi_ols_imputed}
    \widehat{\theta}_{t,a} = \arg\min\limits_{\vec{\theta} \in \mathbb{R}^d} \sum\nolimits_{j=1}^t 1_{\llbracket a_j = a \rrbracket} \big( f_{i_j}^\intercal\vec{\theta} - \widetilde{y}_{i_j,a} \big)^2 + \lambda\|\vec{\theta}\|^2 \, ,
\end{equation}  
where $\widetilde{y}_{i_j,a} = \min(y_{i_j,a}, \log(C))$ is the possibly imputed logarithmic runtime.

When considering censoring, the least-squares formulation in \eqref{defi_ols_imputed} appears to have an important disadvantage. Those weight vectors producing overestimates of $C$ in case of a timeout are penalized (quadratically) for predictions $C < \widehat{y} < m(i,a)$, although these predictions are actually closer to the unknown ground truth than $C$. In fact, one can verify this intuition theoretically by showing that, for $\lambda=0$, the RR estimate $\widehat{\theta}_{t,a}$ is downward biased in the case of censored samples (cf.\  \citet{greene2005censored}).

Although this imputation strategy has been shown to work astonishingly well in practice in offline AS \cite{tornedeWWMH20_run2survive}, the bias in the RR estimates requires a bias-correction in the confidence bounds of BLindUCB to ensure that the estimate falls indeed into the bounds with a certain probability. 
The corresponding bias-corrected confidence bound widths are $w_{t,a}^{(bc)}(\vec{f}_{i_t})  = (1 + 2 \log(C) \, (N^{(C)}_{a,t})^{1/2} ) w_{t,a}(\vec{f}_{i_t}),$ where $N^{(C)}_{a,t}$ is the amount of timeouts of algorithm $a$ until $t$ (cf.\ Section \ref{sec_appendix_bc_cb} of the appendix). 
The resulting LinUCB variant, which we call BClinUCB (bias-corrected LinUCB), employs the same action rule as in \eqref{eq_choice_linucb}, but uses $w_{t,a}^{(bc)}$ instead of $w_{t,a}$ and the estimate in \eqref{defi_ols_imputed}.
%

Unfortunately,  these bias-corrected confidence bounds reveal a decisive disadvantage in practice, namely, the confidence bound of an algorithm $a \in \mathcal{A}$ is usually much larger than the actually estimated (log-)runtime $\vec{f}_{i_t}^\intercal\widehat{\vec{\theta}}_{t,a}$ for instance $i_t.$
Therefore, the UCB value of $a$\,---\,let us call it $\delta_{t,a} = \vec{f}_{i_t}^\intercal\widehat{\vec{\theta}}_{t,a} - w^{(bc)}_{t,a}(\vec{f}_{i_t})$\,---\,is such that $\delta_{t,a} < 0$ if the algorithm has experienced at least one timeout. This prefers algorithms that experienced a timeout over those that did not.
This in turn explains the poor performance of the \biascorrecteducb{} strategies in the evaluation in Section \ref{sec:experimental_evaluation}.

\subsection{Randomization of Upper Confidence Bounds} \label{subsec_rand_ucb}

One way of mitigating the problem of the bias-corrected confidence bounds is to leverage a generalized form of UCB, called randomized UCB (RandUCB) \cite{vaswaniMDK20_randUCB}, where the idea is to multiply the bias-corrected bounds $w^{(bc)}_{t,a}(\vec{f}_{i_t})$ with a random realization of a specific distribution having positive support. RandUCB can be thought of as a mix of the classical UCB strategy, where the exploration-exploitation trade-off is tackled via the confidence bounds, and Thompson sampling \cite{thompson1933likelihood,russo2017tutorial}, which leverages randomization in a clever way for the same purpose (see next section). To this end, we define randomized confidence widths $\widetilde{w}_{t,a}(\vec{f}_{i_t}) = w_{t,a}^{(bc)}(\vec{f}_{i_t}) \cdot r$, where $r \in \mathbb{R}$ is sampled from a half-normal distribution with $0$ mean and standard deviation $\widetilde{\sigma}^2$. This ensures that $r \geq 0$ and that the confidence widths do indeed shrink when the distribution is properly parametrized. 
Although this improves the performance of LinUCB as we will see later, the improvement is not significant enough to achieve competitive results.

All variants of LinUCB for OAS  introduced so far can be jointly defined as in Alg. \ref{alg:lin_ucb} in Section \ref{sec_appendix_pseudo_codes} of the appendix.

\subsection{Bayesian Approach: Thompson Sampling} \label{sec_TS}

As the confidence bounds used by LinUCB seem to be a problem in practice, one may think of Thompson sampling (TS) as an interesting alternative.
The idea of TS is to assume a prior loss distribution for every arm, and in each time step, select an arm (i.e. algorithm) according to its probability of being optimal, i.e., according to its posterior loss distribution conditioned on all of the data seen so far. In particular, this strategy solves the exploration-exploitation trade-off through randomization driven by the posterior loss distribution. 

More specifically, let the (multivariate) Gaussian distribution with mean vector $ \boldsymbol{\mu} \in \R^d$ and covariance matrix $\Sigma\in \R^{d\times d}$ be denoted by $\mathrm{N} \big( \boldsymbol{\mu},  \Sigma \big)$. Similarly, the cumulative distribution function of a (univariate) Gaussian distribution with mean $\mu\in \R$ and variance $\sigma^2$ at some point $z\in \R$ is denoted by $\Phi_{\mu,\sigma^2}(z)$. 
A popular instantiation of TS for stochastic linear bandits \cite{agrawal2013thompson} assumes a  Gaussian prior distribution $\mathrm{N} \big( \widehat{\vec{\theta}}_{t,a} ,  \sigma  A_{t,a}^{-1} \big)$ for each weight vector of an algorithm $a$, where $\lambda,\sigma>0$ and  $\widehat{\vec{\theta}}_{t,a} $ denotes the RR estimate \eqref{defi_ols_imputed}.
This yields $\mathrm{N} \big( \widehat{\vec{\theta}}_{t+1,a},  \sigma  A_{t+1,a}^{-1} \big)$ as the posterior distribution at time step $t+1$. 
The choice mechanism is then defined by
\begin{equation} \label{eq_choice_ts}
    a_t = \arg\min\nolimits_{a \in \mathcal{A}} \vec{f}_{i_t}^\intercal \widetilde{\vec{\theta}}_a \, ,  
\end{equation} 
where $\widetilde{\vec{\theta}}_a \sim \mathrm{N} \big( \widehat{\vec{\theta}}_{t,a} ,  \sigma   A_{t,a}^{-1} \big)$ for each $a \in \mathcal{A}.$ 
Interestingly, as the experiments will show later on, this rather na\"ive version of Thompson sampling in the presence of censored data works astonishingly well in practice.

\section{Expected PAR10 Loss Minimization} \label{sec_exp_par_10_loss_min}
Due to the possibility of observing only a censored loss realization, i.e., $\mathcal{P}(C)$, it is reasonable to believe that a successful online algorithm selector needs to be able to properly incorporate the probability of observing such a realization into its selection mechanism. For this purpose, we derive the following decomposition of the expected loss  under the assumptions made in Section \ref{sec_runtime_model}  (details in Section \ref{sec_appendix_expec_loss} of the appendix):
\begin{align}
    \expectation{ l_{t,a} \vert  \vec{f}_{i_t} } &=  \big(1-\Phi_{\vec{f}_{i_t}^\intercal\vec{\theta}^*_a,\sigma^2}(\log(C)) \big) \cdot \left(\mathcal{P}(C) - E_{C}\right)  \notag
     \\
    &\quad +   E_{C}  \label{eq:expected_loss_3} 
\end{align} 
where  $C_{i_t,a}^{(1)} = \nicefrac{( \log(C)- \vec{f}_{i_t}^\intercal\vec{\theta}^*_a -\sigma^2 )}{\sigma} ,$  $C_{i_t,a}^{(2)} = C_{i_t,a}^{(1)} + \sigma $ and
$$ E_{C} = E_{C}(\vec{f}_{i_t}^\intercal\vec{\theta}^*_a,\sigma) = \exp( \vec{f}_{i_t}^\intercal\vec{\theta}^*_a + \nicefrac{\sigma^2}{2}) \cdot \frac{\Phi_{0,1}(C_{i_t,a}^{(1)} )}{\Phi_{0,1}( C_{i_t,a}^{(2)} )}$$  is the conditional expectation of a log-normal distribution with parameters $\vec{f}_{i_t}^\intercal\vec{\theta}^*_a$ and $\sigma^2$ under a cutoff $C.$
As such, the decomposition suggests that there are two core elements driving the expected loss of an algorithm $a$ conditioned on a problem instance $i_t$ with features $\vec{f}_{i_t}$: its expected log-runtime $\vec{f}_{i_t}^\intercal\vec{\theta}^*_a$ and its probability of running into a timeout, i.e., $\mathbb{P}(m(i_t,a) > C \vert \vec{f}_{i_t}) = \big(1-\Phi_{\vec{f}_{i_t}^\intercal\vec{\theta}^*_a,\sigma^2}(\log(C)) \big) $.

\subsection{LinUCB Revisited} \label{sec_lin_ucb_rev}
Having the refined expected loss representation in \eqref{eq:expected_loss_3}, one could simply plug-in the confidence bound estimates used by LinUCB for the log-runtime predictions to obtain a selection strategy following the optimism in the face of uncertainty principle, i.e., using an estimate of the target value to be minimized (here the expected loss in \eqref{eq:expected_loss_3}), which underestimates the target value with high probability.
Denote by $o_{t,a} = \vec{f}_{i_t}^\intercal \widehat{\vec{\theta}}_{t,a} - \alpha \cdot w_{t,a}(\vec{f}_{i_t})$ the optimistic  and by $p_{t,a} = \vec{f}_{i_t}^\intercal \widehat{\vec{\theta}}_{t,a} + \alpha \cdot w_{t,a}(\vec{f}_{i_t})$ the pessimistic estimate used by LinUCB (or its variants), where $ w_{t,a}$ is the confidence width of the corresponding LinUCB variant. 
With this, the selection strategy at time $t$ is to use $a \in \algorithms$ minimizing
\begin{align} \label{eq_choice_linucb_revisited_2}
%
         &\big(1-\Phi_{p_{t,a},\sigma}(\log(C)) \big)  
        \cdot \big(\mathcal{P}(C) - \hat E_C^{(1)} \big) 
        %
         + \hat E_C^{(2)}, 
%
\end{align} 
where 
\begin{align*}
    \hat E_C^{(1)} &=  \exp( p_{t,a} + \nicefrac{\sigma^2}{2})  \cdot \nicefrac{\Phi_{0,1}( \hat C_{i_t,a}^{(o)} )}{\Phi_{0,1}( \hat C_{i_t,a}^{(p)} + \sigma  )}, \\
    \hat E_C^{(2)} &= \exp( o_{t,a} + \nicefrac{\sigma^2}{2}) \cdot \nicefrac{\Phi_{0,1}( \hat C_{i_t,a}^{(p)} )}{\Phi_{0,1}( \hat C_{i_t,a}^{(o)} + \sigma )},
    %
    %
%
\end{align*}
and $\hat C_{i_t,a}^{(p)} = \nicefrac{( \log(C)- p_{t,a} -\sigma^2) }{\sigma},$ $\hat C_{i_t,a}^{(o)} = \nicefrac{( \log(C)- o_{t,a} -\sigma^2) }{\sigma}$. 
As $o_{t,a}$ ($p_{t,a}$) underestimates (overestimates)  $\vec{f}_{i_t}^\intercal {\vec{\theta}}_{a}^*$ it is easy to see that the terms in \eqref{eq_choice_linucb_revisited_2} are underestimating the corresponding terms occurring in \eqref{eq:expected_loss_3} with high probability, respectively.
%
%
However, as our experiments will reveal later on, the issues of the LinUCB-based algorithms caused either by the wide confidence bands or the biased RR estimate remain.


\subsection{Thompson Sampling revisited}\label{sec:thompson_rev}

Fortunately, the refined expected loss representation in \eqref{eq:expected_loss_3} can be exploited quite elegantly  by Thompson Sampling using Gaussian priors as in Section \ref{sec_TS}.
Our suggested instantiation of TS chooses algorithm $a_t \in \algorithms$ which minimizes
%
\begin{align} \label{eq_choice_tompcen_2}
%
%
   \big(1-\Phi_{\vec{f}_{i_t}^\intercal\widetilde{\vec{\theta}}_a,\tilde{\sigma}^2_{t,a}}(\log(C)) \big)  \big(\mathcal{P}(C) - \tilde E_{C}\big) 
    %
    + \tilde E_{C},
    %
%
\end{align} 
where $\widetilde{\vec{\theta}}_a $  is a random sample of the posterior $\mathrm{N} \big( \widehat{\vec{\theta}}_{t,a} ,  \sigma  A_{t,a}^{-1} \big),$ and $\tilde E_{C} = E_C(\vec{f}_{i_t}^\intercal\widetilde{\vec{\theta}}_a,\tilde{\sigma}_{t,a}),$ 
$\tilde{\sigma}^2_{t,a} = \sigma \| \vec{f}_{i_t} \|_{A_{t,a}}^2.$

Although the TS approach just presented does involve consideration of the timeout probability, it still suffers from the problem that the estimates for $\vec{\theta}^*_a$ are downward-biased as they are based on the RR estimate obtained from imputing censored samples with the cutoff time $C$. In the spirit of the Kaplan-Meier estimator \cite{kaplan1958nonparametric} from the field of survival analysis, \citet{buckley1979linear} suggested to augment censored samples by their expected value according to the current model and then solve the standard least-squares problem (for an overview of alternative approaches, we refer to \citet{miller1982regression}). This idea is particularly appealing, as it allows for an easy integration into online algorithms, due to its use of the least-squares estimator. Also, it has the potential to produce more accurate (i.e., less biased) estimates for $\vec{\theta}^*_a$. The integration is shown in lines 20--22 in Alg.\ \ref{alg:thompson} in Section \ref{sec_appendix_pseudo_codes} of the appendix together with a discussion of its hyperparameters.

\begin{figure*}[t]
  \begin{subfigure}[t]{0.33\textwidth}
    \centering
    \includegraphics[width=\textwidth]{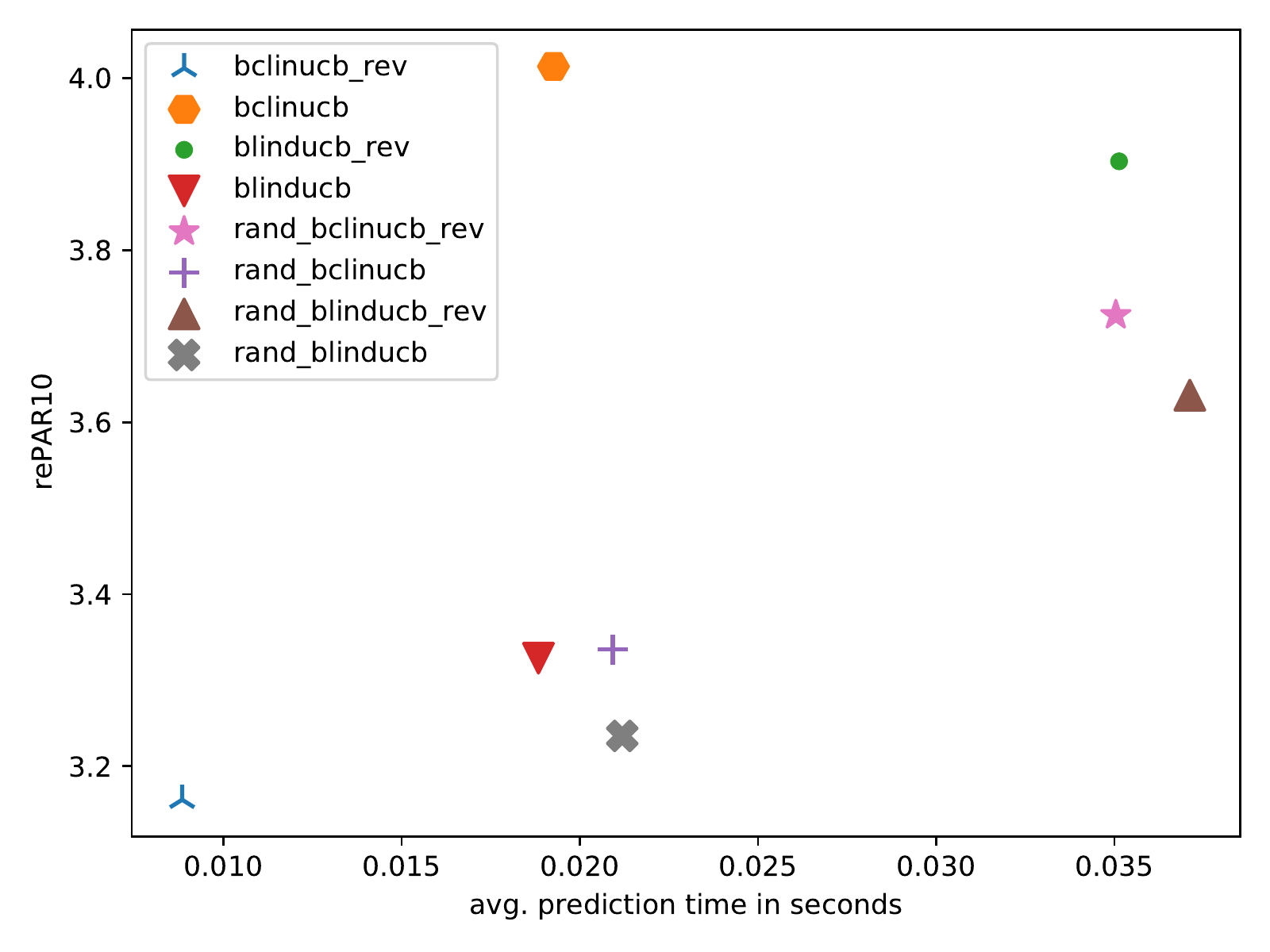}
    \caption{LinUCB variants}\label{fig:ablation_ucb}
  \end{subfigure}
  \begin{subfigure}[t]{0.33\textwidth}
    \centering
    \includegraphics[width=\textwidth]{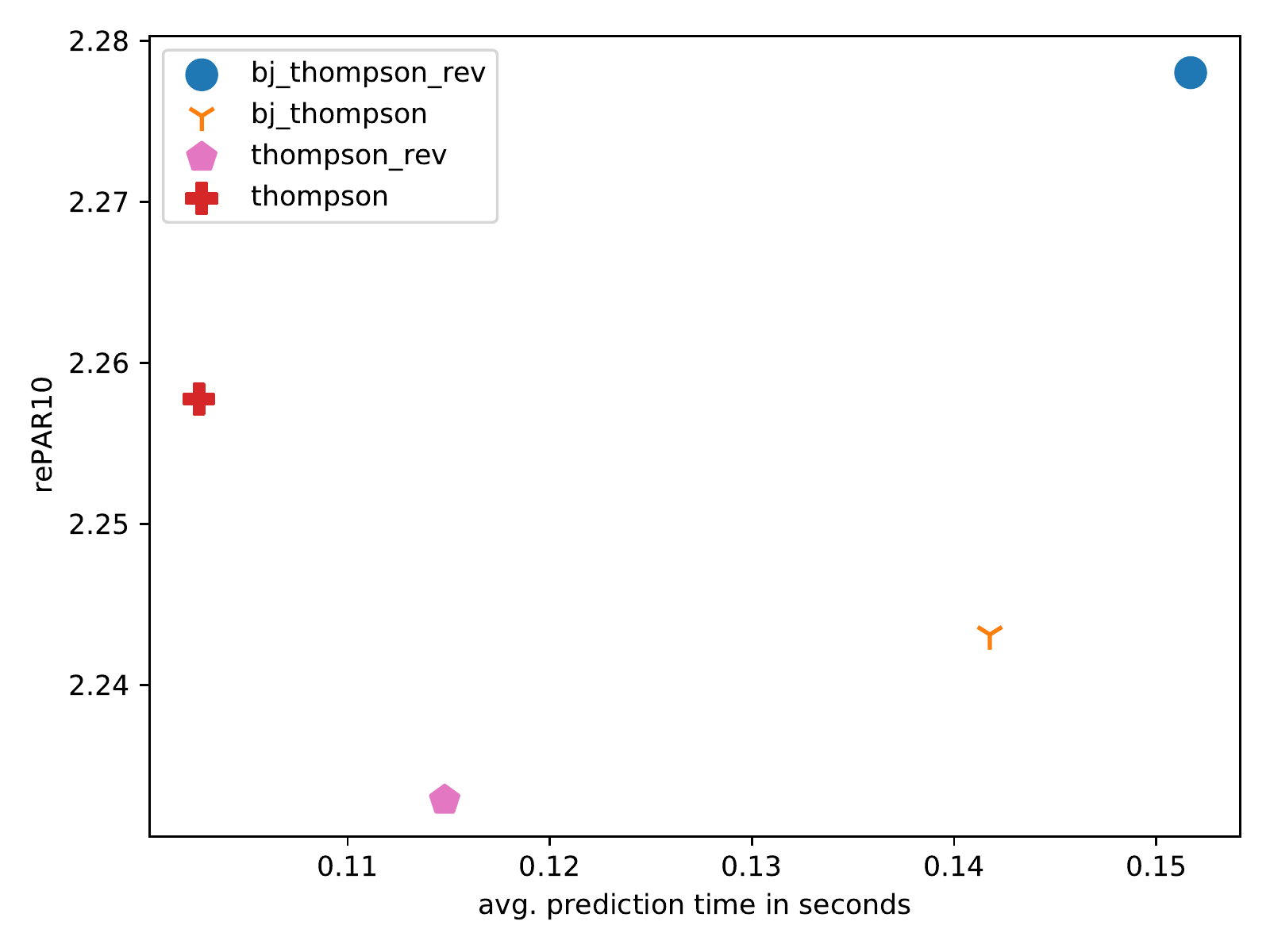}
    \caption{Thompson variants}\label{fig:ablation_thompson}
  \end{subfigure}
  \begin{subfigure}[t]{0.33\textwidth}
    \centering
    \includegraphics[width=\textwidth]{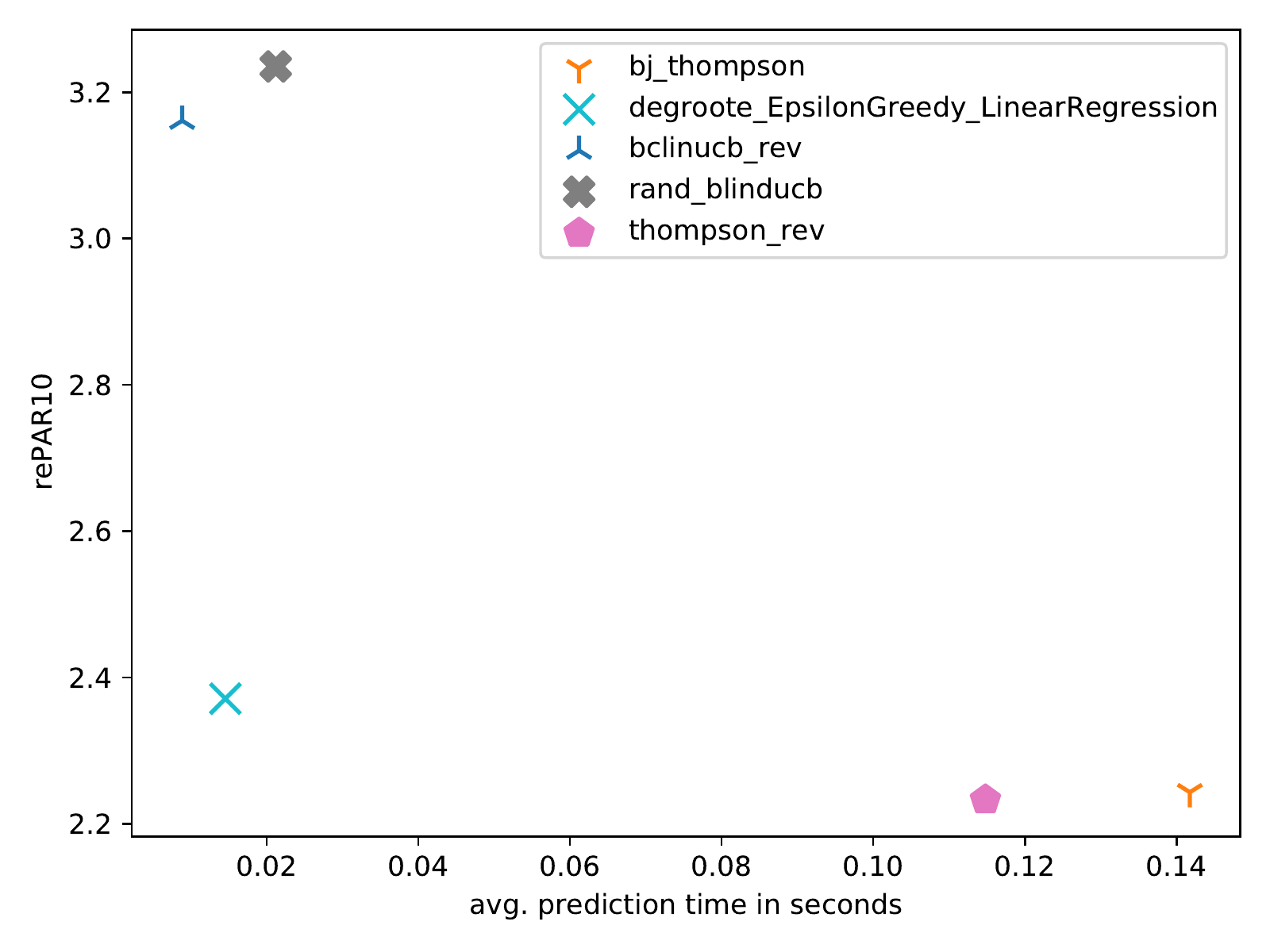}
    \caption{Degroote vs. this work}\label{fig:competitor}
  \end{subfigure}
  \caption{rePAR10 score averaged over all scenarios of approaches plotted against their average prediction time in seconds.}
\end{figure*}

\section{Evaluation}\label{sec:experimental_evaluation}

We base our evaluation on the standard algorithm selection benchmark library ASlib (v4.0) \cite{as_lib} and compare to the most relevant competitor approaches. ASlib is a curated collection of over 25 different algorithm selection problems, called scenarios, based on different algorithmic problem classes such as SAT, TSP, CSP. Each scenario comprises several instances for which the performance of a set of algorithms has been evaluated using a certain cutoff to avoid excessively long algorithm runs. An overview of the included scenarios and their statistics can be found in Section \ref{sec:aslib_overview} of the appendix. Since ASlib was originally designed for offline AS, we do not use the train/test splits provided by the benchmark, but rather pass each instance one by one to the corresponding online approaches, ask them to select an algorithm and return the corresponding feedback. To increase evaluation robustness, we randomly shuffle the instances of each scenario, repeat the evaluation ten times with different seeds and always report average or median aggregations across those ten repetitions. As ASlib contains missing feature values for some instances in some scenarios, we imputed these using the mean feature value of all instances seen until that point. Moreover, features were scaled to unit vectors by dividing by their norm. If the according variant does not self-impute censored values, these were imputed with the cutoff time.

All experiments were run on machines featuring Intel Xeon E5-2695v4@2.1GHz CPUs with 16 cores and 64GB RAM, where each approach was limited to a single-core. All code including detailed documentation can be found \ifthenelse{\boolean{blind}}{}{on Github\footnote{ \url{https://github.com/alexandertornede/online_as}} and} in the appendix. The corresponding hyperparameter settings used for the experiments can be found in Section \ref{sec:hyperparams} of the appendix\ifthenelse{\boolean{blind}}{}{ and in the repository}, parameter sensitivity analyses in Section \ref{sec:app_sensitivity_analysis}.

Instead of directly reporting PAR10 scores, we sometimes resolve to reporting a normalized version called rePAR10 (relative PAR10), which is comparable across scenarios and defined with respect to the oracle\footnote{Although in standard AS one usually uses the nPAR10 which is defined wrt. to both the oracle and algorithm best on average (aka. SBS), we abstain from using it as the SBS cannot be as easily defined as in the standard setting since instead of offline training data only the performance of the selected algorithm (and not of all) in the current time step is available to update the underlying model.}. The rePAR10 is simply defined as the PAR10 score of the corresponding approach divided by the PAR10 score of the oracle, i.e., the smaller the rePAR10, the better. Moreover, we will explicitly analyze the ``prediction time'', i.e., the time an approach requires for making a single selection and updating its model with the corresponding feedback. 

\subsection{Ablation Study}
First, we analyze how the different LinUCB and Thompson variants perform in terms of rePAR10 performance, when some of their components are activated or deactivated.

\subsubsection{LinUCB}

Recall that we differentiate in principle between \blinducb{} and \biascorrecteducb. Both the randomization idea (denoted by a 'rand\_' prefix) and the expected PAR10 loss minimization (denoted by '\_rev' suffix) can in principle be incorporated into both yielding a total of $8$ variants. 

Figure \ref{fig:ablation_ucb} shows the median rePAR10 score over all scenarios of the corresponding variant plotted against its prediction time in seconds. First of all, it is very clear that all of the LinUCB variants are at least $3.1$ times as worse as the oracle. A closer look at the selections made by the corresponding models shows two things. First, although \blinducb{} heavily underestimates runtimes as it completely ignores censored samples, its estimates yield some of the best algorithm selections among all LinUCB variants. Second, except for the revisited versions, \biascorrecteducb{} yields worse results than the corresponding \blinducb{} variant in all cases. As hinted at earlier, \biascorrecteducb{} suffers from very large confidence bounds due to the correction, yielding suboptimal selections in many cases. Moreover, one can see that directly minimizing the expected PAR10 loss does not prove to be very beneficial (except for the pure \biascorrecteducb{} variant) and can even worsen the performance for some variants. From a methodological point of view, this is not surprising for \blinducb{}, as it would technically require a different treatment of the expected PAR10 loss based on a truncated (instead of a censored) linear model (cf.\ \citet{greene2005censored}). However, for the \biascorrecteducb{} variants, this is rather disappointing. In contrast, the randomization (i.e.\ RandUCB) yields consistent improvements (except for one case), making some of the randomized variants the best among all LinUCB variants. This also coincides with our observation that the poor selection performance is caused by large confidence widths due to the correction, which are decreased through randomization.

\subsubsection{Thompson}
We presented both a na\"ive and a revisited form of Thompson incorporating expected PAR10 loss minimization ('\_rev' suffix). Moreover, both versions can be equipped with the Buckley-James imputation strategy discussed at the end of Section \ref{sec:thompson_rev} ('bj\_' prefix), yielding a total of 4 variants. 

Figure \ref{fig:ablation_thompson} shows the median rePAR10 score over all scenarios of the corresponding variant plotted against its average prediction time per instance. As expected, the more components are active, the longer the prediction time becomes. However, the average prediction time per instance still remains below $0.16$s. Both the revisited and the Buckley-James variant yield an improvement over the plain Thompson variant. A combined variant worsens the performance, meaning that the revisited variant achieves the best performance. However, overall one has to note that all variants behave rather similar with only small differences in performance.

\subsection{Comparison to Competitors}

\begin{table}[t]
    \centering
    \resizebox{\columnwidth}{!}{%
        \begin{tabular}{llll}
\toprule
Approach & bj\_thompson & thompson\_rev & degroote\_$\epsilon$-greedy\_LR \\\midrule
Scenario          &                                                  &                                                   &                                                                               \\
\midrule
ASP-POTASSCO           &                   \underline{949.38} $\pm$ 62.38 &                       \textbf{902.64} $\pm$ 78.43 &                                1047.13 $\pm$ 46.50 \\
BNSL-2016              &                 \underline{9638.04} $\pm$ 378.05 &                     \textbf{9467.01} $\pm$ 252.52 &                             12510.26 $\pm$ 1291.03 \\
CPMP-2015              &                            8241.01 $\pm$ 1164.85 &                 \underline{8158.72} $\pm$ 1268.83 &                      \textbf{6991.97} $\pm$ 501.36 \\
CSP-2010               &                             8295.76 $\pm$ 699.43 &                  \underline{7892.67} $\pm$ 692.83 &                      \textbf{7593.13} $\pm$ 208.94 \\
CSP-MZN-2013           &                             8207.06 $\pm$ 532.70 &                  \underline{8171.21} $\pm$ 594.49 &                      \textbf{8034.62} $\pm$ 113.78 \\
CSP-Minizinc-Time-2016 &                 \underline{4811.54} $\pm$ 409.79 &                     \textbf{4759.50} $\pm$ 306.03 &                               5258.70 $\pm$ 406.91 \\
GRAPHS-2015            &              \underline{4.1e+07} $\pm$ 4.4e+06 &                           4.2e+07 $\pm$ 3.4e+06 &                   \textbf{3.5e+07} $\pm$ 1.4e+06 \\
MAXSAT-PMS-2016        &                 \underline{2853.44} $\pm$ 210.21 &                     \textbf{2808.51} $\pm$ 218.55 &                               3279.54 $\pm$ 133.00 \\
MAXSAT-WPMS-2016       &                 \underline{6304.15} $\pm$ 166.98 &                              6592.87 $\pm$ 210.25 &                      \textbf{6287.21} $\pm$ 541.69 \\
MAXSAT12-PMS           &                 \underline{5347.39} $\pm$ 291.87 &                              5408.40 $\pm$ 482.42 &                      \textbf{5308.11} $\pm$ 129.30 \\
MAXSAT15-PMS-INDU      &                 \underline{3046.05} $\pm$ 128.34 &                      \textbf{3032.08} $\pm$ 90.71 &                               3867.70 $\pm$ 255.98 \\
MIP-2016               &                    \textbf{8081.57} $\pm$ 845.74 &                 \underline{8746.73} $\pm$ 1159.36 &                             10644.68 $\pm$ 3405.18 \\
PROTEUS-2014           &                   \textbf{13484.34} $\pm$ 541.83 &                 \underline{14115.69} $\pm$ 768.16 &                              15622.29 $\pm$ 784.60 \\
QBF-2011               &                            15708.25 $\pm$ 784.81 &                 \underline{15178.86} $\pm$ 904.72 &                     \textbf{13912.24} $\pm$ 356.69 \\
QBF-2014               &                    \textbf{3629.40} $\pm$ 220.68 &                  \underline{3679.96} $\pm$ 256.03 &                               4116.15 $\pm$ 116.27 \\
QBF-2016               &                 \underline{5082.59} $\pm$ 718.71 &                     \textbf{5045.16} $\pm$ 848.59 &                               5346.29 $\pm$ 210.05 \\
SAT03-16\_INDU          &                   \textbf{11980.15} $\pm$ 193.67 &                 \underline{12154.46} $\pm$ 221.01 &                              12754.50 $\pm$ 200.55 \\
SAT11-HAND             &                           30484.08 $\pm$ 1379.35 &                 \underline{30085.51} $\pm$ 764.32 &                     \textbf{29544.70} $\pm$ 952.78 \\
SAT11-INDU             &                            17540.58 $\pm$ 530.82 &                 \underline{17028.84} $\pm$ 479.15 &                     \textbf{17018.24} $\pm$ 647.90 \\
SAT11-RAND             &                  \textbf{18061.78} $\pm$ 2770.70 &                \underline{19061.88} $\pm$ 2522.11 &                              21008.77 $\pm$ 530.22 \\
SAT12-ALL              &                    \textbf{4720.22} $\pm$ 432.14 &                  \underline{5132.48} $\pm$ 395.74 &                               5650.32 $\pm$ 214.36 \\
SAT12-HAND             &                    \textbf{7443.01} $\pm$ 180.51 &                  \underline{7509.02} $\pm$ 199.39 &                               7634.24 $\pm$ 267.89 \\
SAT12-INDU             &                     \textbf{4511.68} $\pm$ 76.33 &                              4945.79 $\pm$ 228.37 &                   \underline{4755.52} $\pm$ 206.95 \\
SAT12-RAND             &                    \textbf{4008.79} $\pm$ 206.59 &                  \underline{4523.33} $\pm$ 170.56 &                               5023.73 $\pm$ 174.68 \\
SAT15-INDU             &                    \textbf{7700.27} $\pm$ 310.65 &                  \underline{7856.08} $\pm$ 522.84 &                               8220.22 $\pm$ 525.13 \\
SAT18-EXP              &                \underline{25201.41} $\pm$ 681.42 &                    \textbf{24906.56} $\pm$ 540.36 &                              25272.35 $\pm$ 881.19 \\
TSP-LION2015           &                    \textbf{1226.11} $\pm$ 309.42 &                  \underline{1411.06} $\pm$ 329.16 &                               1634.79 $\pm$ 112.29 \\ \midrule
avgrank                &                                         \textbf{1.814815} &                                          \underline{1.888889} &                                           2.296296 \\
\bottomrule
\end{tabular}
    }
    \caption{Average PAR10 scores and standard deviation of Thompson variants and Degroote.}
    \label{tab:par10_thompson_vs_degroote}
\end{table}

In the following, we only compare two UCB and Thompson variants to the competitors to avoid overloading the evaluation. In particular, we compare to an approach from \citeauthor{online_as_main_degrooteCBK18} (cf.\ Section \ref{sec:related_work}). Their approaches essentially employ batch machine learning models (linear regression or random forests) on the runtime, which are fully retrained after each seen instance. The selection is either done via a simple $\epsilon$-greedy strategy \cite[Chapter 2]{sutton2018reinforcement} or using a UCB strategy, where the confidence bounds are estimated using the standard deviation extracted from the underlying random forest by means of the Jackknife \cite{sextonL09} method. In fact, the Degroote approaches cannot be considered true online algorithms due to their dependence on the time horizon\,---\,they become intractable with an increasing number of instances. Although one can update the underlying models in principle less often (e.g., every ten instances as in the original paper), we abstain here from doing so, because our approaches also incorporate every sample immediately. 
As we only consider linear models in this work, we only compare to the linear $\epsilon$-greedy strategy presented by \citeauthor{online_as_main_degrooteCBK18} and abstain from comparing against the random forest versions to avoid that the model complexity becomes a confounding factor in the evaluation.

Figure \ref{fig:competitor} illustrates the rePAR10 value in comparison to the prediction time in seconds of our most successful bandit algorithms and the linear $\epsilon$-greedy Degroote approach. First, it is easy to see that the Thompson variants largely outperform the LinUCB variants in terms of performance at the cost of being slightly slower in terms of prediction time.

Second, the Thompson variants improve around $6\%$ in terms of performance upon the Degroote approach. Interestingly, the latter can compete with all online algorithms in terms of prediction time, and even outperforms the Thompson variants. This is mainly because of the limited size of the data, and because the batch linear regression of the library used for implementation of the Degroote approach is extremely efficient, making batch training affordable. Besides, the Thompson variants require sampling from a multi-variate normal distribution, taking up most of the prediction time. Nevertheless, as already said, batch learning will necessarily become impracticable with an increasing number of observations, and sooner or later get slower than the incremental Thompson approach. 

Table \ref{tab:par10_thompson_vs_degroote} illustrates a more nuanced comparison between the best Thompson variants and Degroote, where the best value for each scenario is printed in bold and the second best is underlined.



Overall, one can verify that Thompson sampling is a much more successful strategy than both $\epsilon$-greedy and LinUCB in OAS. Moreover, directly optimizing the expected PAR10 score (\_rev variants) and thereby incorporating the right-censoring of the data often proves beneficial, yielding the best OAS approach in this work in the form of Thompson\_rev. Nevertheless, as the large rePAR10 scores indicate, there is still room for improvement.

\section{Related Work}\label{sec:related_work}
Most related from a problem perspective is the work by \citeauthor{online_as_main_degrooteCBK18}. In a series of papers \cite{degroote2016reinforcement,online_as_degroote2017,online_as_main_degrooteCBK18}, they define the OAS problem in a similar form as we do and present different context-based bandit algorithms. However, their approaches essentially rely on batch learning algorithms, making their time- and space-complexity dependent on the time horizon\footnote{As we show in this work, some of their batch learning algorithms can actually be replaced by online learners.}. Moreover, they do not explicitly consider the problem of censoring, but apply a PAR10 imputation (as standard in ASlib). Lastly, compared to our work, their approaches lack a theoretical foundation, for instance, their models on the runtimes would in principle even allow negative runtime predictions.


The majority of other work related to OAS is situated in the fields of (online) algorithm scheduling \cite{lindauer2016empirical} and dynamic algorithm configuration \cite{biedenkapp2020dynamic} (aka.\ algorithm control  \cite{biedenkapp2019towards}), where the goal is to predict a schedule of algorithms or dynamically control the algorithm during the solution process of an instance instead of predicting a single algorithm as in our case. \citet{dynamic_algorithm_portfolios_gaglioloS06}, \citet{algorithm_survival_analysis_gaglioloL10}, \citet{gagliolo2010algorithm}, \citet{pimpalkhare2021medleysolver}, and \citet{cicirello2005max} essentially consider an online algorithm scheduling problem, where both an ordering of algorithms and their corresponding resource allocation (or simply the allocation) has to be computed. Thus, the prediction target is not a single algorithm as in our problem, but rather a very specific composition of algorithms, which can be updated during the solution process. 
Different bandit algorithms are used to solve this problem variant. \citet{lagoudakis2000algorithm}, \citet{armstrong2006dynamic}, \citet{van2018towards}, and \citet{laroche2017reinforcement} in one way or another consider the problem of switching (a component of) an algorithm during the solution process of an instance by means of reinforcement learning or bandit algorithms. They can be considered to be in the field of algorithm control and dynamic algorithm configuration. 

Another large corpus of related work can be found in the field of learning from data streams, where the goal is to select an algorithm for the next instance assuming that the data generating process might show a distributional shift \cite{gama2012survey}. To achieve this, \citet{rossi2012meta} and \citet{van2014algorithm} apply windowing techniques and apply offline AS approaches, which are trained on the last window of instances and use to predict for the next instance. Similarly, \citet{van2015having} dynamically adjust the composition and weights of an ensemble of streaming algorithms. In a way, the methods presented by \citet{online_as_main_degrooteCBK18} can be seen as windowing techniques where the window size is set to $t-1$, if $t$ is the current time step.

Finally, \citet{gupta2017pac} analyze several versions of the AS problem on a more theoretical level, including our version of OAS, and show for some problem classes the existence of an OAS approach with low regret under specific assumptions.
%

\section{Conclusion and Future Work}
In this paper, we revisited several well-known contextual bandit algorithms and discussed their suitability for dealing with the OAS problem under censored feedback. As a result of the discussion, we adapted them towards runtime-oriented losses, assuming partially censored data while keeping a space- and time-complexity independent of the time horizon. 
Our extensive experimental study shows that the combination of considering right-censored data in the selection process and an appropriate choice of the exploration strategy leads to better performance.
As future work, we plan to investigate whether online adaptations of non-parametric survival analysis methods (such as Cox-regression) are possible. 
Furthermore, results from offline algorithm selection suggest that an extension of our approaches to non-linear models seems useful to further improve performance.
%

%


\ifthenelse{\boolean{blind}}
{}
{
\section*{Acknowledgments}
This work was partially supported by the German Research Foundation (DFG) within the Collaborative Research Center ``On-The-Fly Computing'' (SFB 901/3 project no.\ 160364472). The authors gratefully acknowledge support of this project through computing time provided by the Paderborn Center for Parallel Computing (PC$^2$).
}

\bibliography{literature}

\clearpage
\newpage
\appendix
\onecolumn

\section*{Supplementary Material to ``Machine Learning for Online Algorithm Selection under Censored Feedback''}

\section{Table of Symbols} \label{sec_appendix_symbols}

The following table contains a list of symbols that are frequently used in the main paper as well as in the following supplementary material. \\ \medskip

\small
\begin{tabular}{l|l}
	%
	\hline
	\multicolumn{2}{c}{\textbf{Basics}} \\
	\hline
%
	%
	$1_{\llbracket \cdot \rrbracket}$  & indicator function\\
	$A^\intercal$ & transpose of a matrix $A$ (possibly non-quadratic) \\
	$A^{-1}$ & inverse of an invertible (quadratic) matrix $A$ \\	
	$\|\vec{x}\| $ & $\sqrt{\vec{x}^\intercal \vec{x}}$ for any $\vec{x}\in \R^d$ (Euclidean norm) \\
	$\|\vec{x}\|_A $ & $\sqrt{\vec{x}^\intercal A^{-1} \vec{x}}$ for any $\vec{x}\in \R^d$ and semi-positive definite matrix $A\in \R^{d\times d}$ \\
	$\vec{x}[i]$ &  $i$th component of a vector $\vec{x}\in \R^d$ for $i=1,\ldots,d$ \\
	$I_d$ & identity matrix of size $d \times d$  \\
	$\mathrm{N} \big( \boldsymbol{\mu},  \Sigma \big)$ & multivariate Gaussian distribution with mean vector $ \boldsymbol{\mu} \in \R^d$ and covariance matrix $\Sigma\in \R^{d\times d}$ \\
	$\mathrm{LN}(\mu,\sigma^2)$ & log-normal distribution with parameters $\mu\in \R$ and $\sigma>0$ \\
	$ \Phi_{\mu,\sigma^2}$ & univariate Gaussian distribution with mean $\mu\in \R$ and standard deviation $\sigma>0$  \\
	$X \sim P$ & $X$ is distributed as $P$ \textbf{or} $X$ is sampled from $P$ \\
	$\mathbb{P},\expectation{\cdot}$ & probability, expected value \\
	\hline
	\multicolumn{2}{c}{\textbf{OAS related}} \\
	\hline
	$\algorithms$ & (finite) set of algorithms \\
	$\instances$ & problem instance space \\ 
	$f:\instances \to \R^d$ & feature function ($d \in \mathbb N$ feature dimensionality)\\
	$i_t$ & problem instance at timestep $t$ \\
    $f(i),\vec{f}_{i_t}$ & features of problem instance $i,i_t \in \instances$ \\
	$s: \mathcal{H} \times \instances \rightarrow \algorithms$ & algorithm selector \\
	$a_t$ & algorithm chosen by algorithm selector at time step $t,$ i.e., $a_t=s_t(i_t)$ \\
	$l:\instances \times \algorithms \to \R$  & loss function \\
	$l_{i,a}, l_{t,a}$ & loss of algorithm $a$ (or $a_t$) used on problem instance $i$ (or $i_t$)  \\
	$s^*: \mathcal{H} \times \instances \rightarrow \algorithms$  & oracle or virtual best solver $s^*_t(h_t, i_t) \defeq \arg\min\nolimits_{a \in \algorithms} \mathbb{E}[l(i_t,a) \vert h_t]$ \\
	$T$ & number of overall timesteps (element of $\mathbb N$)\\ 
	$\mathcal{L}(s)$ & averaged cumulative loss of algorithm selector $s$, i.e., $T^{-1} \sum_{t=1}^T l(i_t,a_t)$ if $T$ finite, else $\lim_{T \to \infty } T^{-1} \sum_{t=1}^T l(i_t,a_t)$\\
	\hline
	\multicolumn{2}{c}{\textbf{Runtime related}} \\
	\hline
	$\mathcal{P}:\R \to \R$ & penalty function (for penalizing unsolved instance)\\
	$C$ & cutoff time (element of $\R_+)$ \\
	$\vec{\theta}^*_a $ & unknown weight vector (element of $\R^d$) of algorithm $a$ according to the model in \eqref{defi_exp_runtime_model}  \\
	$m:\instances\times \algorithms \to \R_+$ & runtime function, i.e., $m(i,a)$ is the (noisy) runtime of algorithm $a$ on problem instance $i$ \\
	$m_{i,a}$ & abbreviation for $m(i,a)$ \\
	$y_{i,a}$ & logarithmic (noisy) runtime  of algorithm $a$ on problem instance $i,$ i.e., $\log(m_{i,a})$ \\
	$\tilde y_{i,a}$ & $\min(y_{i,a},\log(C)),$ i.e., possibly imputed logarithmic (noisy) runtime of algorithm $a$ on problem instance $i$ \\
	$\epsilon_{i,a}$ & stochastic noise variable in model \eqref{defi_exp_runtime_model} \\
	\hline
	\multicolumn{2}{c}{\textbf{Bandit related}} \\
	\hline
	$X_{t,a}$ & the design matrix corresponding to algorithm $a$ at timestep $t$ \\
	& (stacking row by row the feature vectors whenever $a$ is chosen) \\
	$A_{t,a}$ & $\lambda I_{d} + X_{t,a}^\intercal X_{t,a},$ i.e., the regularized (through $\lambda>0$) Gram matrix corresponding to algorithm $a$ at timestep $t$ \\
	%
	%
	$\widehat{\vec{\theta}}_{t,a}$ & ridge regression (RR) estimate for $\theta^*_a$ at timestep $t$ for linearized model (cf.\ \eqref{defi_runtime_model}) \\
	& (non-imputed version defined in \eqref{defi_ols_na\"ive}, imputed version defined in \eqref{defi_ols_imputed}) \\
	$ w_{t,a} \, \, :\R^d \to \R_+$ & confidence width (function) of non-imputed RR estimate $\widehat{\vec{\theta}}_{t,a}$ for $\theta^*_a$ at timestep $t$ for linearized model \\
	& (cf.\ text below \eqref{eq_choice_linucb})  \\
	$ w_{t,a}^{(bc)}:\R^d \to \R_+$ & confidence width (function) of imputed RR estimate $\widehat{\vec{\theta}}_{t,a}$ for $\theta^*_a$ at timestep $t$ for linearized model \\
	& (cf.\ Sec. \ref{sec_bclinUCB}) \\
	$ \tilde w_{t,a} \, \, :\R^d \to \R_+$ & randomized confidence width (function) of imputed RR estimate $\widehat{\vec{\theta}}_{t,a}$ for $\theta^*_a$ at timestep $t$ for linearized model \\
	& (cf.\ Sec. \ref{subsec_rand_ucb})  \\
	$N^{(C)}_{a,t}$  & number of timeouts of algorithm $a$ until $t$ \\
	$ E_{C}, E_{C}(\vec{f}_{i_t}^\intercal\vec{\theta}^*_a,\sigma)$ &  conditional expectation of $\mathrm{LN}( \vec{f}_{i_t}^\intercal\vec{\theta}^*_a$ , $\sigma^2)$ under a cutoff $C$ (cf.\ \eqref{eq_cond_exp_lnorm}) \\
	$C_{i_t,a}^{(1)}$ &  $\frac{ \log(C)- \vec{f}_{i_t}^\intercal\vec{\theta}^*_a -\sigma^2 }{\sigma} $ \\
	$C_{i_t,a}^{(2)}$ & $C_{i_t,a}^{(1)} + \sigma $ \\
	\end{tabular}
\normalsize

\section{Pseudo Codes} \label{sec_appendix_pseudo_codes}

In this section, we provide pseudo codes for the algorithms introduced in Sec.\ \ref{sec_lin_bandits} and \ref{sec_exp_par_10_loss_min}.
However, we abstain from defining the pseudo code of the plain Thompson Sampling algorithm in Sec.\ \ref{sec_TS} as it should be clear from its definition. 
Moreover, note that it is straightforward to see that the solution of \eqref{defi_ols_na\"ive} is $ \widehat{\vec{\theta}}_{t,a}  = ( A_{t,a})^{-1} \vec{b}_{t,a},$ where $\vec{b}_{t,a} = X_{t,a}^\intercal \vec y_{t,a}$ and $\vec y_{t,a}$ is the (column) vector storing all observed non-censored log-runtimes until $t$ whenever $a$ has been chosen.
The solution of \eqref{defi_ols_imputed} is $ \widehat{\vec{\theta}}_{t,a}  = ( A_{t,a})^{-1} \vec{b}_{t,a},$ where $\vec{b}_{t,a} = X_{t,a}^\intercal \vec{\tilde{y}}_{t,a}$ and $\vec{\tilde{y}}_{t,a}$  is the (column) vector storing all observed and possibly imputed log-runtimes until $t$ whenever $a$ has been chosen.
Further, both $A_{t,a}$ and $\vec{b}_{t,a}$ can be updated in an iterative fashion without actually storing all seen samples: If algorithm $a$ is chosen at timestep $t,$ then 
$$ A_{t+1,a} =  \begin{cases}
    		    A_{t,a} + \vec{f}_{i_t} \vec{f}_{i_t}^\intercal , &  \\
    		     A_{t,a} + 1_{\llbracket m(i_t,a_t) \leq C \rrbracket} \vec{f}_{i_t} \vec{f}_{i_t}^\intercal, & \\
    		\end{cases}
    		\quad \mbox{and} \quad 
    		\vec{b}_{t+1,a} = \begin{cases}
    		     \vec{b}_{t,a} + \tilde y_{i_t,a} \vec{f}_{i_t}, & \mbox{for \eqref{defi_ols_imputed}}, \\
    		     \vec{b}_{t,a} + 1_{\llbracket m(i_t,a_t) \leq C \rrbracket} y_{i_t,a} \vec{f}_{i_t}, & \mbox{for \eqref{defi_ols_na\"ive}}. \\
    		\end{cases} $$
As the updates of the matrices $A_{t+1,a}$ are via a rank-one update, one can use the well-known Sherman-Morrison formula to compute their inverse in a sequential manner as well (similarly for $A_{t+1,a}$ based on \eqref{defi_ols_na\"ive}):
$$  (A_{t+1,a} )^{-1}  =  (A_{t,a} + \vec{f}_{i_t} \vec{f}_{i_t}^\intercal )^{-1} = A_{t,a}^{-1} - \frac{A_{t,a}^{-1}  \vec{f}_{i_t} \vec{f}_{i_t}^\intercal A_{t,a}^{-1} }{1+\vec{f}_{i_t}^\intercal A_{t,a}^{-1} \vec{f}_{i_t}}.$$
\subsection{LinUCB} 

Alg.\ \ref{alg:lin_ucb} provides the pseudo code for the LinUCB variants introduced in Sec. \ref{sec_lin_bandits}, that is, BlindUCB, BClinUCB and RandUCB.
For sake of convenience, we recall the role of each input parameter in the following.
\begin{itemize}
    \item $\lambda > 0$  is a regularization parameter due to the considered ridge regression.  
    \item $\alpha>0$ essentially controls the degree of exploration as a multiplicative term of the confidence width (see \eqref{eq_choice_linucb}). The higher (lower) it is chosen, the more (less) exploration is conducted. 
    \item $C>0$ is the cutoff time and depends on the considered problem scenario (specified by the ASlib library).
    \item $\tilde \sigma>0$ is (only) used for RandUCB in order to specify the random sample's variance within the confidence width (see Sec.\ \ref{subsec_rand_ucb}). The higher (lower) its choice, the larger (smaller) the effective exploration term, i.e., $|r|\cdot w_{t,a}^{(bc)}( \vec{x}_t).$
\end{itemize}

\begin{algorithm}[ht]
    \caption{$\texttt{LinUCB}$ variants}
    \label{alg:lin_ucb}
	\begin{algorithmic}[1]
		\STATE {\textit{Input parameters} $\lambda\geq 0, \alpha, C>0$ half-normal parameter $\tilde \sigma$ for RandUCB }
		\STATE {\textit{Initialization}}
		\FOR{all $a \in \algorithms$}
		    \STATE 	$A_{t,a} = \lambda I_{d\times d}, \vec{b}_{t,a}  = 0_{d\times 1}, \widehat{\vec{\theta}}_{t,a}  = 0_{d \times 1},$  $  \widehat{l}_{t,a}  = 0,$ $N^{(C)}_{t,a}=0$ 
		\ENDFOR
		\STATE {\textit{Main Algorithm}}
		\FOR{time steps $t = 1 \dots, T$}
    		\STATE Observe instance $i_t$ and its features $\vec{x}_t = f(i_t)\in \R^d$
    		\IF {$t\leq |\algorithms|$}
    		    \STATE  Take algorithm $a_t \in \algorithms$ and  obtain   $y_t = \min(\log(m(i_t,a_t)),\log(C))$
    		\ELSE
    		    \FOR{all $a \in \algorithms$}
    		        \STATE  $ \widehat{\vec{\theta}}_{t,a}  \leftarrow ( A_{t,a})^{-1} \vec{b}_{t,a} $
    		        \STATE  $ \widehat{l}_{t,a}   \leftarrow \begin{cases}
    		            \vec{x}_t^\intercal \widehat{\vec{\theta}}_{t,a} - \alpha \cdot w_{t,a}( \vec{x}_t), & \mbox{BlindUCB} \\  
    		              \vec{x}_t^\intercal \widehat{\vec{\theta}}_{t,a} - \alpha \cdot w_{t,a}^{(bc)}( \vec{x}_t), & \mbox{BClinUCB} \\  
    		                \vec{x}_t^\intercal \widehat{\vec{\theta}}_{t,a} - \alpha \cdot |r|\cdot w_{t,a}^{(bc)}( \vec{x}_t), \quad r\sim  N(0,\tilde \sigma^2)  & \mbox{RandUCB} 
    		        \end{cases} $
    		    \ENDFOR
    		    \STATE Take algorithm $a_t = \arg\min_{a \in \algorithms} \hat l_{t,a}$ and obtain  $y_t = \min(\log(m(i_t,a_t)),\log(C))$
    		\ENDIF	
    		\STATE \textit{Updates:}
    		\STATE $N^{(C)}_{t,a} \leftarrow N^{(C)}_{t,a} + 1_{\llbracket m(i_t,a_t) > C \rrbracket}$
    		\STATE $ A_{t,a}  \leftarrow   \begin{cases}
    		    A_{t,a} + \vec{x}_t \vec{x}_t^\intercal , &  \\
    		     A_{t,a} + 1_{\llbracket m(i_t,a_t) \leq C \rrbracket} \vec{x}_t \vec{x}_t^\intercal, &  \\
    		\end{cases} \quad \vec{b}_{t,a} \leftarrow \begin{cases}
    		     \vec{b}_{t,a} + y_t \vec{x}_t, & \mbox{(BClinUCB,RandUCB)} \\
    		     \vec{b}_{t,a} + 1_{\llbracket m(i_t,a_t) \leq C \rrbracket} y_t \vec{x}_t, & \mbox{(BlindUCB)} \\
    		\end{cases} $
		\ENDFOR
	\end{algorithmic}
\end{algorithm}

\subsection{LinUCB Revisited} \label{sec_appendix_linucb_revi}

%
%
%

Alg.\ \ref{alg:lin_ucb_rev} provides the pseudo code for the LinUCB variants introduced in Sec. \ref{sec_lin_bandits} adapted to the refined loss decomposition in \eqref{eq:expected_loss_3} by incorporating the selections strategy in \eqref{eq_choice_linucb_revisited_2}.
Compared to Alg.\ \ref{alg:lin_ucb} there are two additional parameters: 
\begin{itemize}
    \item $\sigma>0$ which ideally should correspond to the standard deviation of the noise variables in the model assumption \eqref{defi_exp_runtime_model}.
    \item $\mathcal{P}:\R \to \R$ the penalty function for penalizing unsolved problem instances (we used $\mathcal{P}(z) = 10z$ corresponding to the $PAR10$ score). 
\end{itemize}

\begin{algorithm}[ht]
    \caption{$\texttt{LinUCB}$ variants based on \eqref{eq:expected_loss_3} (\_rev) versions}
    \label{alg:lin_ucb_rev}
	\begin{algorithmic}[1]
		\STATE {\textit{Input parameters} $\lambda\geq 0, \sigma>0,\alpha>0, C>0,\mathcal{P}:\R \to \R,$ half-normal parameter $\tilde \sigma$ for RandUCB }
		\STATE {\textit{Initialization}}
		\FOR{all $a \in \algorithms$}
		    \STATE 	$A_{t,a} = \lambda  I_{d\times d}, \vec{b}_{t,a}  = 0_{d\times 1}, \widehat{\vec{\theta}}_{t,a}  = 0_{d \times 1},$  $  \widehat{l}_{t,a}  = 0,$ $N^{(C)}_{t,a}=0$ 
		\ENDFOR
		\STATE {\textit{Main Algorithm}}
		\FOR{time steps $t = 1 \dots, T$}
    		\STATE Observe instance $i_t$ and its features $\vec{x}_t = f(i_t)\in \R^d$
    		\IF {$t\leq |\algorithms|$}
    		    \STATE  Take algorithm $a_t \in \algorithms$ and  obtain   $y_t = \min(\log(m(i_t,a_t)),\log(C))$
    		\ELSE
    		    \FOR{all $a \in \algorithms$}
    		        \STATE  $ \widehat{\vec{\theta}}_{t,a}  \leftarrow ( A_{t,a})^{-1} \vec{b}_{t,a} $
    		        \STATE $r \sim \mathrm{N}(0,\tilde \sigma^2)$ \qquad (only for RandUCB)
    		        \STATE  $ p_{t,a}  \leftarrow \begin{cases}
    		            \vec{x}_t^\intercal \widehat{\vec{\theta}}_{t,a} + \alpha \cdot w_{t,a}( \vec{x}_t), \\  
    		              \vec{x}_t^\intercal \widehat{\vec{\theta}}_{t,a} + \alpha \cdot w_{t,a}^{(bc)}( \vec{x}_t), \\  
    		                \vec{x}_t^\intercal \widehat{\vec{\theta}}_{t,a} + \alpha \cdot |r|\cdot w_{t,a}^{(bc)}( \vec{x}_t),  
    		        \end{cases}  
    		        o_{t,a}   \leftarrow \begin{cases}
    		            \vec{x}_t^\intercal \widehat{\vec{\theta}}_{t,a} - \alpha \cdot w_{t,a}( \vec{x}_t), & \mbox{BlindUCB} \\  
    		              \vec{x}_t^\intercal \widehat{\vec{\theta}}_{t,a} - \alpha \cdot w_{t,a}^{(bc)}( \vec{x}_t), & \mbox{BClinUCB} \\  
    		                \vec{x}_t^\intercal \widehat{\vec{\theta}}_{t,a} - \alpha \cdot |r|\cdot w_{t,a}^{(bc)}( \vec{x}_t),   & \mbox{RandUCB} 
    		        \end{cases} $
    		        \STATE $
\hat C_{i_t,a}^{(1,p)} \leftarrow \frac{ \log(C)- p_{t,a} -\sigma^2 }{\sigma}, \quad 
\hat C_{i_t,a}^{(1,o)} \leftarrow \frac{ \log(C)- o_{t,a} -\sigma^2 }{\sigma}, \quad 
\hat C_{i_t,a}^{(2,o)} \leftarrow \frac{ \log(C)- o_{t,a} }{\sigma},  \quad \mbox{and} \quad 
\hat C_{i_t,a}^{(2,p)} \leftarrow \frac{ \log(C)- p_{t,a} }{\sigma}. $
    		        %
    		        \STATE $  \widehat{l}_{t,a}  \leftarrow \exp( o_{t,a} + \nicefrac{\sigma^2}{2}) \cdot \frac{\Phi_{0,1}( \hat C_{i_t,a}^{(1,p)} )}{\Phi_{0,1}( \hat C_{i_t,a}^{(2,o)})}
        + \big(1-\Phi_{p_{t,a},\sigma}(\log(C)) \big)  
        \cdot \left(\mathcal{P}(C) - \exp( p_{t,a} + \nicefrac{\sigma^2}{2})  \cdot \frac{\Phi_{0,1}( \hat C_{i_t,a}^{(1,o)} )}{\Phi_{0,1}( \hat C_{i_t,a}^{(2,p)})} \right)$
    		    \ENDFOR
    		    \STATE Take algorithm $a_t = \arg\min_{a \in \algorithms} \hat l_{t,a}$ and obtain  $y_t = \min(\log(m(i_t,a_t)),\log(C))$
    		\ENDIF	
    		\STATE \textit{Updates:} Same as lines 19--20 of Alg.\ \ref{alg:lin_ucb}
		\ENDFOR
	\end{algorithmic}
\end{algorithm}

\clearpage

\subsection{Thompson Sampling Revisited} 

Alg.\ \ref{alg:thompson} provides the pseudo code for the revisited Thompson algorithm and its variant inspired by the Buckley-James estimate introduced in Sec. \ref{sec:thompson_rev}.
As before, we discuss the role each input parameter plays in the behavior of the algorithm:
\begin{itemize}    
    \item the role of $\lambda, \mathcal{P},C$ is the same as in Alg.\ \ref{alg:lin_ucb_rev}, respectively.  
    \item $\sigma>0$ specifies the magnitude of the posterior distribution's variance and is therefore slightly different to $\sigma$ in Alg.\ \ref{alg:lin_ucb_rev}.
    \item $\mathrm{BJ}$ specifies whether the Buckley-James inspired imputation strategy    described at the end of Sec.\ \ref{sec:thompson_rev}  ($\mathrm{BJ} = \mathrm{TRUE}$) or the n\"aive imputation strategy  ($\mathrm{BJ} = \mathrm{FALSE}$) should be deployed.
\end{itemize}
\begin{algorithm}[ht]
    \caption{$\texttt{(bj\_)Thompson\_rev}$}
    \label{alg:thompson}
	\begin{algorithmic}[1]
		\STATE {\textit{Input parameters} $\sigma>0,\lambda\geq 0,\mathcal{P}:\R \to \R,C,$  $\mathrm{BJ}\in \{\mathrm{TRUE,FALSE}\}$   }
		\STATE {\textit{Initialization}}
		\FOR{all $a \in \algorithms$}
		    \STATE 	$A_{t,a}  = \lambda \cdot I_{d\times d}, \vec{b}_{t,a}  = 0_{d\times 1}, \widehat{\vec{\theta}}_{t,a}  = 0_{d \times 1}, \tilde{\sigma}^2_{t,a} \leftarrow 0$ and $  \widehat{l}_{t,a}  = 0$
		\ENDFOR
		\STATE {\textit{Main Algorithm}}
		\FOR{time steps $t = 1 \dots, T$}
    		\STATE Observe instance $i_t$ and its features $\vec{x}_t = f(i_t)\in \R^d$
    		\IF {$t\leq |\algorithms|$}
    		    \STATE  Take algorithm $a_t \in \algorithms$ and  obtain   $y_t = \min(\log(m(i_t,a_t)),\log(C))$
    		\ELSE
    		    \FOR{all $a \in \algorithms$}
    		        \STATE  $ \widehat{\vec{\theta}}_{t,a}  \leftarrow ( A_{t,a})^{-1} \vec{b}_{t,a} $ \quad $\tilde{\sigma}^2_{t,a} \leftarrow \sigma \| \vec{f}_{i_t} \|_{A_{t,a}}^2$
    		        \STATE Sample  $  \widetilde{\vec{\theta}}_{a}  \sim   \mathrm{N} \big( \widehat{\vec{\theta}}_{t,a} , \sigma (A_{t,a})^{-1} \big) $ 
    		        \STATE  $ \widehat{l}_{t,a}   \leftarrow 
    		        \big(1-\Phi_{\vec{f}_{i_t}^\intercal\widetilde{\vec{\theta}}_a,\tilde{\sigma}^2_{t,a}}(\log(C)) \big) \cdot \big(\mathcal{P}(C) - E_{C}(\vec{f}_{i_t}^\intercal\widetilde{\vec{\theta}}_a,\tilde{\sigma}_{t,a}) \big) 
                    + E_{C}(\vec{f}_{i_t}^\intercal\widetilde{\vec{\theta}}_a,\tilde{\sigma}_{t,a}) $ \quad (RHS of \eqref{eq_choice_tompcen_2})
    		    \ENDFOR
    		    \STATE Take algorithm $a_t = \arg\min_{a \in \algorithms} \hat l_{t,a}$ and obtain   $y_t = \min(\log(m(i_t,a_t)),\log(C))$
    		\ENDIF	
    		\STATE \textit{Updates:}
    		\IF {$y_t = \log(C)$ and $\mathrm{BJ} = \mathrm{TRUE}$}
    		    \STATE Sample  $ \widecheck{\vec{\theta}}_{a}  \sim   \mathrm{N} \big( \widehat{\vec{\theta}}_{t,a} , \sigma (A_{t,a})^{-1} \big) $ (if $ \exp(\vec{x}_t^\intercal \widecheck{\vec{\theta}}_{a}) \leq C$ sample again)
    		    \STATE $y_t \gets \log(\vec{x}_t^\intercal \widecheck{\vec{\theta}}_{a})$
    		\ENDIF
    		\STATE $ A_{t,a}  \leftarrow   A_{t,a} + \vec{x}_t \vec{x}_t^\intercal $
    		\STATE  $\vec{b}_{t,a} \leftarrow \vec{b}_{t,a} + y_t \vec{x}_t $
		\ENDFOR
	\end{algorithmic}
\end{algorithm}

\section{Deriving the Bias-corrected Confidence Bounds} \label{sec_appendix_bc_cb}

This section is concerned with giving the details for deriving the bias-corrected confidence bounds in  Sec. \ref{sec_bclinUCB}.
In particular, we focus on the solution of \eqref{defi_ols_imputed} which is given by $ \widehat{\vec{\theta}}_{t,a}  = ( A_{t,a})^{-1} \vec{b}_{t,a},$ where $\vec{b}_{t,a} = X_{t,a}^\intercal \vec{\tilde{y}}_{t,a}$ and $\vec{\tilde{y}}_{t,a}$  is the (column) vector storing all observed and possibly imputed log-runtimes until $t$ whenever $a$ has been chosen.
We follow the approach by \cite{chuLRS11_linucb} and analyze the deviation of $ \vec{x}_t^\intercal \hat \theta_{t,a}$ and $ \vec{x}_t^\intercal  \vec{\theta}_{a}^*,$ where we write for sake of convenience $\vec{x}_t=f(i_t).$
First, note that
\allowdisplaybreaks
\begin{equation} \label{first_aux_ineq_cb}
\begin{split}
	\vec{x}_t^\intercal \widehat{\vec{\theta}}_{t,a} - \vec{x}_t^\intercal \vec{\theta}_{a}^*
	&= \vec{x}_t^\intercal  A_{t,a}^{-1} \vec{b}_{t,a} 
	- \vec{x}_t^\intercal A_{t,a}^{-1} A_{t,a} \vec{\theta}_{a}^* \\
	&= \vec{x}_t^\intercal  A_{t,a}^{-1} X_{t,a}^\intercal \vec{\tilde{y}}_{t,a}
	- \vec{x}_t^\intercal A_{t,a}^{-1} \Big(\lambda I_{d} + X_{t,a}^\intercal X_{t,a}\Big) \vec{\theta}_{a}^* \\
	&= \vec{x}_t^\intercal  A_{t,a}^{-1} X_{t,a}^\intercal \Big(  \vec{\tilde{y}}_{t,a}  - X_{t,a} \vec{\theta}_{a}^* \Big)  
	 - \lambda \vec{x}_t^\intercal  A_{t,a}^{-1}  \vec{\theta}_{a}^*  \\
	&= \vec{z}_{t,a} (   \vec{\tilde{y}}_{t,a}  - X_{t,a} \vec{\theta}_{a}^* )  
	- \lambda \vec{x}_t^\intercal  A_{t,a}^{-1}  \vec{\theta}_{a}^* \\
    &=: (A1) - (A2),
\end{split}
\end{equation}
where we abbreviated  $\vec{x}_t^\intercal  A_{t,a}^{-1} X_{t,a}^\intercal$ by $\vec{z}_{t,a},$ which is a row vector with $N_a(t)$ components, i.e., $N_a(t)$ is the total  number of times $a$ has been chosen till $t$.
We can write the term $(A1)$ as
\begin{align*}
	\sum_{j:  m_{i_j,a}\leq C }  \vec{z}_{t,a}[j] \Big(  \log(m(i_j,a))  - \vec{x}_{i_j}^\intercal \vec{\theta}_{a}^* \Big) + 	\sum_{j:  m_{i_j,a}>C }  \vec{z}_{t,a}[s] \Big( \log(C)  - \vec{x}_{i_j}^\intercal \vec{\theta}_{a}^*\Big),
\end{align*}
i.e., we split it into the sum over all uncensored observations and the sum of all censored observations.
This is necessary as we want to apply Azuma's inequality and this can only be done for the terms in the first sum, since $\expectation{ \log(m(i_j,a))} = \vec{x}_{i_j}^\intercal \theta_{a}^*$ due to our assumptions made in Sec.\ \ref{sec_runtime_model}.
By applying Azuma's inequality we get for any $\tilde \alpha>0$ that
\begin{align*}
	\mathbb{P}\Big(	\Big| \sum\nolimits_{j:  m_{i_j,a}\leq C }  \vec{z}_{t,a}[j] \big(  \log(m(i_j,a))  - \vec{x}_{i_j}^\intercal \vec{\theta}_{a}^* \big) \Big| > \tilde \alpha \,	w_{t,a}(\vec{x}_t)	\Big)
	\leq 2 \exp\left( - 	\frac{2 \, \tilde \alpha^2 \, 	w_{t,a}^2(\vec{x}_t)}{\|\vec{z}_{t,a}\|^2}		\right),
\end{align*}
where $w_{t,a}$ is as in Sec.\ \ref{sec_lin_bandits} defined by $	w_{t,a}(\vec{x}_t) = \| \vec{x}_t \|_{A_{t,a}}$.
Note that 
\begin{align*}
    \|\vec{z}_{t,a}\|^2 
    &= \vec{x}_t^\intercal  A_{t,a}^{-1} X_{t,a}^T  X_{t,a}   A_{t,a}^{-1} \vec{x}_t 
    \leq \vec{x}_t^\intercal  A_{t,a}^{-1}  (  \lambda I_d +  X_{t,a}^T  X_{t,a}) A_{t,a}^{-1} \vec{x}_t 
    = \vec{x}_t^\intercal  A_{t,a}^{-1}  \vec{x}_t =  w_{t,a}^2(\vec{x}_t),
\end{align*}
since $ \lambda A_{t,a}^{-1}   I_d  A_{t,a}^{-1} $ is semi-positive definite.
Thus, by choosing $\tilde \alpha$ appropriately the latter probability is small.
In particular, we obtain that
\begin{align} \label{sec_aux_ineq_cb}
    \begin{split}
         \Big| \sum\nolimits_{j:  m_{i_j,a}\leq C }  \vec{z}_{t,a}[j] \big(  \log(m(i_j,a))  - \vec{x}_{i_j}^\intercal \vec{\theta}_{a}^* \big) \Big| \leq \tilde \alpha \,	w_{t,a}(\vec{x}_t)	
    \end{split}
\end{align}
holds with high probability by choosing $\tilde \alpha>0$ appropriately.
Now, we bound the (absolute values of the) censored sum.
Using the Cauchy-Schwarz inequality we can infer
\begin{align} \label{third_aux_ineq_cb}
    \begin{split}
	\Big|\sum\nolimits_{j:  m_{i_j,a}>C }  \vec{z}_{t,a}[s] \big( \log(C)  - \vec{x}_{i_j}^\intercal \vec{\theta}_{a}^*\big) \Big|
	&\leq \|\vec{z}_{t,a}\| \, \Big\|	\sum\nolimits_{j:  m_{i_j,a}>C } \big(  \log(C)  -  \vec{x}_{i_j}^\intercal \vec{\theta}_{a}^* \big)	\Big\| \\
	&\leq w_{t,a}(\vec{x}_t) \, \Big\|	\sum\nolimits_{j:  m_{i_j,a}>C } \big(  \log(C)  -  \vec{x}_{i_j}^\intercal \vec{\theta}_{a}^* \big)	\Big\| \\
	&\leq 2 \, w_{t,a}(\vec{x}_t) \, \sqrt{N_{a}^{(C)}(t)} \, \log(C),
    \end{split}
\end{align}
where we used for the last inequality that $ \vec{x}_{i_j}^\intercal \vec{\theta}_{a}^* = \vec{f}_{i_j}^\intercal \vec{\theta}^*_a \leq \log(C)$ holds by our assumptions made in Sec.\ \ref{sec_runtime_model}.
Finally, the absolute value of the term $(A2)$ can be bounded as follows 
\begin{align}\label{fourth_aux_ineq_cb}
    \begin{split}
         |(A2)| 
         &\leq \lambda \| \vec{\theta}_a^* \| \| \vec{x}_t^\intercal A_{t,a}^{-1} \| 
         \leq \lambda \| \vec{\theta}_a^* \| w_{t,a}(\vec{x}_t).
    \end{split}
\end{align}
Combining \eqref{first_aux_ineq_cb}--\eqref{fourth_aux_ineq_cb}, we have with high probability (depending on the choice of $\tilde \alpha$) that
\begin{align*}
	|\vec{x}_t^\intercal \vec{\widehat{\theta}}_{t,a} - \vec{x}_t^\intercal \vec{\theta}_{a}^*| \leq |(A1) | + |(A2)| 
	\leq w_{t,a}(\vec{x}_t) \, \Big(\tilde \alpha + \lambda \| \vec{\theta}_a^* \| + 2 \log(C) \sqrt{N_{a}^{(C)}(t)}  \Big).
\end{align*}
Thus, for some appropriate $\alpha>0$ we have with high probability 
\begin{align*}
|\vec{x}_t^\intercal \vec{\widehat{\theta}}_{t,a} - \vec{x}_t^\intercal \vec{\theta}_{a}^*| 
	\leq \alpha w_{t,a}^{(bc)}(\vec{x}_t).
\end{align*}

\paragraph{Bias issue.} Note that from the definition of $w_{t,a}^{(bc)}(\vec{x}_t)$ we can see the bias issue of $\vec{\widehat{\theta}}_{t,a}$ due to the imputation employed. 
The term $w_{t,a}(\vec{x}_t)$ is asymptotically tending against zero with the rate $\approx \sqrt{1/N_a(t)}.$
However, $w_{t,a}^{(bc)}(\vec{x}_t)$ does not tend to zero asymptotically for $t \to \infty$, if $\sqrt{N_{a}^{(C)}(t)/N_a(t)} \to C'$ for some constant $C'>0$.
The latter condition in turn seems to be satisfied if for any $t$ it holds that $\mathbb{P}(	m_{i_t,a} > C)>\epsilon>0,$ i.e., the probability of observing a censored runtime does not vanish in the course of time.

\section{Deriving the Refined Expected Loss Representation} \label{sec_appendix_expec_loss}

In this section, we provide the details for showing \eqref{eq:expected_loss_3}.
For this purpose, we need the following lemma showing the explicit form of the conditional expectation of a log-normal distribution under a certain cutoff.

\begin{Lem}\label{lemma_cond_exp_log_norm}
Let $Y\sim \mathrm{LN}(\mu,\sigma^2),$ i.e., a log-normally distributed random variable with parameters $\mu\in \R$ and $\sigma>0.$
Then, for any $C>0$ it holds that
\begin{equation} \label{eq_cond_exp_lnorm}
    \expectation{Y \vert Y \leq C} 
    = \exp(\mu+\nicefrac{\sigma^2}{2} ) 
    \cdot \frac{\Phi_{0,1}\left(\frac{\log(C)-\mu-\sigma^2}{\sigma}\right)}{\Phi_{0,1}\left(\frac{\log(C)-\mu}{\sigma}   \right)},  
\end{equation}
where $\Phi_{0,1}(\cdot)$ is the cumulative distribution function of a standard normal distribution.
\end{Lem} 
\begin{proof}
\newcommand{\dx}{\mathrm{d}x}
\newcommand{\dz}{\mathrm{d}z}
The density function of $Y$ is given by
$$ f(x)=\frac{\exp\left( \frac{-(\log(x)-\mu)^2}{2\sigma^2} \right)}{x\sigma \sqrt{2\pi}}, \quad x>0.   $$
Thus, $f(x) = \frac{\phi_{\mu,\sigma}(\log(x))}{x},$ where  $\phi_{\mu,\sigma}(\cdot)$ is the density function of a normal distribution with mean $\mu$ and standard deviation $\sigma.$
Next, note that the density function of $Y$ conditioned on $Y\leq C$ is $f(x \vert x\leq C) = \frac{f(x)}{F(C)},$ where $F(\cdot)$ is the cumulative distribution function of $Y$ and given by  $ F(x) = \Phi_{0,1}\left( \frac{\log(x)-\mu}{\sigma}   \right)$  for any $x\in \R.$
With this,
\begin{align*}
    \expectation{Y \vert Y \leq C} 
    &= \int_{0}^C x f(x \vert x \leq C) \, \dx \\
    &= \frac{1}{\Phi_{0,1}\left( \frac{\log(C)-\mu}{\sigma}   \right)} \int_{0}^C  \phi_{\mu,\sigma}(\log(x))  \, \dx \\
    &= \frac{\exp(\mu+\nicefrac{\sigma^2}{2} )}{\Phi_{0,1}\left( \frac{\log(C)-\mu}{\sigma}   \right)} \int_{-\infty}^{  \frac{\log(C)-\mu}{\sigma} }  \phi_{0,1}(z-\sigma)  \, \dz \\
    &= \exp(\mu+\nicefrac{\sigma^2}{2} )\frac{\Phi_{0,1}\left(\frac{\log(C)-\mu-\sigma^2}{\sigma}\right)}{\Phi_{0,1}\left( \frac{\log(C)-\mu}{\sigma}   \right)}.
\end{align*}
Here, we used for the third inequality the substitution $z=\frac{\log(x)-\mu}{\sigma},$ so that
$\exp(\sigma z+\mu)\sigma \, \dz = \dx$ and 
$$\exp\left( -\frac{(z-\sigma)^2}{2} \right) \exp\left( \frac{\sigma^2}{2} \right) = \exp\left( -\frac{z^2}{2}  \right)\exp(\sigma z). $$
\end{proof}
\noindent Recalling the modelling assumption on the runtimes made in \eqref{defi_exp_runtime_model} as well as assumption (\textbf{A2}), we obtain that $ m_{i_t,a} \sim \mathrm{LN}(\vec{f}_{i_t}^\intercal\vec{\theta}^*_a,\sigma^2).$
Using Lemma \ref{lemma_cond_exp_log_norm} and that $C_{i_t,a}^{(1)} = \nicefrac{( \log(C)- \vec{f}_{i_t}^\intercal\vec{\theta}^*_a -\sigma^2 )}{\sigma} ,$  $C_{i_t,a}^{(2)} = \nicefrac{( \log(C)- \vec{f}_{i_t}^\intercal\vec{\theta}^*_a )}{\sigma},$ we can derive \eqref{eq:expected_loss_3} as follows:
\begin{align*}
\begin{split}
    \expectation{ l_{t,a} \vert  \vec{f}_{i_t} } 
    &= \expectation{l_{t,a} \vert m_{i_t,a} \leq C , \vec{f}_{i_t}} \cdot \mathbb{P}(m_{i_t,a} \leq C \vert \vec{f}_{i_t}) 
    + \expectation{l_{t,a} \vert m_{i_t,a} > C , \vec{f}_{i_t}} \cdot \mathbb{P}(m_{i_t,a} > C \vert \vec{f}_{i_t}) \\ 
    &= \exp\left( \vec{f}_{i_t}^\intercal\vec{\theta}^*_a + \frac{\sigma^2}{2}\right) \cdot \left(   \frac{\Phi_{0,1}(C_{i_t,a}^{(1)} )}{\Phi_{0,1}( C_{i_t,a}^{(2)}  )}  \right) \cdot \mathbb{P}(m_{i_t,a} \leq C \vert \vec{f}_{i_t} ) 
    + \mathcal{P}(C) \cdot \mathbb{P}(m_{i_t,a} > C \vert \vec{f}_{i_t}) \\
    &= \exp\left( \vec{f}_{i_t}^\intercal\vec{\theta}^*_a + \frac{\sigma^2}{2}\right) \cdot \left(   \frac{\Phi_{0,1}(C_{i_t,a}^{(1)} )}{\Phi_{0,1}( C_{i_t,a}^{(2)}  )}  \right) 
    + \mathbb{P}(m_{i_t,a} > C \vert \vec{f}_{i_t})\cdot \left( \mathcal{P}(C) - \exp\left( \vec{f}_{i_t}^\intercal\vec{\theta}^*_a + \frac{\sigma^2}{2}\right) \cdot  \frac{\Phi_{0,1}(C_{i_t,a}^{(1)} )}{\Phi_{0,1}( C_{i_t,a}^{(2)}  )} 
    \right) \\
    %
        %
    &= E_{C}(\vec{f}_{i_t}^\intercal\vec{\theta}^*_a,\sigma) 
    + \big(1-\Phi_{\vec{f}_{i_t}^\intercal\vec{\theta}^*_a,\sigma}(\log(C) )\big) \cdot \left( \mathcal{P}(C) - E_{C}(\vec{f}_{i_t}^\intercal\vec{\theta}^*_a,\sigma) \right),
\end{split}
\end{align*} 
where $  E_{C}(\vec{f}_{i_t}^\intercal\vec{\theta}^*_a,\sigma) = \exp( \vec{f}_{i_t}^\intercal\vec{\theta}^*_a + \nicefrac{\sigma^2}{2}) \cdot \frac{\Phi_{0,1}(C_{i_t,a}^{(1)} )}{\Phi_{0,1}( C_{i_t,a}^{(2)} )}.$

\section{ASlib Overview}\label{sec:aslib_overview}
ASlib \cite{as_lib} is a curated collection of over 25 different algorithm selection problems, called scenarios, based on different algorithmic problem classes such as SAT, TSP, CSP. Each scenario is made up of several instances for which the performance of a set of algorithms has been evaluated using a certain cutoff to avoid unacceptably long algorithm runs. Table \ref{tab:aslib_scenario_table} gives an overview of the examined scenarios including relevant statistics.

\begin{table}[ht]
    \centering
    \begin{tabular}{lrrrrrrrr}
\toprule
Scenario &    \#I &   \#F &  \#A &            C &    T &   MT & minT & maxT \\
\midrule
ASP-POTASSCO           &  1294 &  138 &  11 &       600.00 & 0.20 & 0.20 & 0.14 & 0.28 \\
BNSL-2016              &  1179 &   86 &   8 &      7200.00 & 0.28 & 0.18 & 0.12 & 0.59 \\
CPMP-2015              &   527 &   22 &   4 &      3600.00 & 0.28 & 0.28 & 0.19 & 0.37 \\
CSP-2010               &  2024 &   86 &   2 &      5000.00 & 0.20 & 0.20 & 0.14 & 0.25 \\
CSP-MZN-2013           &  4642 &  155 &  11 &      1800.00 & 0.70 & 0.70 & 0.48 & 0.87 \\
CSP-Minizinc-Time-2016 &   100 &   95 &  20 &      1200.00 & 0.50 & 0.52 & 0.28 & 0.82 \\
GRAPHS-2015            &  5725 &   35 &   7 & 1e+08.00 & 0.07 & 0.04 & 0.03 & 0.27 \\
MAXSAT-PMS-2016        &   601 &   37 &  19 &      1800.00 & 0.39 & 0.19 & 0.12 & 0.96 \\
MAXSAT-WPMS-2016       &   630 &   37 &  18 &      1800.00 & 0.58 & 0.46 & 0.21 & 0.96 \\
MAXSAT12-PMS           &   876 &   37 &   6 &      2100.00 & 0.41 & 0.43 & 0.23 & 0.57 \\
MAXSAT15-PMS-INDU      &   601 &   37 &  29 &      1800.00 & 0.49 & 0.42 & 0.12 & 0.96 \\
MIP-2016               &   218 &  143 &   5 &      7200.00 & 0.20 & 0.10 & 0.04 & 0.45 \\
PROTEUS-2014           &  4021 &  198 &  22 &      3600.00 & 0.60 & 0.63 & 0.37 & 0.67 \\
QBF-2011               &  1368 &   46 &   5 &      3600.00 & 0.55 & 0.51 & 0.42 & 0.72 \\
QBF-2014               &  1254 &   46 &  14 &       900.00 & 0.56 & 0.56 & 0.37 & 0.71 \\
QBF-2016               &   825 &   46 &  24 &      1800.00 & 0.36 & 0.31 & 0.20 & 0.61 \\
SAT03-16\_INDU          &  2000 &  483 &  10 &      5000.00 & 0.25 & 0.24 & 0.19 & 0.32 \\
SAT11-HAND             &   296 &  115 &  15 &      5000.00 & 0.61 & 0.62 & 0.50 & 0.68 \\
SAT11-INDU             &   300 &  115 &  18 &      5000.00 & 0.33 & 0.33 & 0.28 & 0.39 \\
SAT11-RAND             &   600 &  115 &   9 &      5000.00 & 0.47 & 0.45 & 0.40 & 0.58 \\
SAT12-ALL              &  1614 &  115 &  31 &      1200.00 & 0.54 & 0.53 & 0.25 & 0.74 \\
SAT12-HAND             &   767 &  115 &  31 &      1200.00 & 0.67 & 0.64 & 0.52 & 0.93 \\
SAT12-INDU             &  1167 &  115 &  31 &      1200.00 & 0.50 & 0.36 & 0.26 & 0.99 \\
SAT12-RAND             &  1362 &  115 &  31 &      1200.00 & 0.73 & 0.87 & 0.27 & 0.98 \\
SAT15-INDU             &   300 &   54 &  28 &      3600.00 & 0.24 & 0.21 & 0.13 & 0.74 \\
TSP-LION2015           &  3106 &  122 &   4 &      3600.00 & 0.10 & 0.03 & 0.00 & 0.32 \\
\bottomrule
\end{tabular}
    \caption{Overview of ASlib scenarios including their number of instances (\#I), number of features (\#F), number of algorithms (\#A), cutoff time (C), average (T), median (MT), minimum (minT) and maximum (maxT) relative number of timeouts per algorithms.}
    \label{tab:aslib_scenario_table}
\end{table}

\section{Software and Hyperparameter Settings}\label{sec:hyperparams}
All experiments are based on Python 3 implementations. A complete list of used packages and the corresponding version number can be found on Github. Below we give a short (non-exhaustive) list of the most important software used for the corresponding approaches and the hyperparameter settings used for the evaluation.

\subsection{Our Approaches}

\subsubsection{Software}
All presented approaches (LinUCB and Thompson variants) were implemented in Python by using scipy\footnote{\url{https://www.scipy.org/}} and numpy\footnote{\url{https://numpy.org/}}.

\subsubsection{Hyperparameter Settings}
If not stated differently at the beginning of the corresponding experiment (e.g., sensitivity analysis), the following hyperparameters were used:

\begin{itemize}
    \item Thompson variants
    \begin{itemize}
        \item $\sigma = 1.0$
        \item $\lambda = 0.5$
    \end{itemize}
    \item LinUCB variants
    \begin{itemize}
        \item $\lambda = 1.0$  
        \item $\alpha = 1.0$
        \item $\sigma = 10.0$ (for the \_rev variants)
        \item $\widetilde{\sigma}^2 = 0.25$ (for the rand\_ variants)
    \end{itemize}
\end{itemize}
The values of these parameters were chosen according to a hyperparameter sensitivity analysis (cf. Sec. \ref{sec:app_sensitivity_analysis}).

\subsubsection{Caveat}
All Thompson variants rely on sampling from the multi-variate normal distribution, which we implemented using the 'np.random.multivariate\_normal' method. Unfortunately, this method seems to have a bug, which is caused by the underlying BLAS implementation of the corresponding SVD, which is performed as part of the method. Changing to various versions of numpy and BLAS did not solve the problem for us. As a consequence, some of the repetitions of the experiments of some scenarios did not complete for some Thompson variants. Below you can find a table indicating how many repetitions are \textit{missing} for the corresponding variant on the corresponding scenario. We reported the bug and hope to add the missing values to the experimental results once the bug is fixed. However, due to the very few amount of data points missing, we do not assume a relevant change for the final results. 

\begin{center}
\begin{tabular}{c|c|c}
     \toprule
     Scenario & Approach & \#Missing seeds  \\ 
     \midrule
     CSP-MZN-2013 & bj\_thompson & 2 \\
     PROTEUS-2014 & bj\_thompson & 1 \\
     PROTEUS-2014 & thompson\_rev & 1 \\
     PROTEUS-2014 & thompson & 2 \\
     SAT03-16\_INDU & bj\_thompson\_rev & 2 \\
     SAT03-16\_INDU & bj\_thompson & 1 \\
     SAT03-16\_INDU & thompson\_rev & 1\\
     SAT03-16\_INDU & thompson & 3\\
     SAT12-RAND & bj\_thompson\_rev & 1 \\
     TSP-LION2015 & bj\_thompson & 1\\
     \bottomrule
\end{tabular}
\end{center}

\subsection{Degroote Approach}

\subsubsection{Software}
We re-implemented the Degroote approach using scikit-learn\footnote{\url{https://scikit-learn.org/stable/}} in Python. In particular the linear model is implemented using the LinearRegression estimator from scikit-learn. 

\subsubsection{Hyperparameter Settings}
\begin{itemize}
    \item Epsilon-Greedy linear regression
    \begin{itemize}
         \item $\epsilon = 0.05$ 
    \end{itemize}
    \item the hyperparameters of underlying models from scikit-learn were set according to their default values
\end{itemize}
The value of the hyperparameter $\epsilon$ is as suggested by the authors in \cite{online_as_main_degrooteCBK18}.

\section{Sensitivity Analysis}\label{sec:app_sensitivity_analysis}
In this section we provide a sensitivity analysis of the most important hyperparameters of our presented approaches. To keep the amount of experiments to perform in a reasonable dimension, we limit ourselves to the most advanced variant of both LinUCB and Thompson we presented in this work. Moreover, we selected six scenarios from the ASlib covering a range of algorithmic problems, number of instances and features for this analysis. All figures described in the following display the average PAR10 score over ten seeds for different settings of the corresponding hyperparameter. The error bars indicate the corresponding standard deviation.

\subsection{Thompson Variants}
Figure \ref{fig:app_sensitivity_bj_thompson_rev_sigma} displays the average PAR10 score over ten seeds for different settings of $\sigma$ on a selection of scenarios where $\lambda=0.5$ is fixed. Overall, small values of $\sigma$ tend to lead to better results indicating that sampling $\vec{\widetilde{\theta}}_a$, i.e., our belief about the weight vector according to the posterior distribution, should be based on a rather small variance and hence, not too much exploration. This is quite in line with our findings regarding the amount of exploration of the LinUCB variants.

Figure \ref{fig:app_sensitivity_bj_thompson_rev_lambda} displays the average PAR10 score over $10$ seeds for different settings of $\lambda$ on a selection of scenarios where $\sigma=10$ is fixed. Overall a clear trend whether small or large values of $\lambda$ lead to good results seems hard to detect indicating that the performance is rather robust with respect to the choice of $\lambda$.

\subsection{LinUCB Variants}
Figure \ref{fig:app_sensitivity_rand_bclinucb_rev_sigma} displays the average PAR10 score over ten seeds for different settings of $\sigma$ on a selection of scenarios where $\lambda=0.5$ and $\widetilde{\sigma}^2 = 0.25$ are fixed. In contrast to the Thompson variants previously discussed, where small values of $\sigma$ tend to lead to better results, here, large values of $\sigma$ tend to lead to better PAR10 scores indicating that the noise terms (defined in \eqref{defi_exp_runtime_model}) have a large standard deviation.

Figure \ref{fig:app_sensitivity_rand_bclinucb_rev_alpha} displays the average PAR10 score over ten seeds for different settings of $\alpha$ on a selection of scenarios where $\sigma=1$ and $\widetilde{\sigma}^2 = 0.25$ are fixed. Overall, no clear trend can be observed whether small or large values of $\alpha$ lead to better results.

Figure \ref{fig:app_sensitivity_rand_bclinucb_rev_randsigma} displays the average PAR10 score over ten seeds for different settings of $\widetilde{\sigma}^2$ on a selection of scenarios where $\alpha=1$ and $\lambda = 0.5$ are fixed. Once again, overall no clear trend can be observed whether small or large values of $\widetilde{\sigma}$ lead to good results, due to the wide error bars.

\begin{figure}[htb]
	\centering
\begin{subfigure}{0.25\textwidth}
	\includegraphics[width=\linewidth]{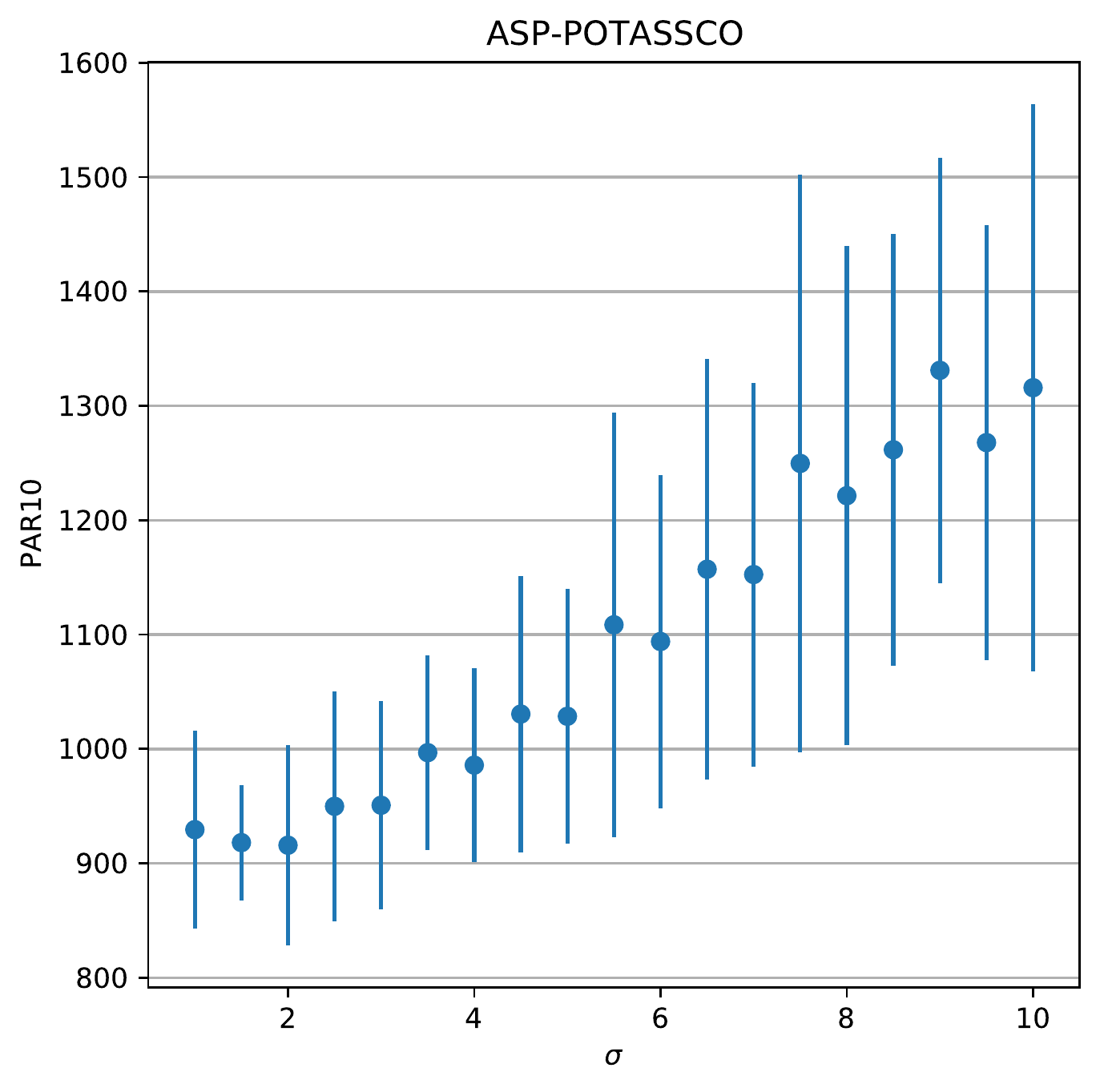}
	\label{fig:app_sensitivity_ASP-POTASSCO_sigma_bj_thompson_rev}
\end{subfigure}
\begin{subfigure}{0.25\textwidth}
	\includegraphics[width=\linewidth]{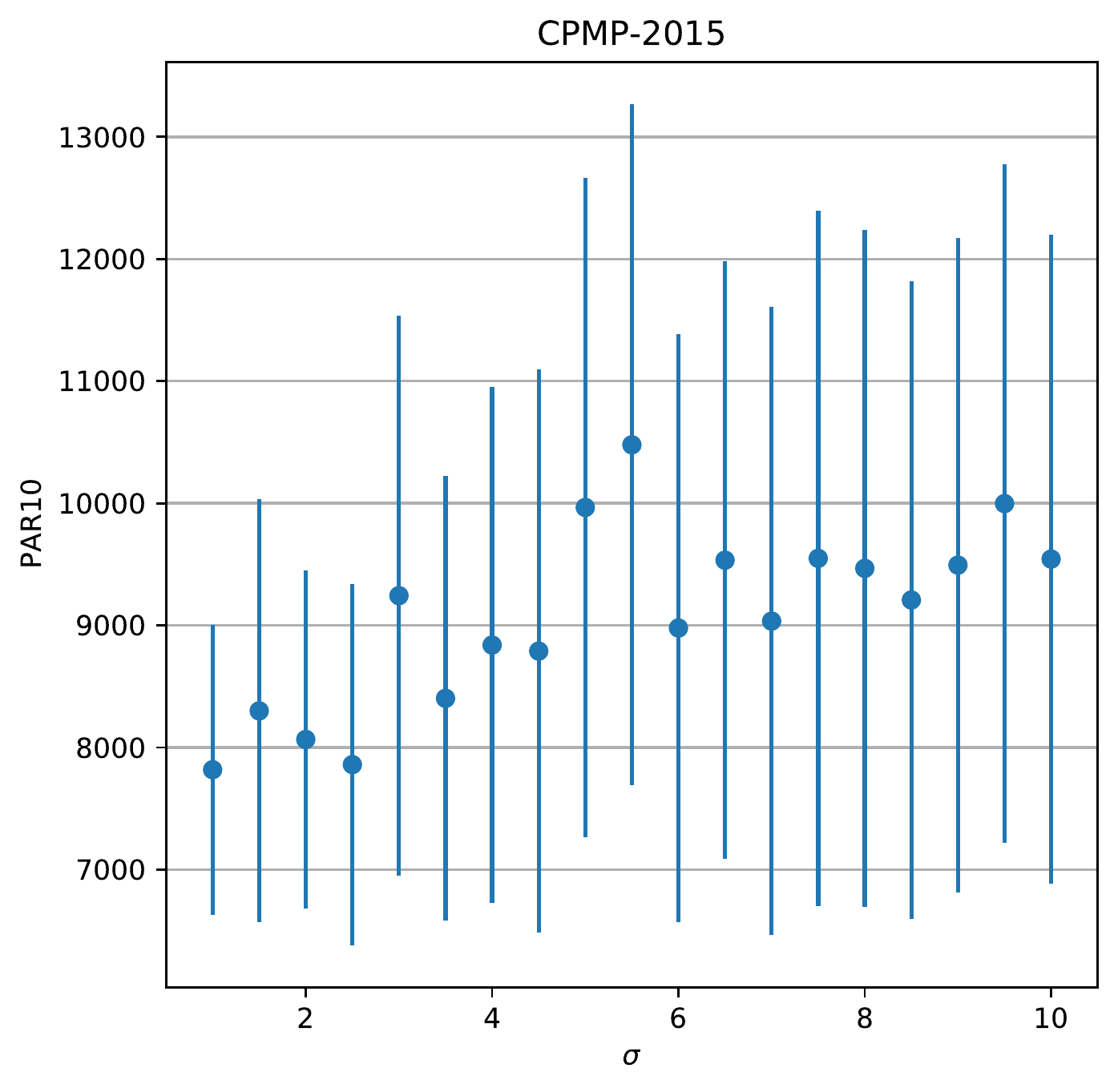}
	\label{fig:app_sensitivity_CPMP-2015_sigma_bj_thompson_rev}
\end{subfigure}
\begin{subfigure}{0.25\textwidth}
	\includegraphics[width=\linewidth]{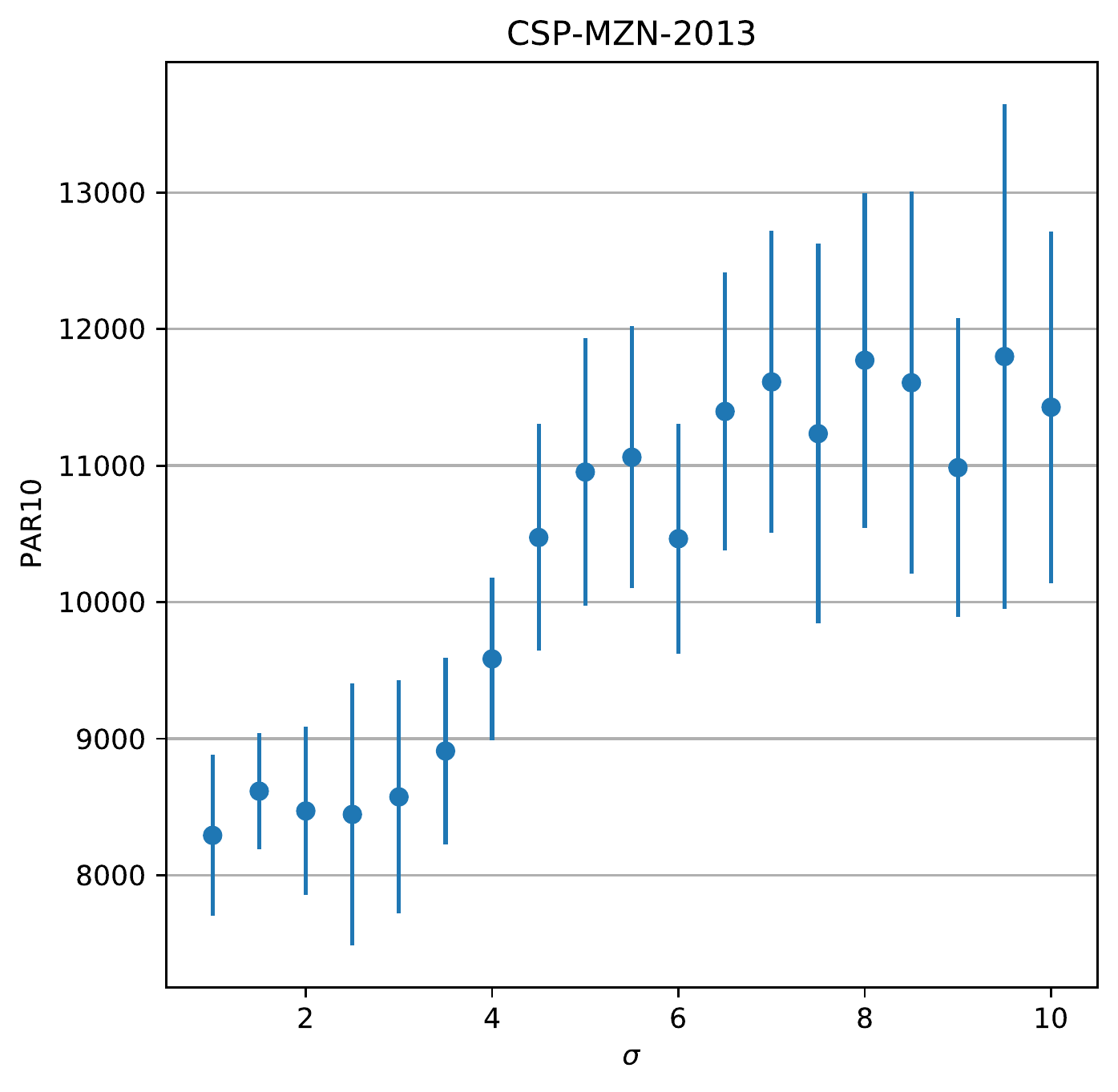}
	\label{fig:app_sensitivity_CSP-MZN-2013_sigma_bj_thompson_rev}
\end{subfigure}
\begin{subfigure}{0.25\textwidth}
	\includegraphics[width=\linewidth]{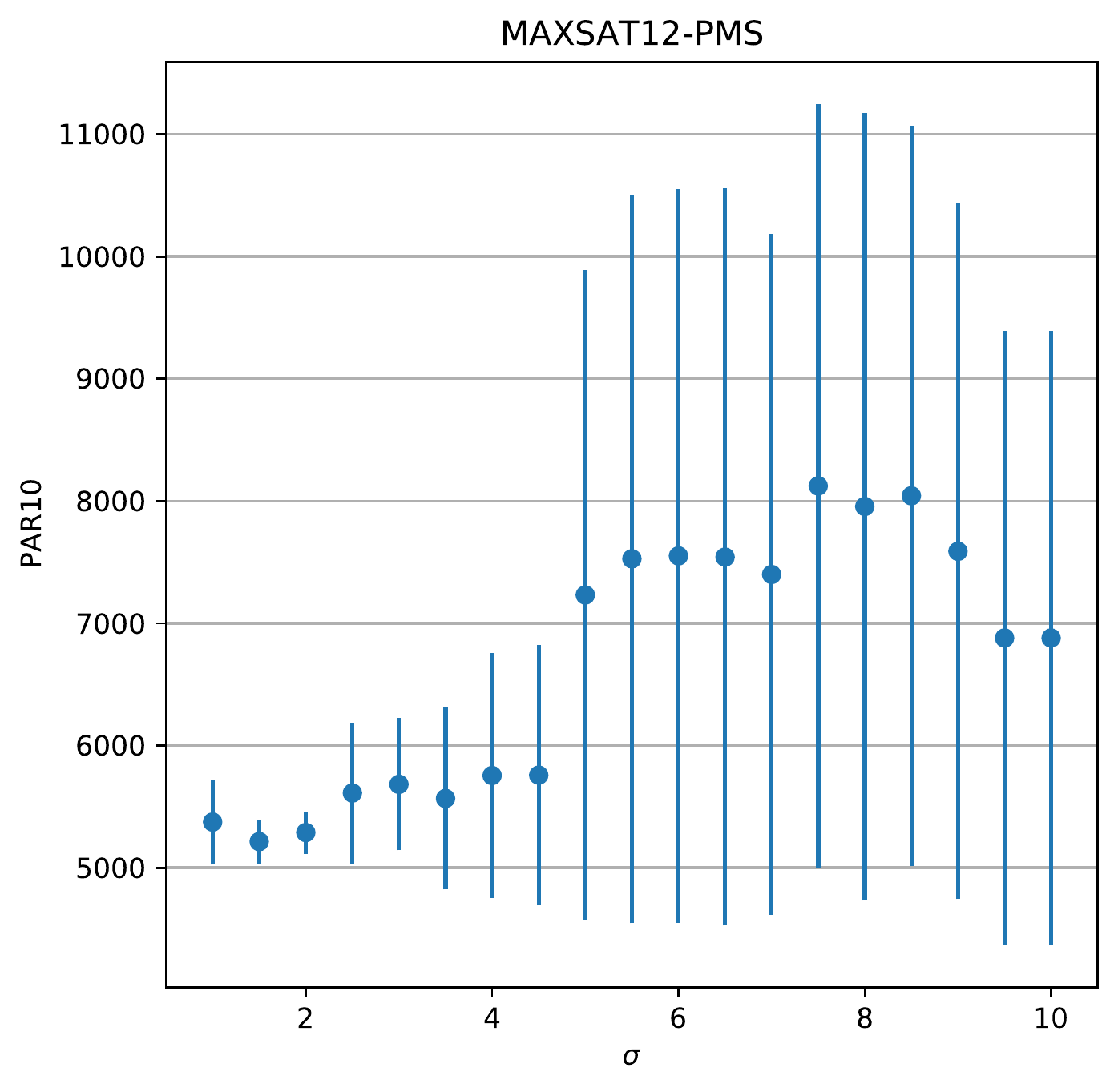}
	\label{fig:app_sensitivity_MAXSAT12-PMS_sigma_bj_thompson_rev}
\end{subfigure}
\begin{subfigure}{0.25\textwidth}
	\includegraphics[width=\linewidth]{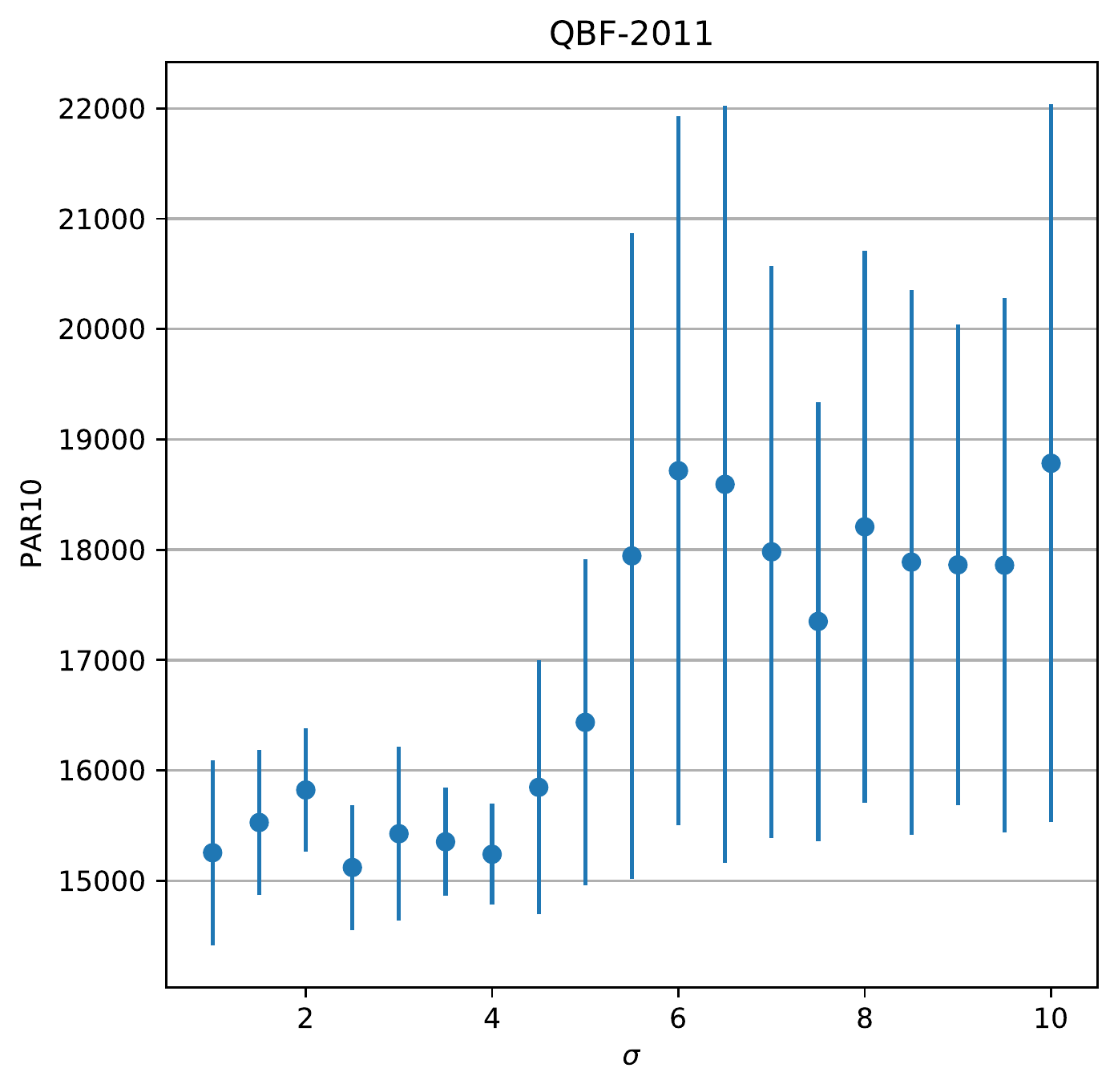}
	\label{fig:app_sensitivity_QBF-2011_sigma_bj_thompson_rev}
\end{subfigure}
\begin{subfigure}{0.25\textwidth}
	\includegraphics[width=\linewidth]{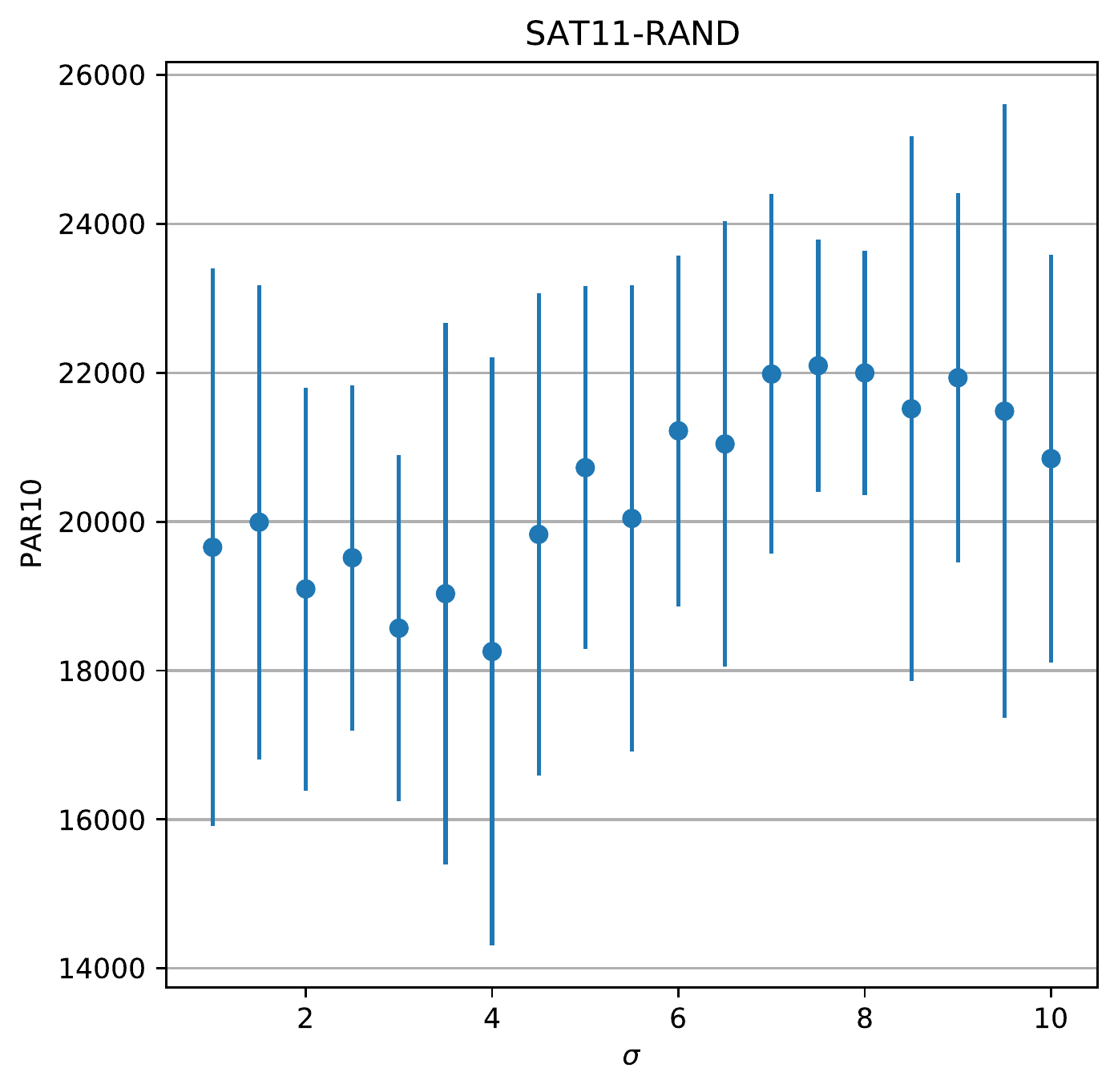}
	\label{fig:app_sensitivity_SAT11-RAND_sigma_bj_thompson_rev}
\end{subfigure}
\caption{Sensitivity analysis for parameter $\sigma$ of approach bj\_thompson\_rev.}
\label{fig:app_sensitivity_bj_thompson_rev_sigma}
\end{figure}

\begin{figure}[htb]
	\centering
\begin{subfigure}{0.25\textwidth}
	\includegraphics[width=\linewidth]{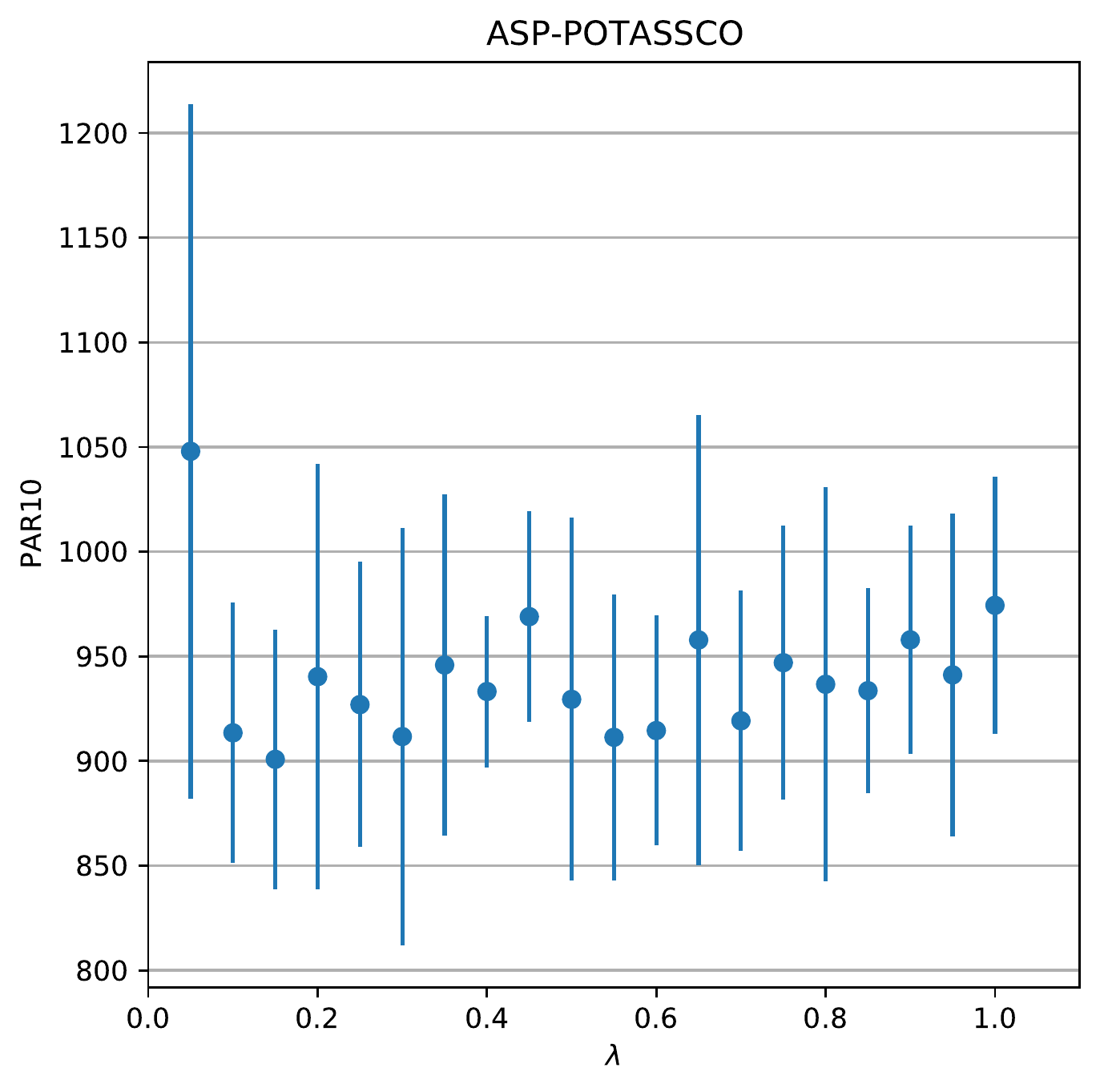}
	\label{fig:app_sensitivity_ASP-POTASSCO_lambda_bj_thompson_rev}
\end{subfigure}
\begin{subfigure}{0.25\textwidth}
	\includegraphics[width=\linewidth]{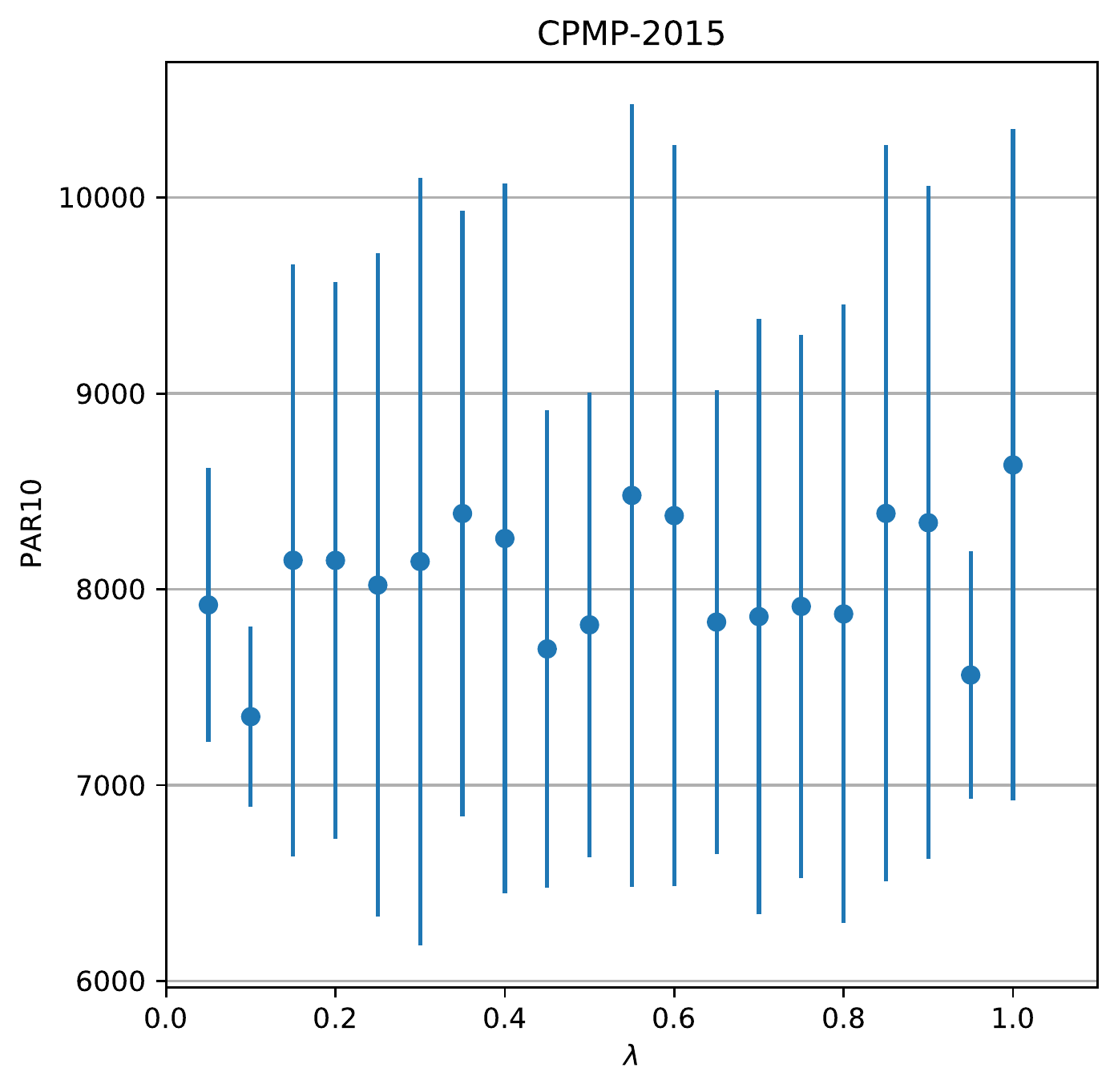}
	\label{fig:app_sensitivity_CPMP-2015_lambda_bj_thompson_rev}
\end{subfigure}
\begin{subfigure}{0.25\textwidth}
	\includegraphics[width=\linewidth]{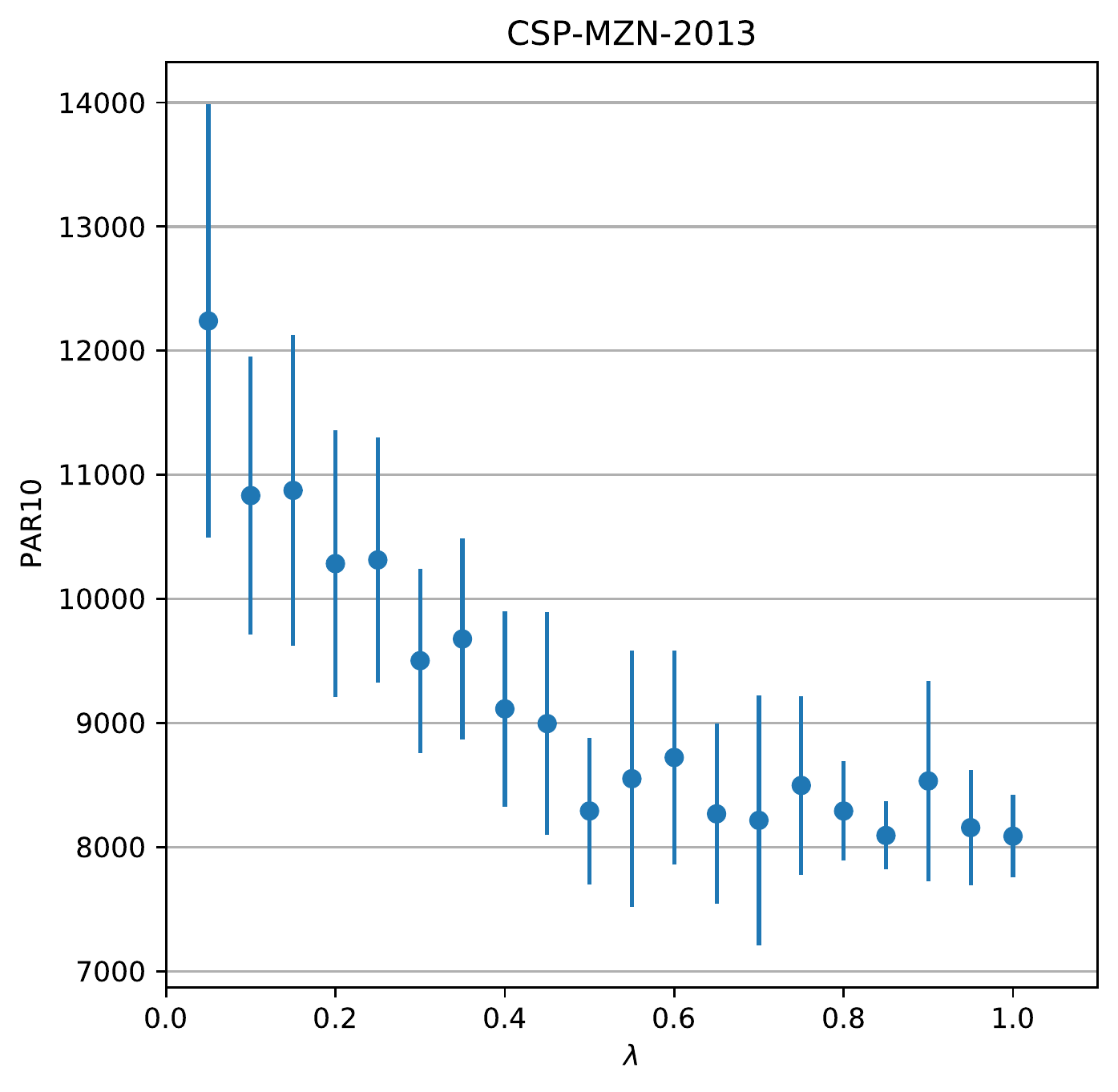}
	\label{fig:app_sensitivity_CSP-MZN-2013_lambda_bj_thompson_rev}
\end{subfigure}
\begin{subfigure}{0.25\textwidth}
	\includegraphics[width=\linewidth]{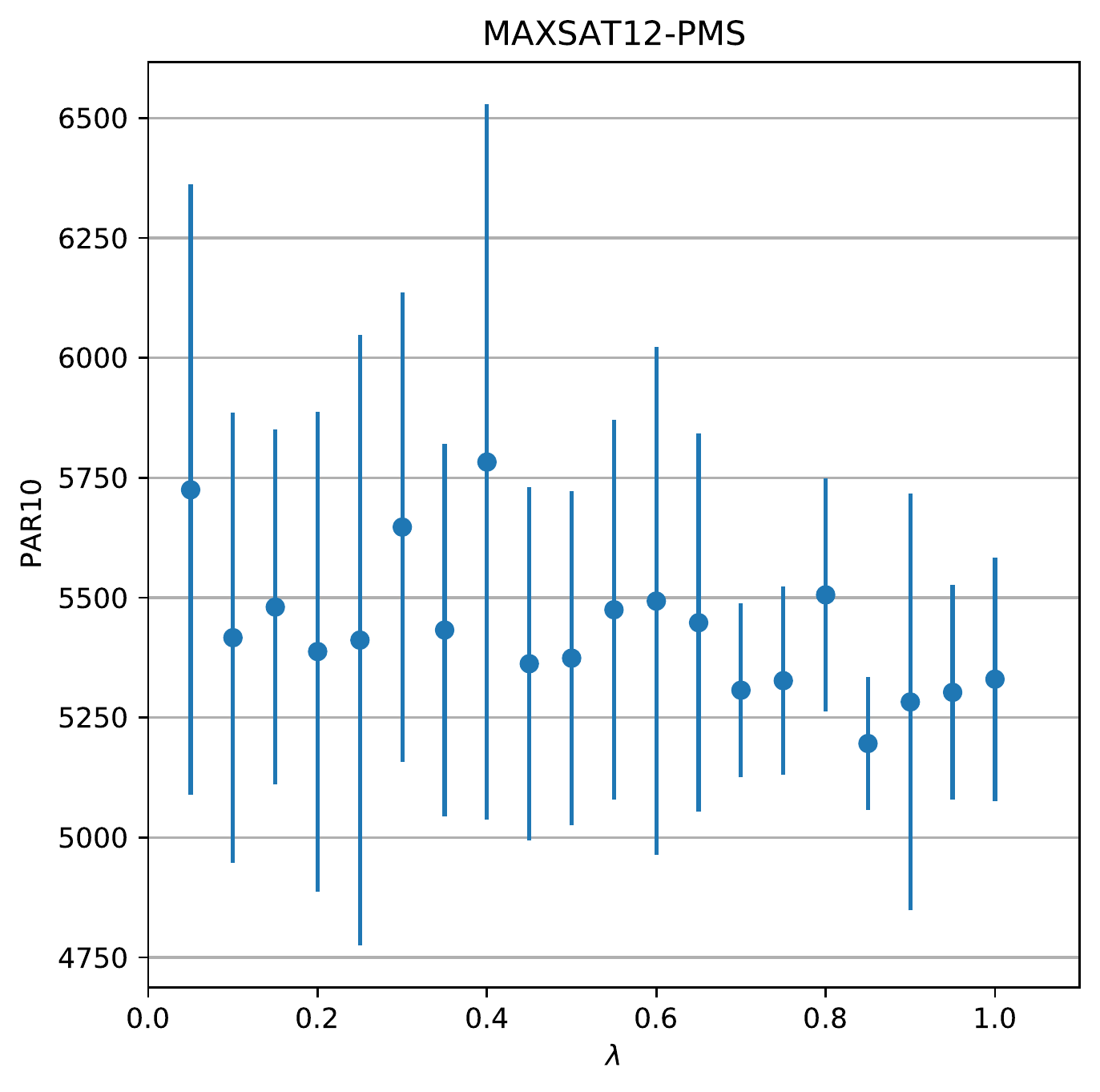}
	\label{fig:app_sensitivity_MAXSAT12-PMS_lambda_bj_thompson_rev}
\end{subfigure}
\begin{subfigure}{0.25\textwidth}
	\includegraphics[width=\linewidth]{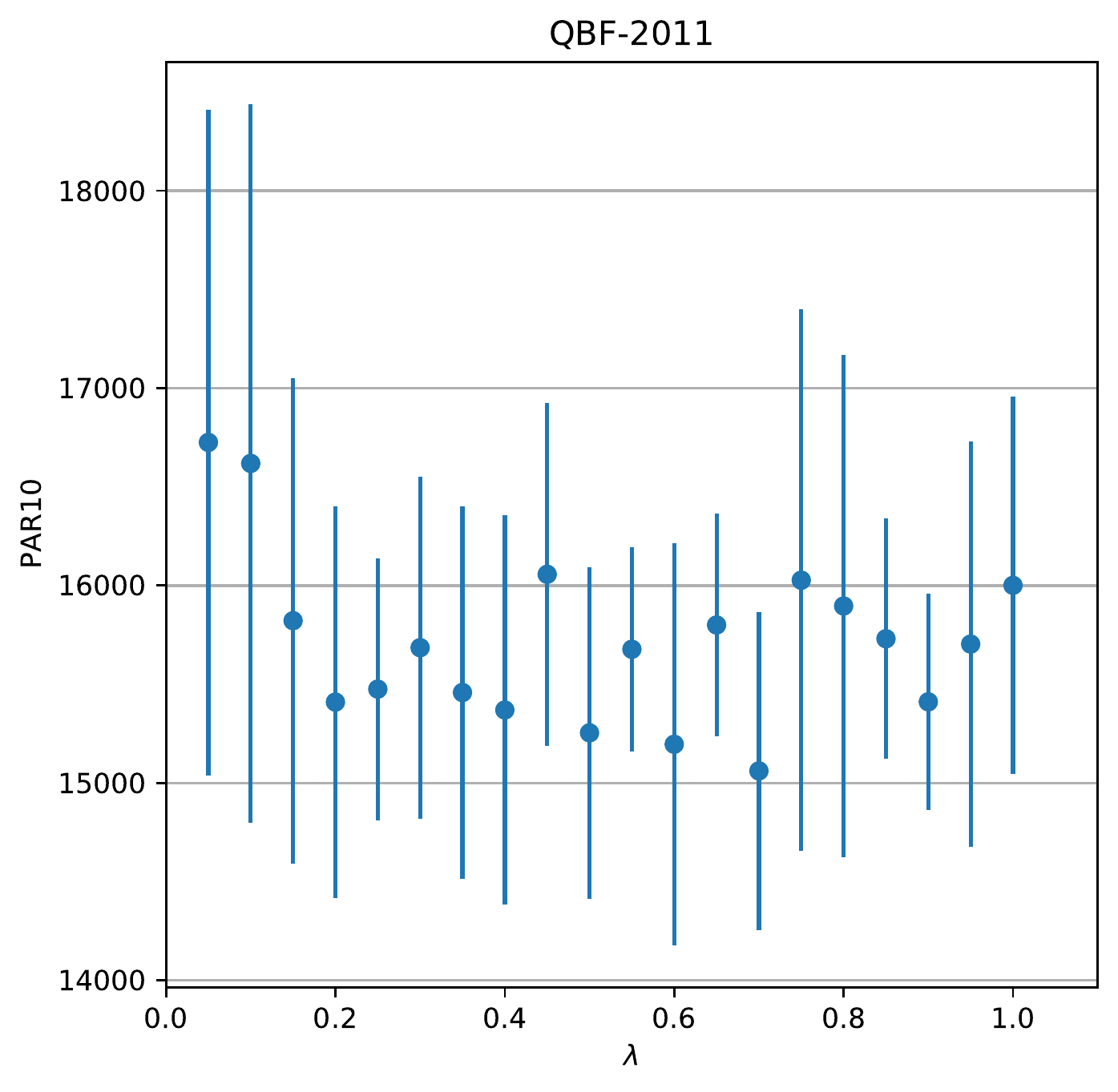}
	\label{fig:app_sensitivity_QBF-2011_lambda_bj_thompson_rev}
\end{subfigure}
\begin{subfigure}{0.25\textwidth}
	\includegraphics[width=\linewidth]{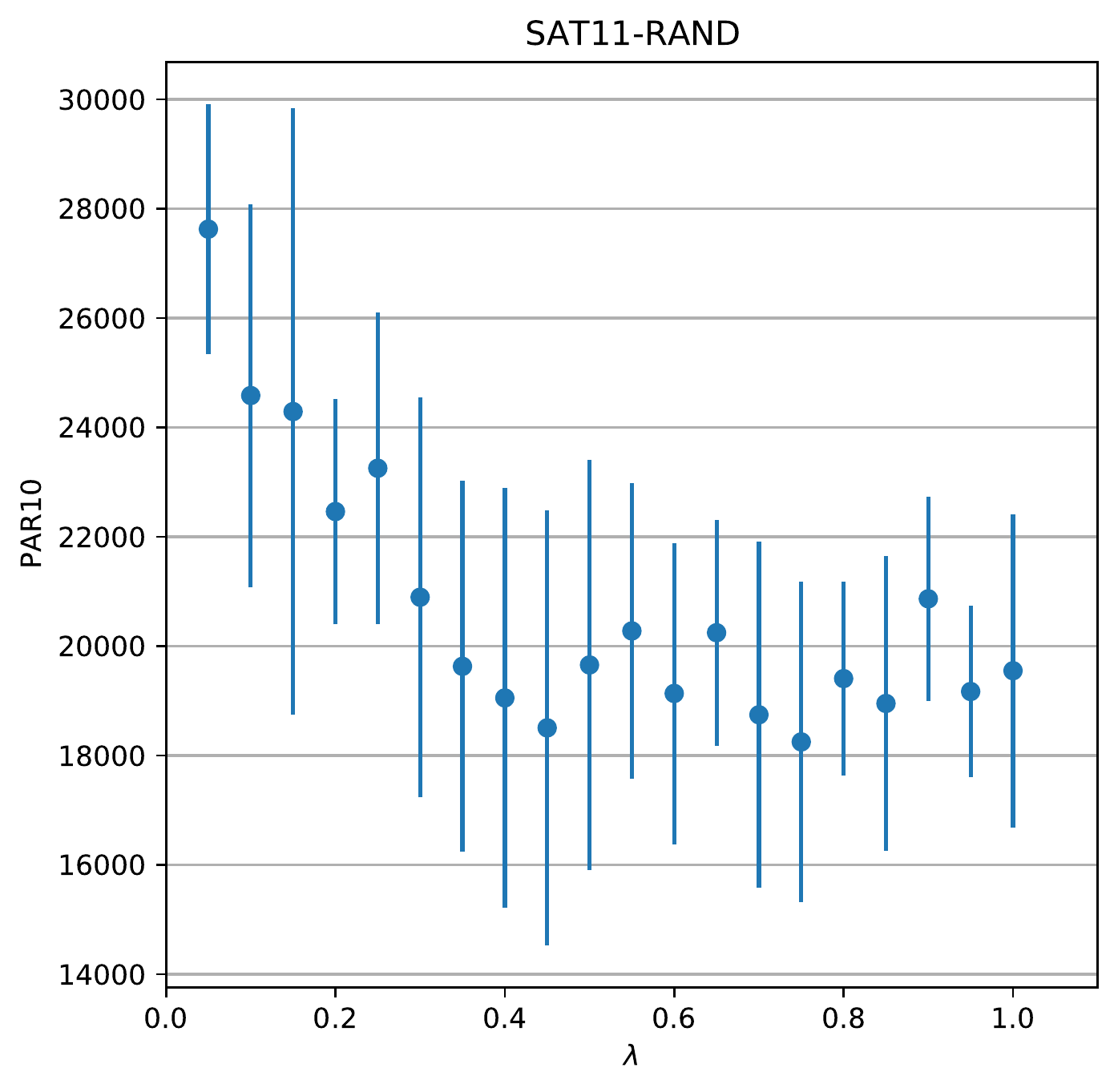}
	\label{fig:app_sensitivity_SAT11-RAND_lambda_bj_thompson_rev}
\end{subfigure}
\caption{Sensitivity analysis for parameter $\lambda$ of approach bj\_thompson\_rev.}
\label{fig:app_sensitivity_bj_thompson_rev_lambda}
\end{figure}

\begin{figure}[htb]
	\centering
\begin{subfigure}{0.25\textwidth}
	\includegraphics[width=\linewidth]{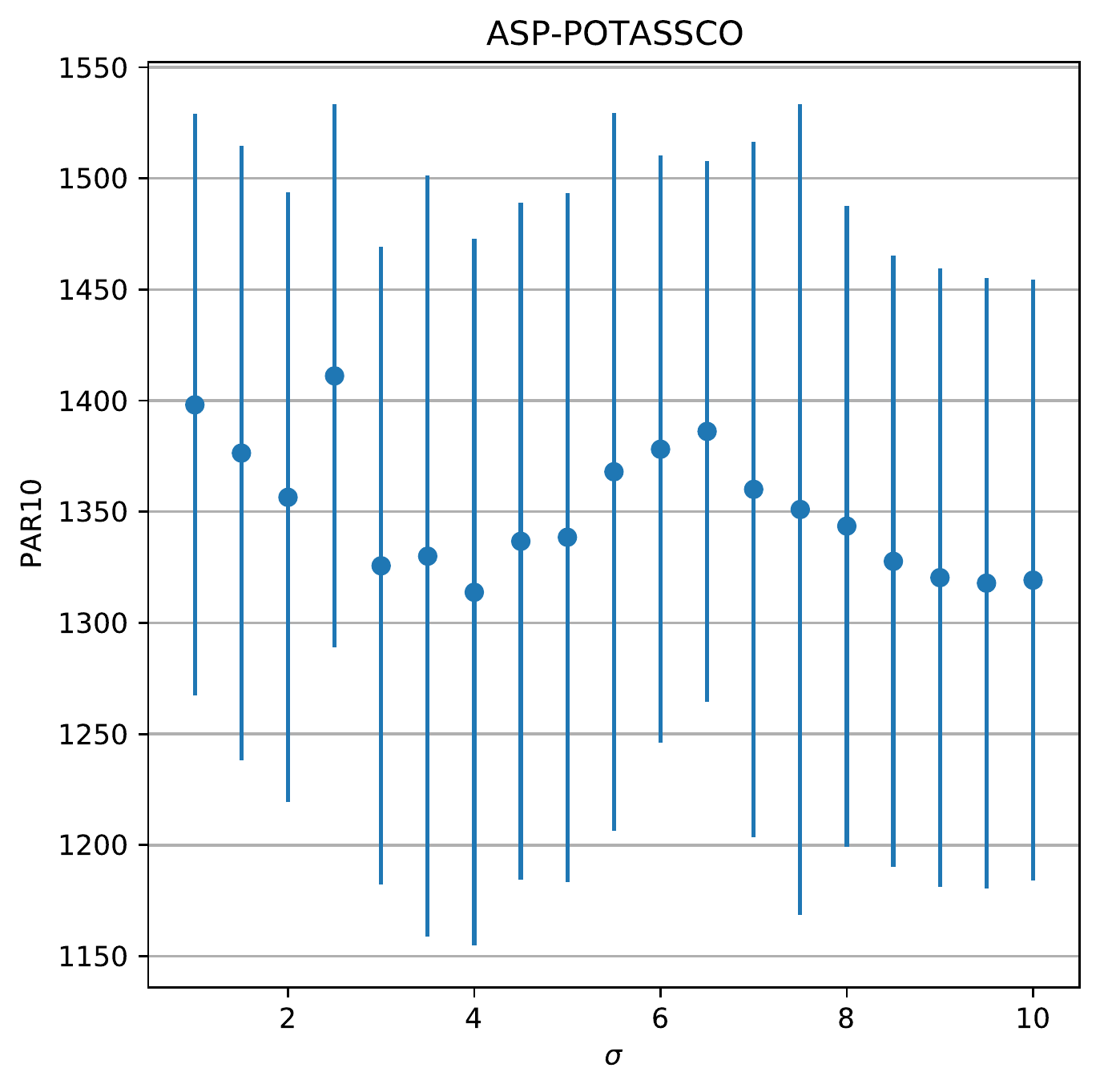}
	\label{fig:app_sensitivity_ASP-POTASSCO_sigma_rand_bclinucb_rev}
\end{subfigure}
\begin{subfigure}{0.25\textwidth}
	\includegraphics[width=\linewidth]{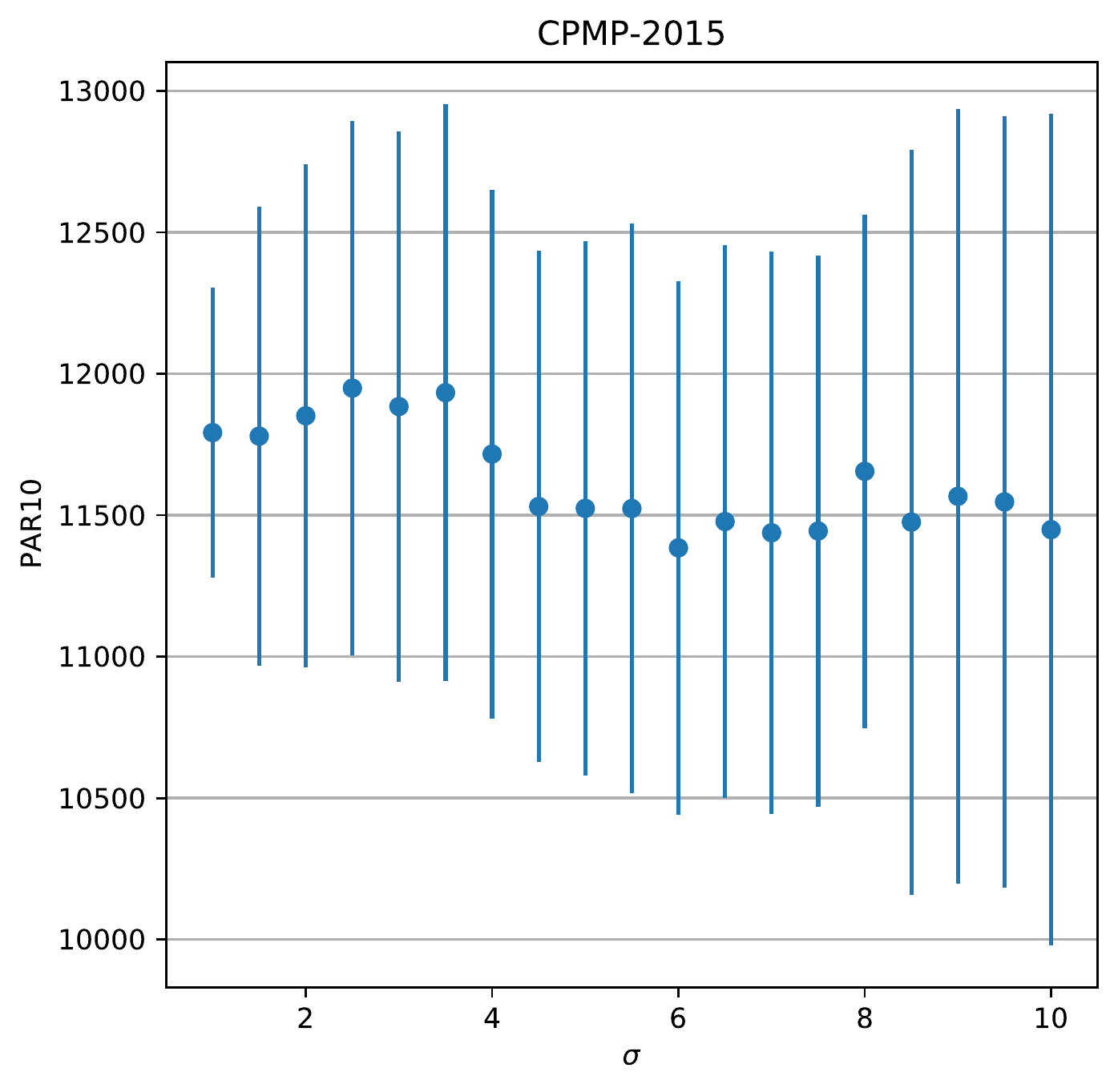}
	\label{fig:app_sensitivity_CPMP-2015_sigma_rand_bclinucb_rev}
\end{subfigure}
\begin{subfigure}{0.25\textwidth}
	\includegraphics[width=\linewidth]{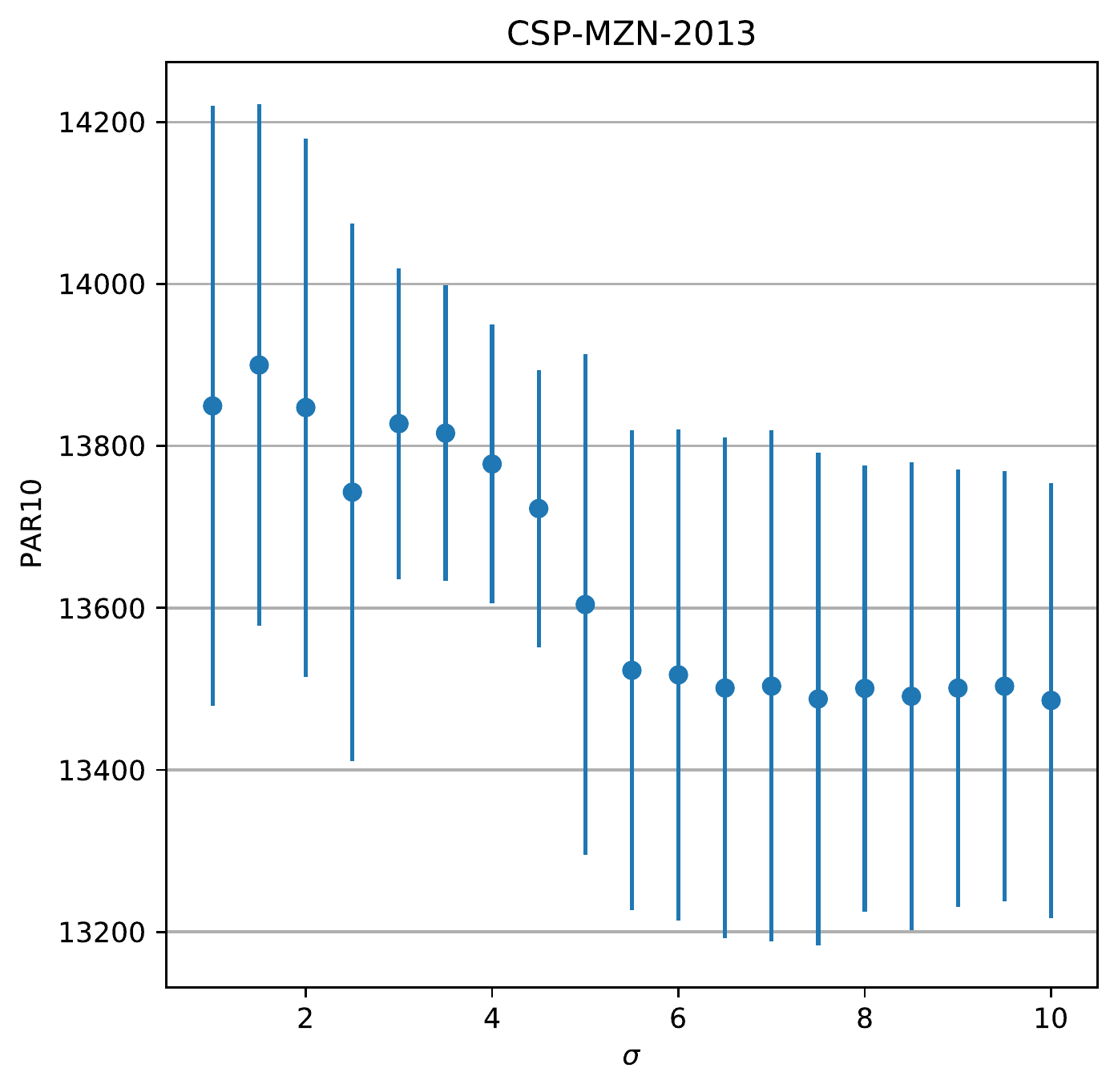}
	\label{fig:app_sensitivity_CSP-MZN-2013_sigma_rand_bclinucb_rev}
\end{subfigure}
\begin{subfigure}{0.25\textwidth}
	\includegraphics[width=\linewidth]{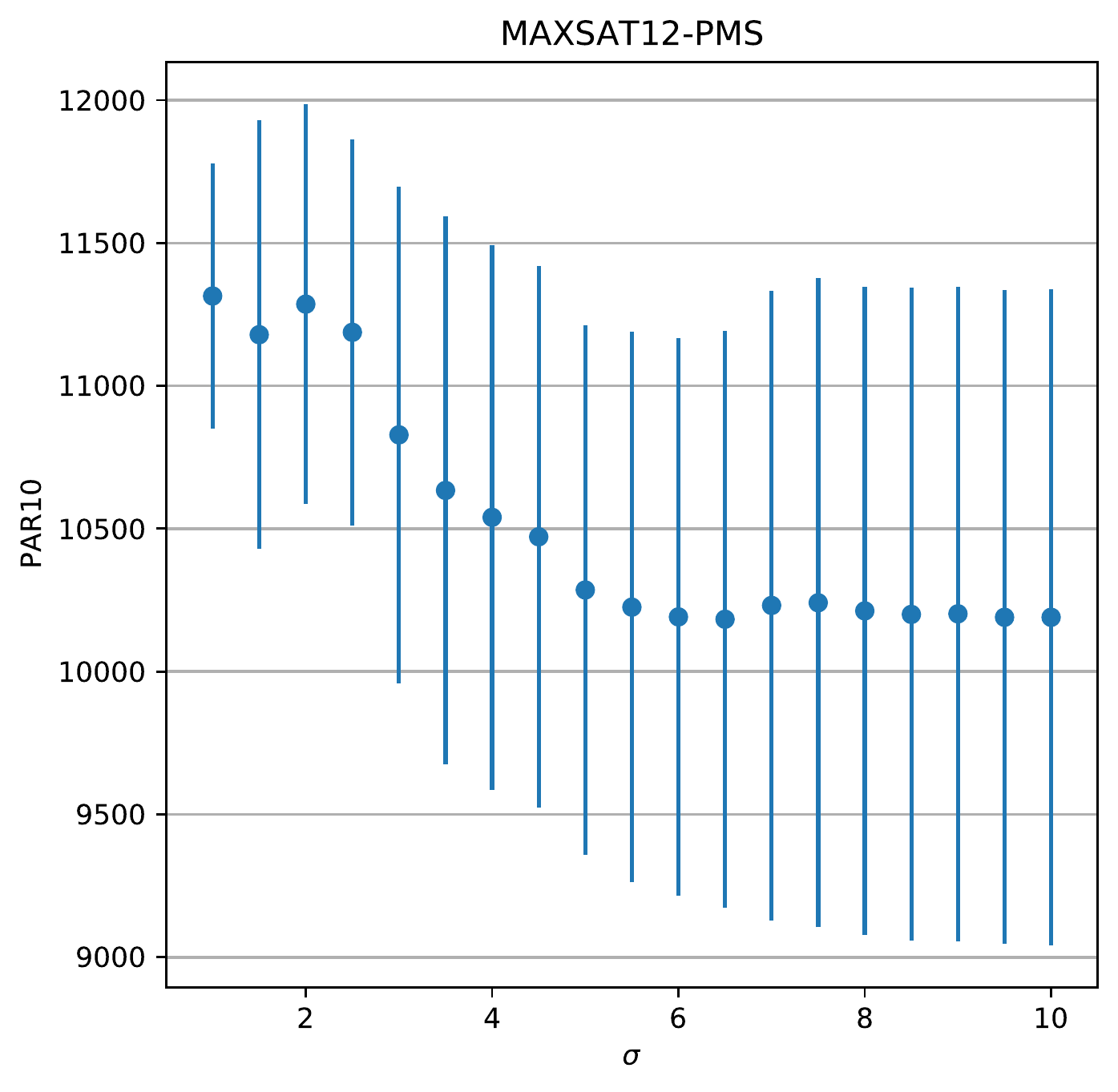}
	\label{fig:app_sensitivity_MAXSAT12-PMS_sigma_rand_bclinucb_rev}
\end{subfigure}
\begin{subfigure}{0.25\textwidth}
	\includegraphics[width=\linewidth]{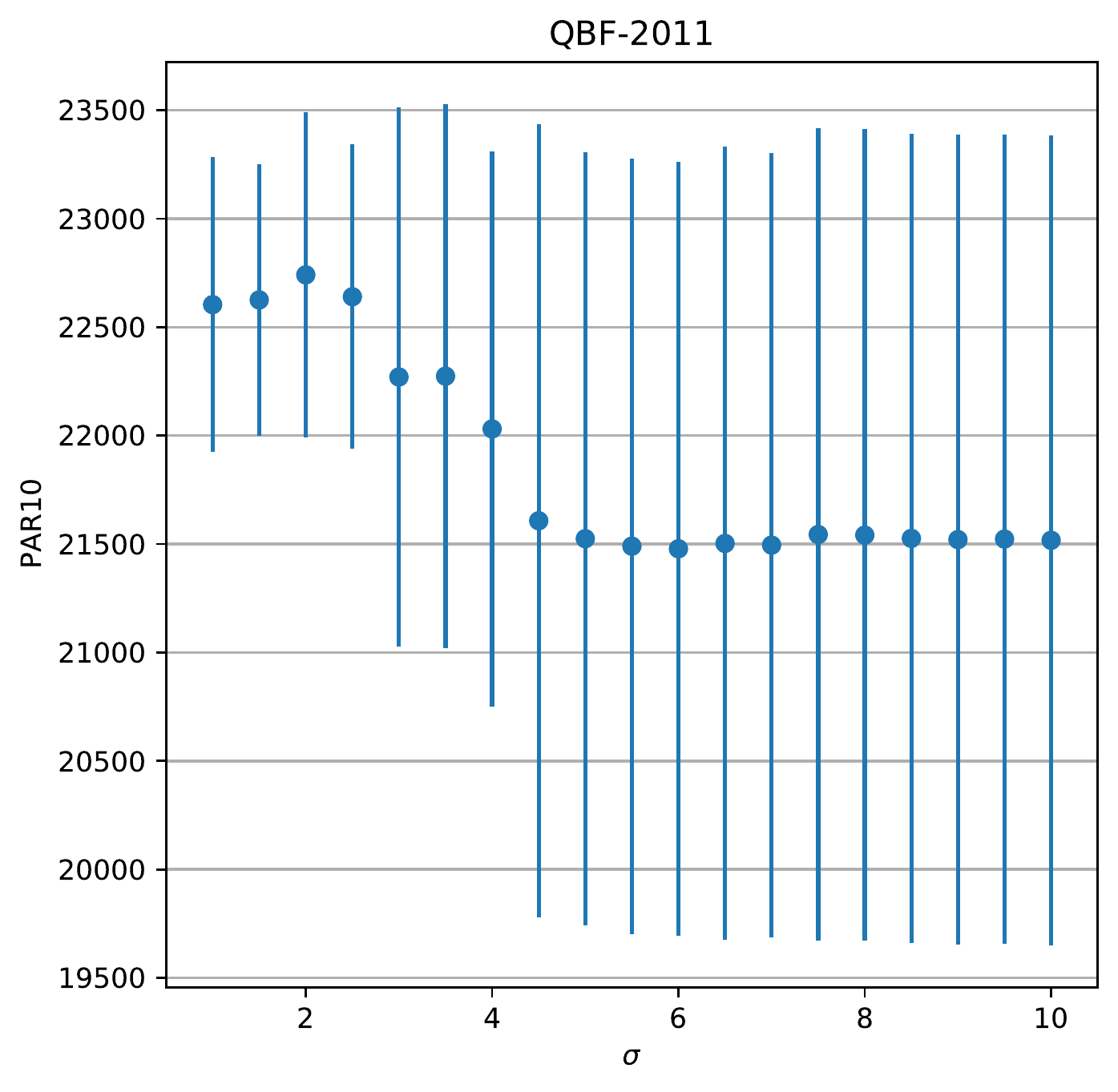}
	\label{fig:app_sensitivity_QBF-2011_sigma_rand_bclinucb_rev}
\end{subfigure}
\begin{subfigure}{0.25\textwidth}
	\includegraphics[width=\linewidth]{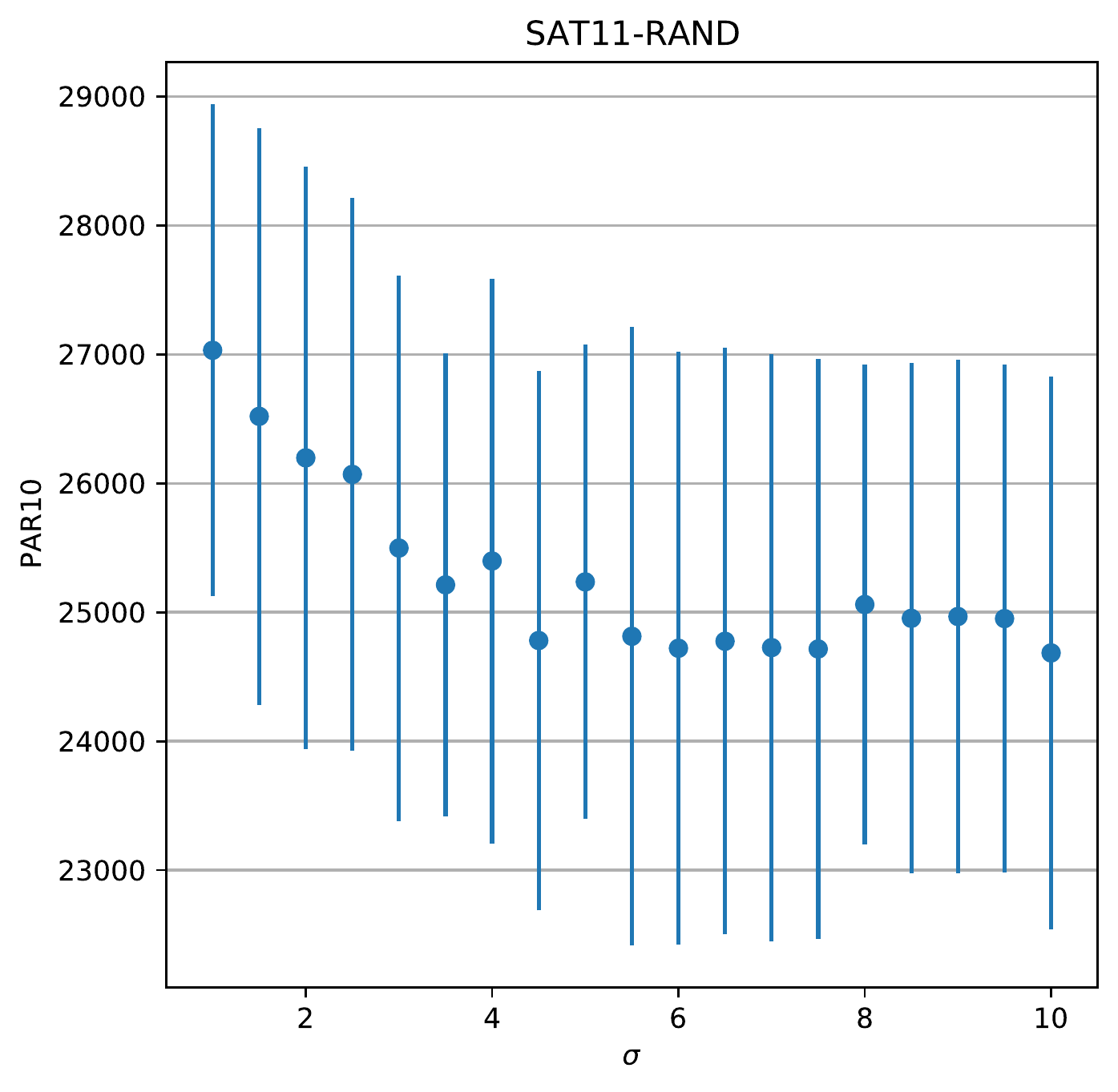}
	\label{fig:app_sensitivity_SAT11-RAND_sigma_rand_bclinucb_rev}
\end{subfigure}
\caption{Sensitivity analysis for parameter $\sigma$ of approach rand\_bclinucb\_rev.}
\label{fig:app_sensitivity_rand_bclinucb_rev_sigma}
\end{figure}

\begin{figure}[htb]
	\centering
\begin{subfigure}{0.25\textwidth}
	\includegraphics[width=\linewidth]{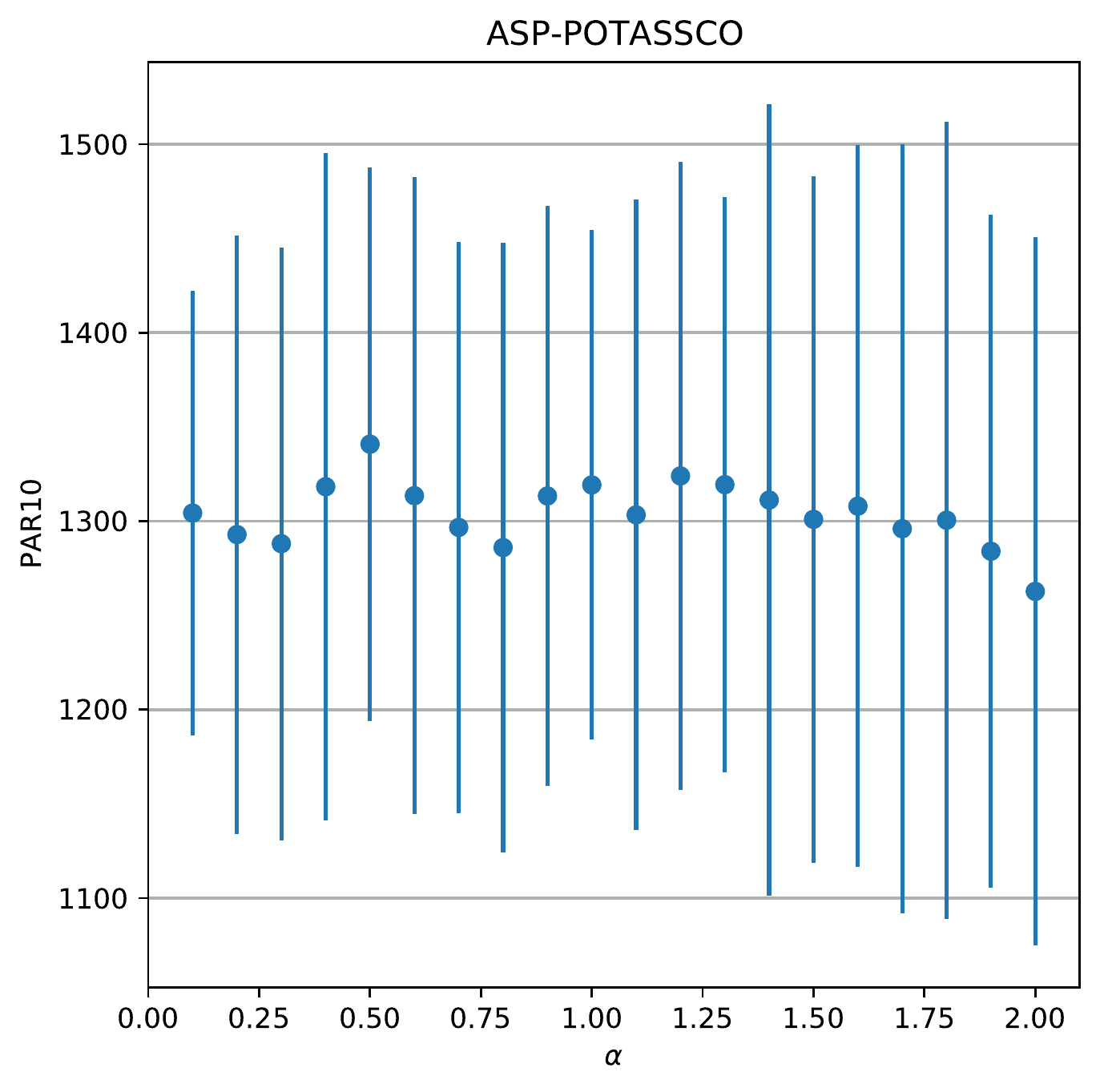}
	\label{fig:app_sensitivity_ASP-POTASSCO_alpha_rand_bclinucb_rev}
\end{subfigure}
\begin{subfigure}{0.25\textwidth}
	\includegraphics[width=\linewidth]{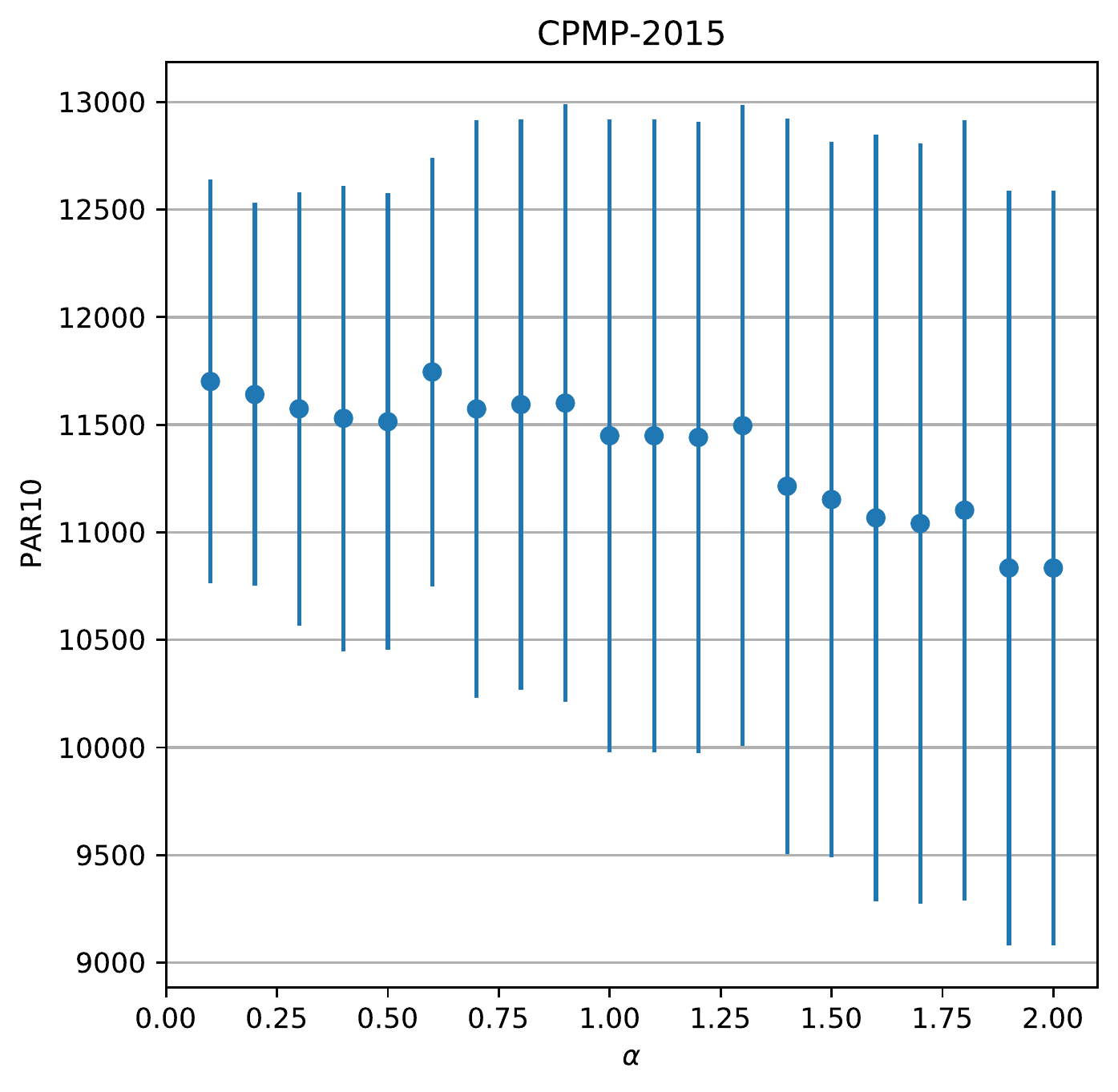}
	\label{fig:app_sensitivity_CPMP-2015_alpha_rand_bclinucb_rev}
\end{subfigure}
\begin{subfigure}{0.25\textwidth}
	\includegraphics[width=\linewidth]{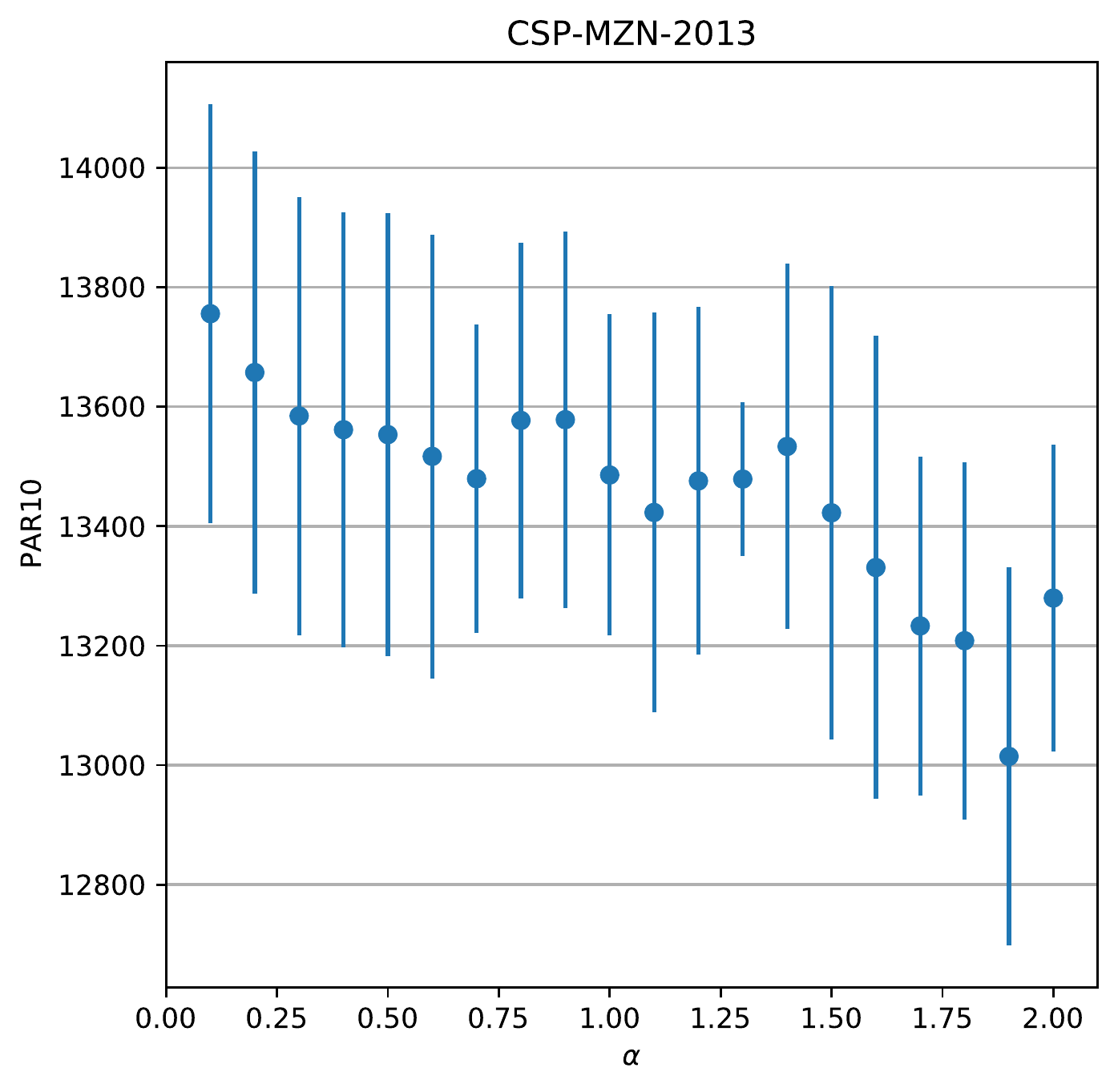}
	\label{fig:app_sensitivity_CSP-MZN-2013_alpha_rand_bclinucb_rev}
\end{subfigure}
\begin{subfigure}{0.25\textwidth}
	\includegraphics[width=\linewidth]{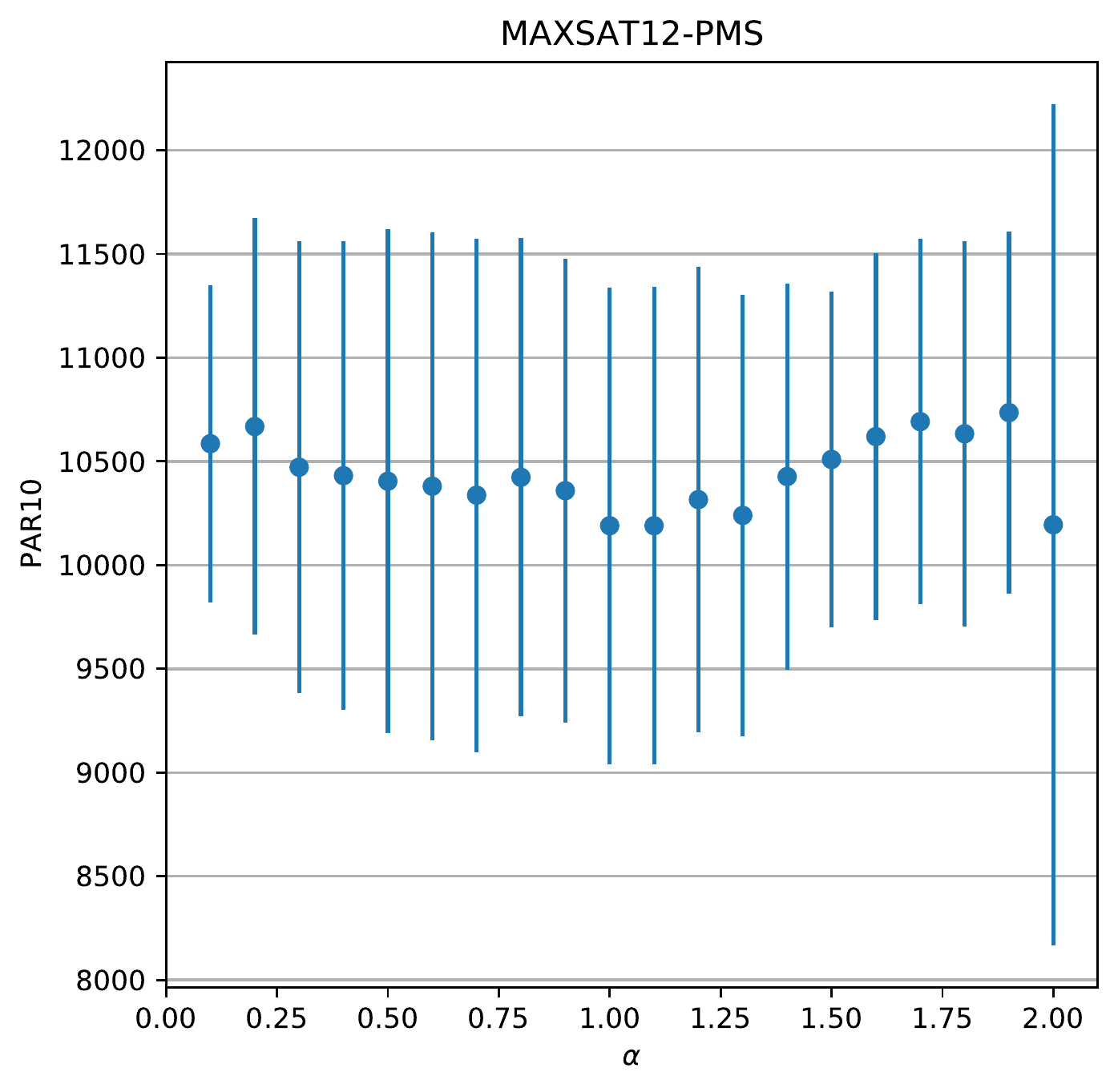}
	\label{fig:app_sensitivity_MAXSAT12-PMS_alpha_rand_bclinucb_rev}
\end{subfigure}
\begin{subfigure}{0.25\textwidth}
	\includegraphics[width=\linewidth]{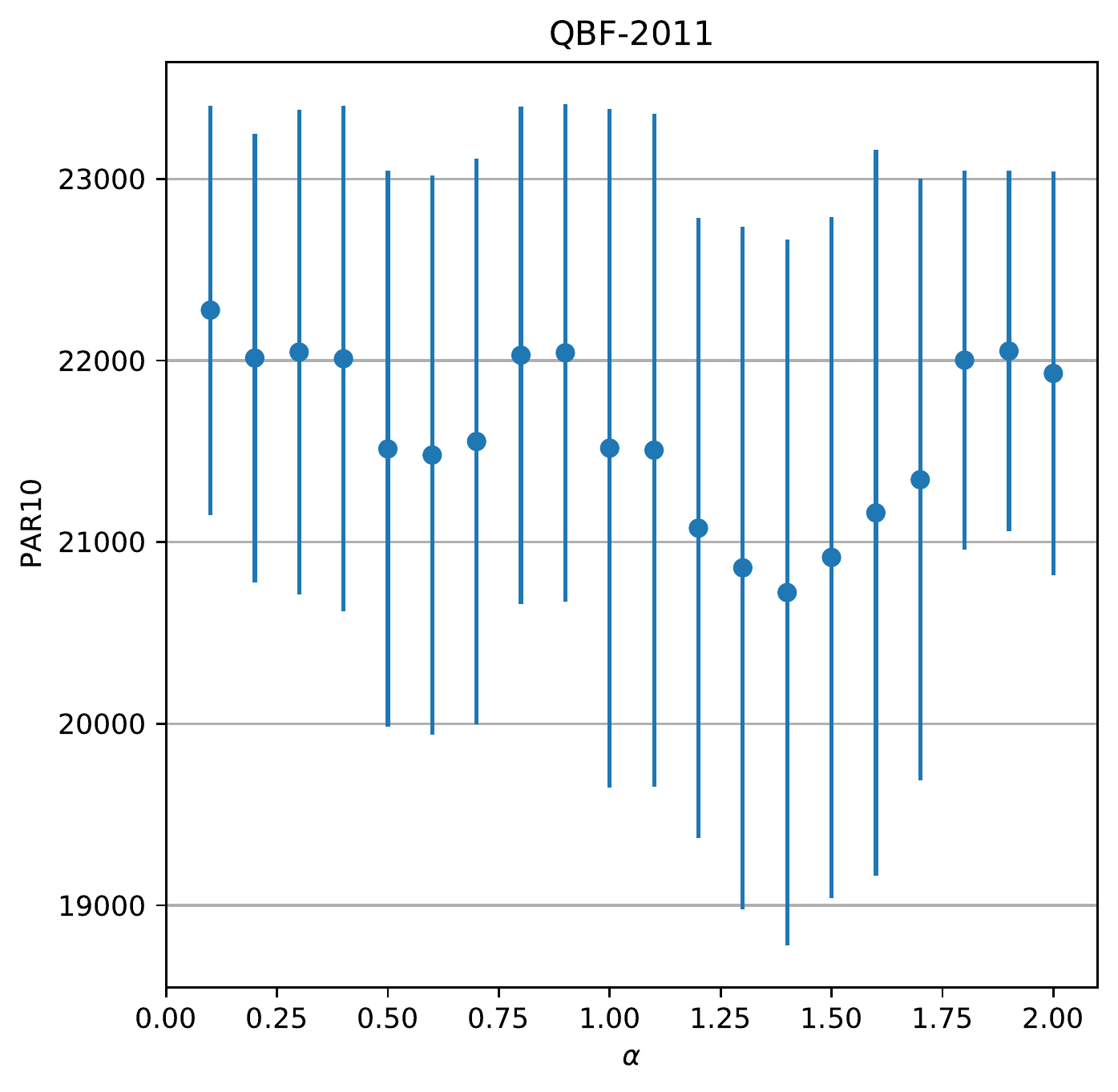}
	\label{fig:app_sensitivity_QBF-2011_alpha_rand_bclinucb_rev}
\end{subfigure}
\begin{subfigure}{0.25\textwidth}
	\includegraphics[width=\linewidth]{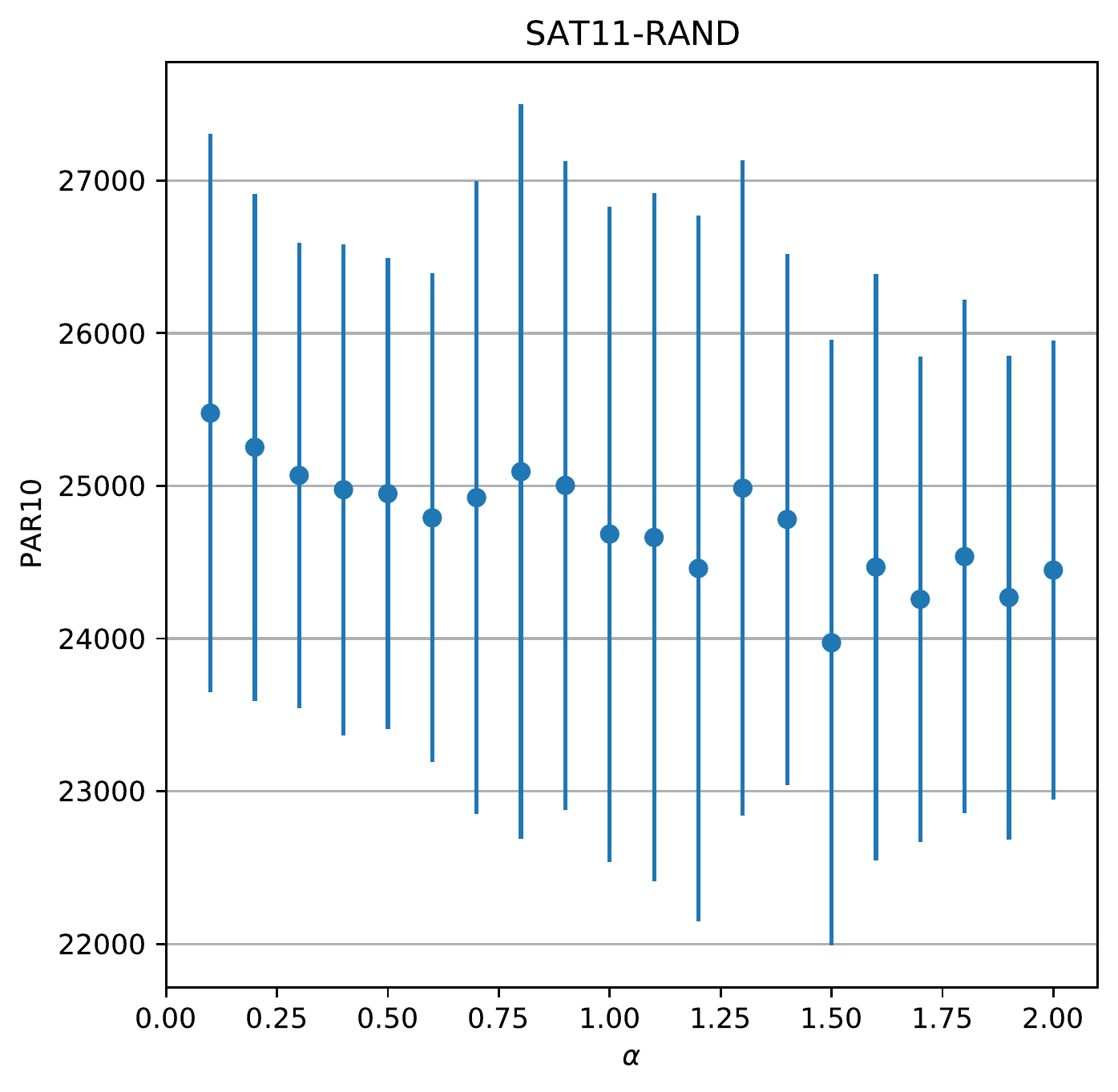}
	\label{fig:app_sensitivity_SAT11-RAND_alpha_rand_bclinucb_rev}
\end{subfigure}
\caption{Sensitivity analysis for parameter $\alpha$ of approach rand\_bclinucb\_rev.}
\label{fig:app_sensitivity_rand_bclinucb_rev_alpha}
\end{figure}

\begin{figure}[htb]
	\centering
\begin{subfigure}{0.25\textwidth}
	\includegraphics[width=\linewidth]{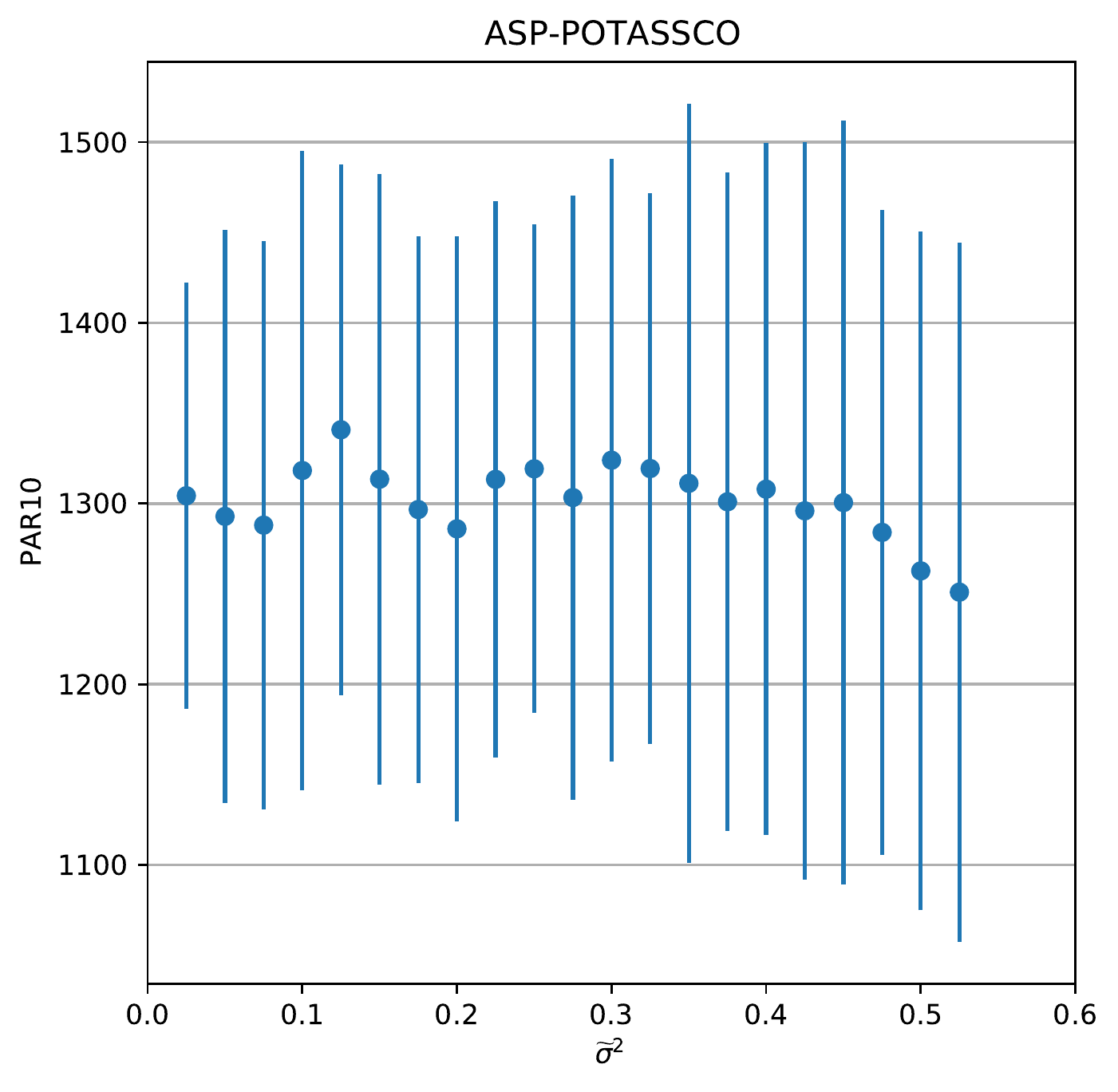}
	\label{fig:app_sensitivity_ASP-POTASSCO_randsigma_rand_bclinucb_rev}
\end{subfigure}
\begin{subfigure}{0.25\textwidth}
	\includegraphics[width=\linewidth]{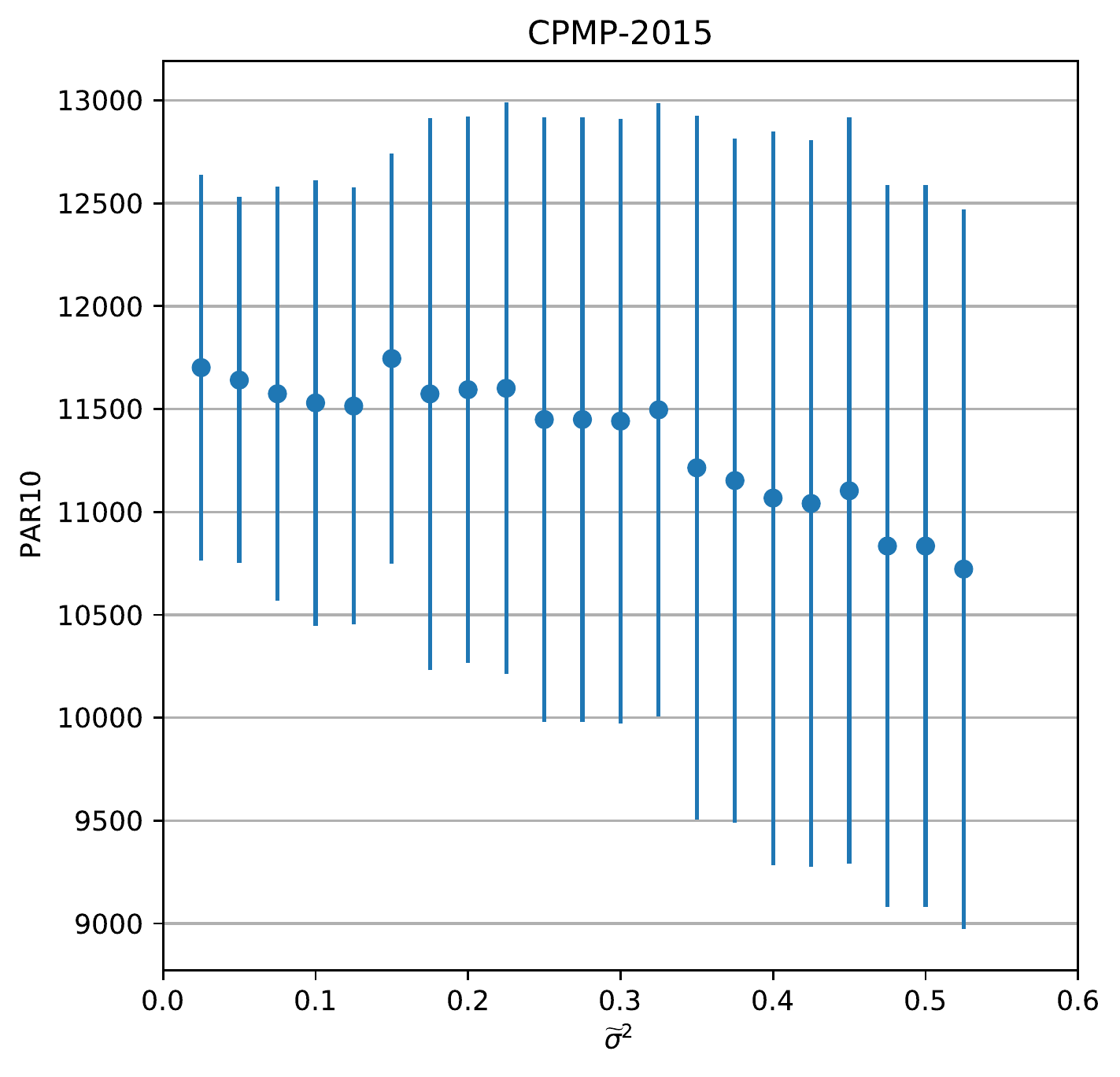}
	\label{fig:app_sensitivity_CPMP-2015_randsigma_rand_bclinucb_rev}
\end{subfigure}
\begin{subfigure}{0.25\textwidth}
	\includegraphics[width=\linewidth]{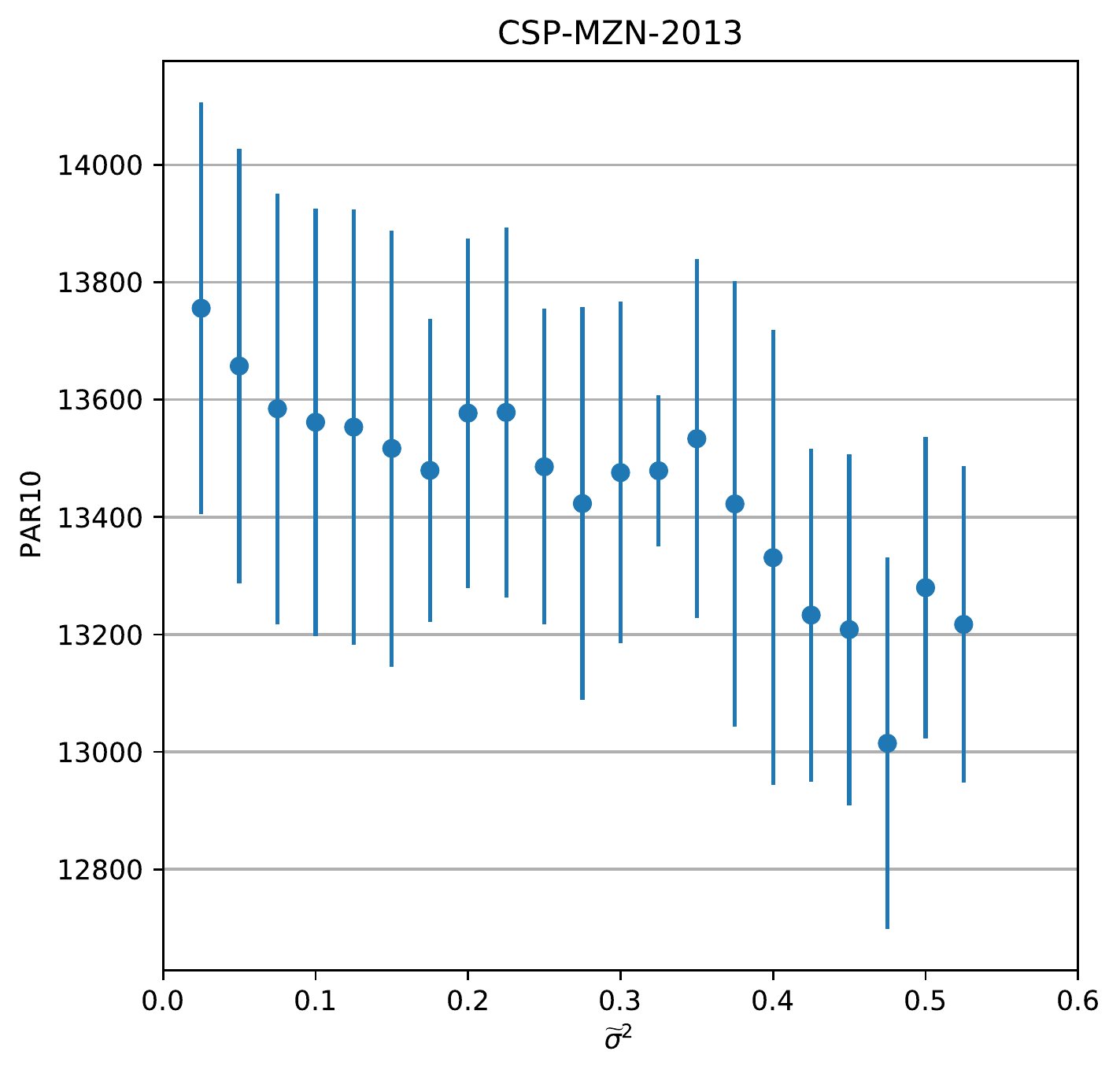}
	\label{fig:app_sensitivity_CSP-MZN-2013_randsigma_rand_bclinucb_rev}
\end{subfigure}
\begin{subfigure}{0.25\textwidth}
	\includegraphics[width=\linewidth]{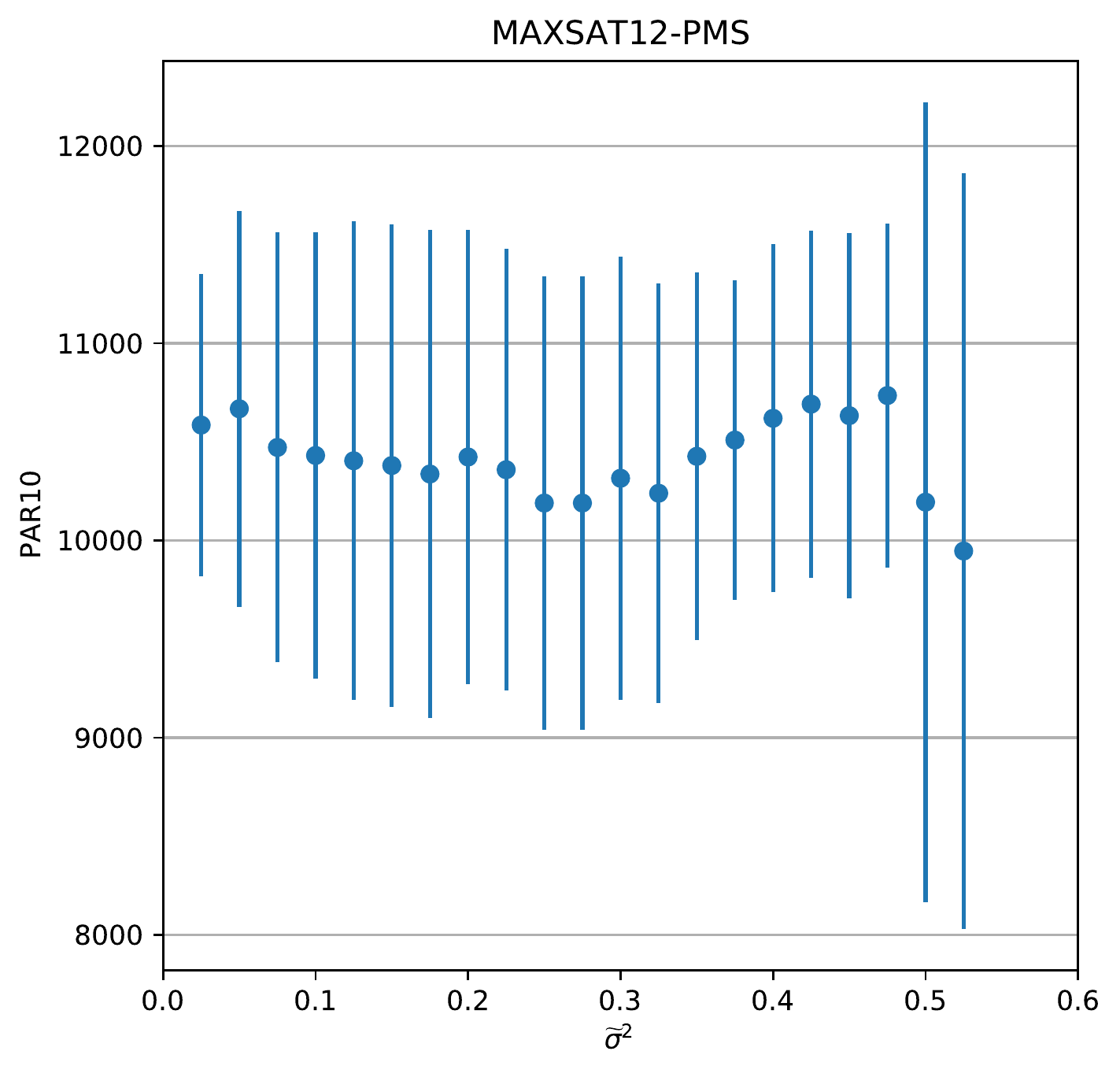}
	\label{fig:app_sensitivity_MAXSAT12-PMS_randsigma_rand_bclinucb_rev}
\end{subfigure}
\begin{subfigure}{0.25\textwidth}
	\includegraphics[width=\linewidth]{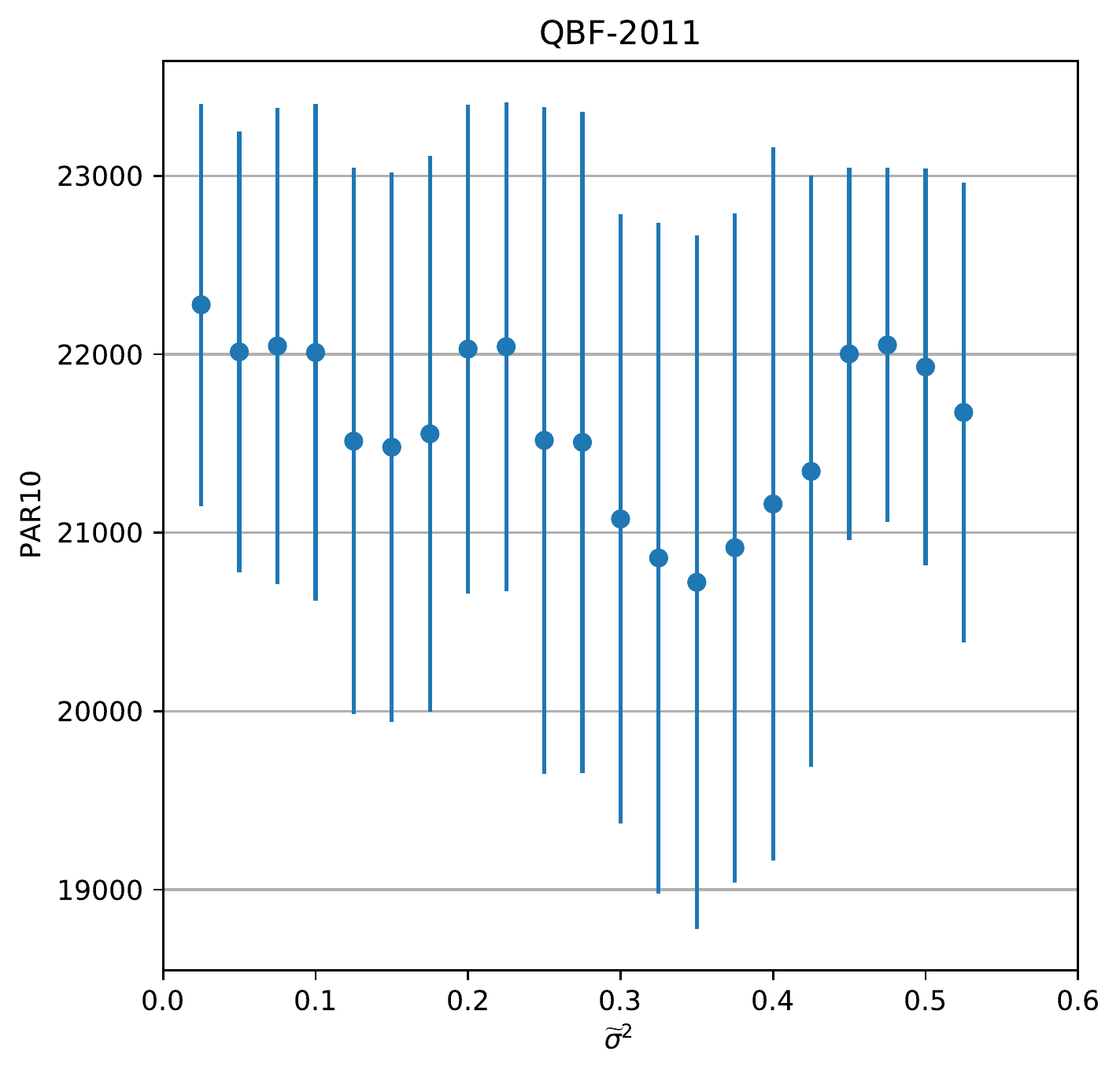}
	\label{fig:app_sensitivity_QBF-2011_randsigma_rand_bclinucb_rev}
\end{subfigure}
\begin{subfigure}{0.25\textwidth}
	\includegraphics[width=\linewidth]{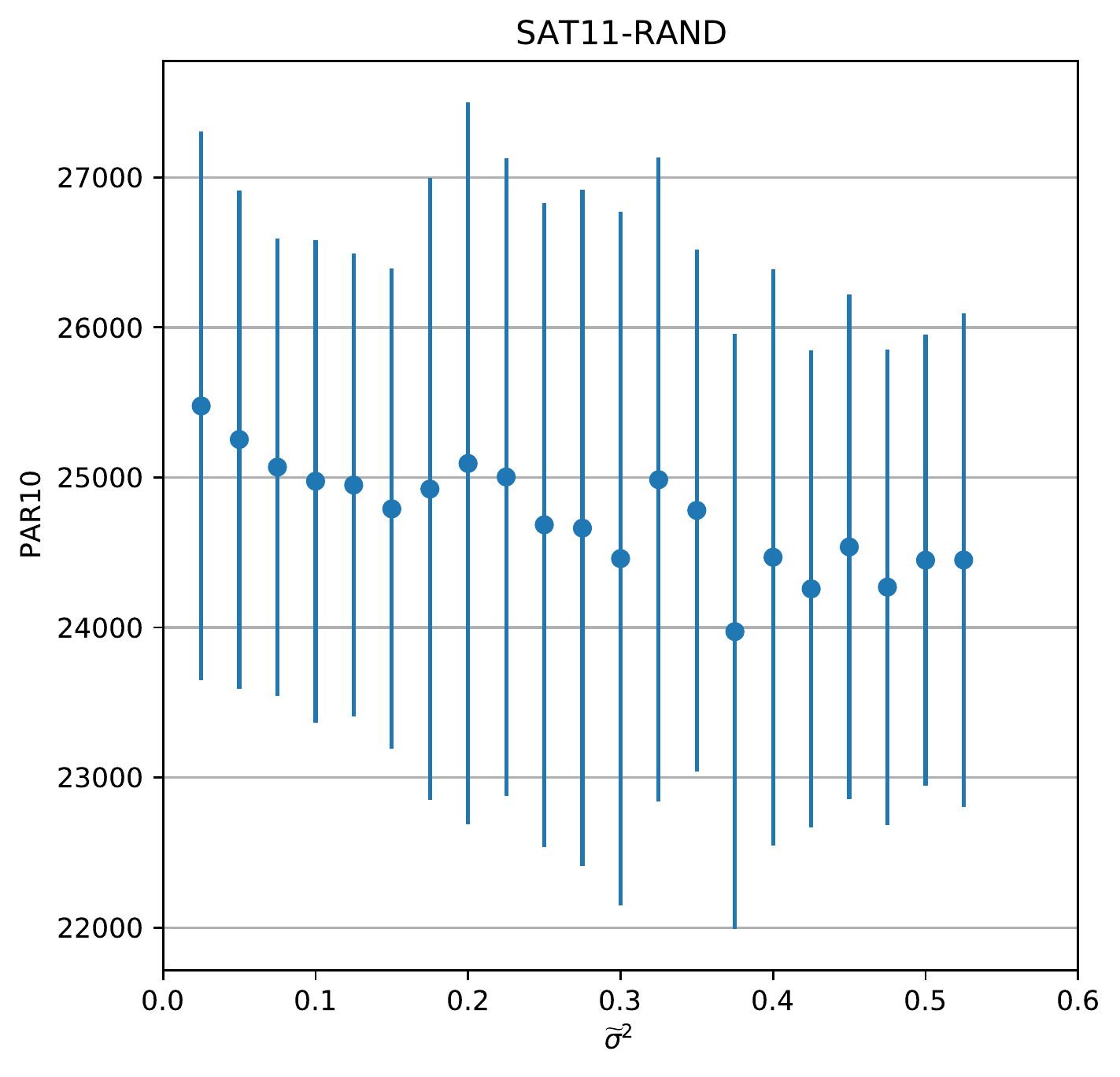}
	\label{fig:app_sensitivity_SAT11-RAND_randsigma_rand_bclinucb_rev}
\end{subfigure}
\caption{Sensitivity analysis for parameter $\widetilde{\sigma}^2$ of approach rand\_bclinucb\_rev.}
\label{fig:app_sensitivity_rand_bclinucb_rev_randsigma}
\end{figure}

\section{Detailed Performance Values}
Table \ref{tab:full_results} shows the average PAR10 scores (averaged over 10 seeds) and the corresponding standard deviation of all discussed approach variants and all Degroote variants for reference. Once again, the best value for each scenario is printed in bold, whereas the second best is underlined. As elaborated on in the main paper, the Thompson variants achieve the best performance.

\begin{table}[ht]
    \centering
    \rotatebox{270}{
    \resizebox*{0.9\textheight}{!}{
        \begin{tabular}{l|lllllllllllll}
\toprule
\multicolumn{1}{c|}{\rotatebox{90}{Approach}} & \multicolumn{1}{c}{\rotatebox{90}{bj\_thompson\_rev}} & \multicolumn{1}{c}{\rotatebox{90}{bj\_thompson}} & \multicolumn{1}{c}{\rotatebox{90}{thompson\_rev}} & \multicolumn{1}{c}{\rotatebox{90}{thompson}} & \multicolumn{1}{c}{\rotatebox{90}{bclinucb\_rev}} & \multicolumn{1}{c}{\rotatebox{90}{bclinucb}} & \multicolumn{1}{c}{\rotatebox{90}{blinducb\_rev}} & \multicolumn{1}{c}{\rotatebox{90}{blinducb}} & \multicolumn{1}{c}{\rotatebox{90}{rand\_bclinucb\_rev}} & \multicolumn{1}{c}{\rotatebox{90}{rand\_bclinucb}} & \multicolumn{1}{c}{\rotatebox{90}{rand\_blinducb\_rev}} & \multicolumn{1}{c}{\rotatebox{90}{rand\_blinducb}} & \multicolumn{1}{c}{\rotatebox{90}{degroote\_EpsilonGreedy\_LinearRegression}} \\
Scenario          &                                                       &                                                  &                                                   &                                              &                                                   &                                              &                                                   &                                              &                                                         &                                                    &                                                         &                                                    &                                                                               \\
\midrule
ASP-POTASSCO           &                                 929.45 $\pm$ 86.72 &                               949.38 $\pm$ 62.38 &                       \textbf{902.64} $\pm$ 78.43 &               \underline{916.28} $\pm$ 72.73 &                              1264.98 $\pm$ 303.58 &                         1361.64 $\pm$ 239.21 &                               1471.79 $\pm$ 43.06 &                         1204.20 $\pm$ 161.11 &                               1319.22 $\pm$ 135.21 &                                1232.59 $\pm$ 64.63 &                                1280.76 $\pm$ 24.99 &                               1198.79 $\pm$ 135.63 &                                1047.13 $\pm$ 46.50 \\
BNSL-2016              &                               9656.86 $\pm$ 318.43 &                             9638.04 $\pm$ 378.05 &                  \underline{9467.01} $\pm$ 252.52 &                \textbf{9361.41} $\pm$ 262.75 &                           31224.30 $\pm$ 12227.55 &                       35665.49 $\pm$ 8813.31 &                             30500.86 $\pm$ 852.64 &                       17369.40 $\pm$ 1553.72 &                             32178.72 $\pm$ 6600.21 &                             23650.56 $\pm$ 2177.42 &                              23563.48 $\pm$ 907.10 &                             18531.72 $\pm$ 2259.16 &                             12510.26 $\pm$ 1291.03 \\
CPMP-2015              &                  \underline{7818.47} $\pm$ 1187.55 &                            8241.01 $\pm$ 1164.85 &                             8158.72 $\pm$ 1268.83 &                        8499.38 $\pm$ 2059.18 &                            10441.68 $\pm$ 2449.07 &                       11437.13 $\pm$ 2265.92 &                             11667.65 $\pm$ 369.54 &                       11518.55 $\pm$ 1316.62 &                             11449.01 $\pm$ 1469.61 &                             10664.02 $\pm$ 1014.65 &                              10439.34 $\pm$ 308.03 &                             11350.81 $\pm$ 1338.76 &                      \textbf{6991.97} $\pm$ 501.36 \\
CSP-2010               &                               8138.63 $\pm$ 820.90 &                             8295.76 $\pm$ 699.43 &                  \underline{7892.67} $\pm$ 692.83 &                         8103.96 $\pm$ 887.50 &                             8296.32 $\pm$ 2170.35 &                        8837.75 $\pm$ 2487.60 &                             10348.84 $\pm$ 197.67 &                        9563.85 $\pm$ 1821.74 &                              9533.88 $\pm$ 1360.51 &                               9378.99 $\pm$ 270.41 &                               9847.48 $\pm$ 316.70 &                              8796.21 $\pm$ 1534.82 &                      \textbf{7593.13} $\pm$ 208.94 \\
CSP-MZN-2013           &                               8291.83 $\pm$ 589.63 &                             8207.06 $\pm$ 532.70 &                  \underline{8171.21} $\pm$ 594.49 &                         8472.17 $\pm$ 760.25 &                             12665.29 $\pm$ 647.82 &                       15060.69 $\pm$ 1033.48 &                              14646.81 $\pm$ 64.27 &                        12113.14 $\pm$ 593.67 &                              13485.81 $\pm$ 268.79 &                               12516.57 $\pm$ 90.27 &                               13551.04 $\pm$ 95.32 &                              11595.72 $\pm$ 914.74 &                      \textbf{8034.62} $\pm$ 113.78 \\
CSP-Minizinc-Time-2016 &                      \textbf{4741.99} $\pm$ 505.13 &                             4811.54 $\pm$ 409.79 &                  \underline{4759.50} $\pm$ 306.03 &                         4942.91 $\pm$ 326.89 &                             6284.19 $\pm$ 1377.37 &                         7515.73 $\pm$ 831.41 &                              7138.65 $\pm$ 283.28 &                         5944.38 $\pm$ 481.06 &                               7043.06 $\pm$ 440.10 &                               6104.17 $\pm$ 404.23 &                               6291.28 $\pm$ 431.44 &                               5353.62 $\pm$ 364.20 &                               5258.70 $\pm$ 406.91 \\
GRAPHS-2015            &                            4.11e+07 $\pm$ 3.71e+06 &                          4.13e+07 $\pm$ 4.39e+06 &                           4.21e+07 $\pm$ 3.37e+06 &          \underline{3.97e+07} $\pm$ 4.76e+06 &                           1.53e+08 $\pm$ 1.15e+08 &                      1.53e+08 $\pm$ 1.15e+08 &                           1.11e+08 $\pm$ 1.64e+07 &                      6.91e+07 $\pm$ 4.15e+07 &                            1.74e+08 $\pm$ 8.54e+07 &                            1.23e+08 $\pm$ 4.48e+06 &                            7.85e+07 $\pm$ 1.71e+07 &                            6.72e+07 $\pm$ 4.00e+07 &                   \textbf{3.45e+07} $\pm$ 1.42e+06 \\
MAXSAT-PMS-2016        &                   \underline{2774.15} $\pm$ 218.67 &                             2853.44 $\pm$ 210.21 &                              2808.51 $\pm$ 218.55 &                \textbf{2763.58} $\pm$ 134.14 &                            12391.75 $\pm$ 6441.49 &                        16674.43 $\pm$ 298.37 &                             15087.51 $\pm$ 399.80 &                        5785.25 $\pm$ 1369.27 &                              15107.67 $\pm$ 238.80 &                               7320.34 $\pm$ 877.35 &                              11887.90 $\pm$ 403.61 &                               5405.64 $\pm$ 495.98 &                               3279.54 $\pm$ 133.00 \\
MAXSAT-WPMS-2016       &                               6548.69 $\pm$ 183.77 &                 \underline{6304.15} $\pm$ 166.98 &                              6592.87 $\pm$ 210.25 &                         6527.34 $\pm$ 213.44 &                             7036.99 $\pm$ 3325.59 &                        16479.67 $\pm$ 482.00 &                             15948.34 $\pm$ 203.76 &                        12940.70 $\pm$ 727.41 &                              15326.14 $\pm$ 285.53 &                              11268.89 $\pm$ 443.14 &                              13492.59 $\pm$ 174.10 &                              12518.20 $\pm$ 585.24 &                      \textbf{6287.21} $\pm$ 541.69 \\
MAXSAT12-PMS           &                               5373.99 $\pm$ 348.92 &                             5347.39 $\pm$ 291.87 &                              5408.40 $\pm$ 482.42 &             \underline{5324.88} $\pm$ 208.49 &                             8927.94 $\pm$ 3294.10 &                        11787.82 $\pm$ 385.81 &                             11918.42 $\pm$ 279.90 &                        9829.22 $\pm$ 2477.59 &                             10189.82 $\pm$ 1148.46 &                               7761.18 $\pm$ 632.24 &                               9661.09 $\pm$ 254.48 &                              9580.82 $\pm$ 2684.49 &                      \textbf{5308.11} $\pm$ 129.30 \\
MAXSAT15-PMS-INDU      &                   \underline{3040.87} $\pm$ 196.70 &                             3046.05 $\pm$ 128.34 &                      \textbf{3032.08} $\pm$ 90.71 &                         3080.24 $\pm$ 130.48 &                             14579.90 $\pm$ 814.65 &                        16529.05 $\pm$ 274.72 &                             15482.08 $\pm$ 299.25 &                       12020.48 $\pm$ 1801.39 &                              15342.87 $\pm$ 313.35 &                               9842.87 $\pm$ 419.03 &                              12936.66 $\pm$ 495.89 &                              10416.87 $\pm$ 930.70 &                               3867.70 $\pm$ 255.98 \\
MIP-2016               &                      \textbf{7961.45} $\pm$ 765.53 &                 \underline{8081.57} $\pm$ 845.74 &                             8746.73 $\pm$ 1159.36 &                         8776.59 $\pm$ 823.11 &                           19319.18 $\pm$ 12407.96 &                      23090.18 $\pm$ 10445.65 &                            20305.81 $\pm$ 1787.80 &                        10280.18 $\pm$ 646.54 &                            21472.48 $\pm$ 10159.74 &                             21224.37 $\pm$ 3076.88 &                             16238.36 $\pm$ 1625.48 &                              10102.14 $\pm$ 783.52 &                             10644.68 $\pm$ 3405.18 \\
PROTEUS-2014           &                              14223.14 $\pm$ 766.02 &                   \textbf{13484.34} $\pm$ 541.83 &                             14115.69 $\pm$ 768.16 &            \underline{13550.56} $\pm$ 426.67 &                            20267.46 $\pm$ 3744.35 &                        23480.31 $\pm$ 416.14 &                             23330.42 $\pm$ 206.84 &                        23496.49 $\pm$ 957.96 &                              22713.85 $\pm$ 325.06 &                              21945.81 $\pm$ 518.38 &                              22380.67 $\pm$ 173.49 &                             23508.44 $\pm$ 1081.76 &                              15622.29 $\pm$ 784.60 \\
QBF-2011               &                              15253.81 $\pm$ 839.93 &                            15708.25 $\pm$ 784.81 &                             15178.86 $\pm$ 904.72 &            \underline{14902.82} $\pm$ 834.91 &                            20051.24 $\pm$ 3346.02 &                       23794.53 $\pm$ 2734.83 &                             24467.88 $\pm$ 266.89 &                       24105.19 $\pm$ 2408.37 &                             21517.52 $\pm$ 1867.77 &                              19664.68 $\pm$ 805.56 &                              21819.49 $\pm$ 613.82 &                             24190.56 $\pm$ 1452.98 &                     \textbf{13912.24} $\pm$ 356.69 \\
QBF-2014               &                      \textbf{3552.30} $\pm$ 192.75 &                             3629.40 $\pm$ 220.68 &                              3679.96 $\pm$ 256.03 &             \underline{3599.79} $\pm$ 193.80 &                              4305.61 $\pm$ 623.06 &                         5145.37 $\pm$ 760.32 &                               5669.88 $\pm$ 58.89 &                         4951.40 $\pm$ 630.45 &                               5414.49 $\pm$ 240.11 &                               4976.71 $\pm$ 128.30 &                                5213.18 $\pm$ 81.48 &                               4836.72 $\pm$ 617.95 &                               4116.15 $\pm$ 116.27 \\
QBF-2016               &                      \textbf{4770.36} $\pm$ 595.50 &                             5082.59 $\pm$ 718.71 &                              5045.16 $\pm$ 848.59 &             \underline{4937.47} $\pm$ 710.69 &                             5765.73 $\pm$ 1157.23 &                        7969.34 $\pm$ 1830.92 &                              7631.75 $\pm$ 323.55 &                         6080.61 $\pm$ 828.17 &                               8530.38 $\pm$ 426.05 &                               6019.49 $\pm$ 291.23 &                               6811.82 $\pm$ 209.17 &                               6234.42 $\pm$ 943.66 &                               5346.29 $\pm$ 210.05 \\
SAT03-16\_INDU          &                  \underline{12128.48} $\pm$ 477.79 &                   \textbf{11980.15} $\pm$ 193.67 &                             12154.46 $\pm$ 221.01 &                        12225.57 $\pm$ 501.47 &                            14973.15 $\pm$ 1628.24 &                       14800.11 $\pm$ 2043.18 &                             13682.55 $\pm$ 263.59 &                        12836.59 $\pm$ 385.99 &                             14204.97 $\pm$ 1429.35 &                              13505.17 $\pm$ 405.65 &                              13421.35 $\pm$ 251.79 &                              12671.99 $\pm$ 458.24 &                              12754.50 $\pm$ 200.55 \\
SAT11-HAND             &                             30436.08 $\pm$ 1196.65 &                           30484.08 $\pm$ 1379.35 &                             30085.51 $\pm$ 764.32 &            \underline{29547.31} $\pm$ 871.64 &                             32028.44 $\pm$ 132.51 &                       32086.20 $\pm$ 1315.24 &                             30670.46 $\pm$ 525.38 &                        33755.48 $\pm$ 470.48 &                              30710.59 $\pm$ 932.73 &                              31157.03 $\pm$ 843.87 &                              30519.78 $\pm$ 633.95 &                              33151.03 $\pm$ 689.91 &                     \textbf{29544.70} $\pm$ 952.78 \\
SAT11-INDU             &                              17083.58 $\pm$ 490.45 &                            17540.58 $\pm$ 530.82 &                 \underline{17028.84} $\pm$ 479.15 &                       17360.28 $\pm$ 1064.28 &                            18029.73 $\pm$ 1914.17 &                       17944.22 $\pm$ 1514.41 &                             17488.94 $\pm$ 420.92 &                        17407.75 $\pm$ 887.53 &                              17785.96 $\pm$ 942.15 &                              17900.53 $\pm$ 797.85 &                              17404.10 $\pm$ 671.54 &                              17792.97 $\pm$ 645.24 &                     \textbf{17018.24} $\pm$ 647.90 \\
SAT11-RAND             &                             19656.39 $\pm$ 3747.94 &               \underline{18061.78} $\pm$ 2770.70 &                            19061.88 $\pm$ 2522.11 &              \textbf{16535.77} $\pm$ 2649.93 &                            22349.94 $\pm$ 2672.21 &                       23399.06 $\pm$ 2254.92 &                             32256.56 $\pm$ 333.93 &                        32812.38 $\pm$ 484.62 &                             24684.78 $\pm$ 2146.11 &                              22754.43 $\pm$ 401.94 &                              28780.03 $\pm$ 577.18 &                              32382.10 $\pm$ 453.97 &                              21008.77 $\pm$ 530.22 \\
SAT12-ALL              &                               5110.38 $\pm$ 221.26 &                    \textbf{4720.22} $\pm$ 432.14 &                              5132.48 $\pm$ 395.74 &             \underline{4812.94} $\pm$ 387.46 &                              6433.41 $\pm$ 991.59 &                         7657.90 $\pm$ 673.88 &                              8548.46 $\pm$ 145.06 &                         8235.03 $\pm$ 393.14 &                               7388.34 $\pm$ 313.26 &                               7096.01 $\pm$ 138.67 &                                7417.48 $\pm$ 76.56 &                               8231.51 $\pm$ 319.87 &                               5650.32 $\pm$ 214.36 \\
SAT12-HAND             &                               7707.60 $\pm$ 219.43 &                 \underline{7443.01} $\pm$ 180.51 &                              7509.02 $\pm$ 199.39 &                \textbf{7309.74} $\pm$ 138.02 &                             8086.51 $\pm$ 1043.74 &                         9345.40 $\pm$ 669.01 &                              8395.15 $\pm$ 289.83 &                         9944.53 $\pm$ 702.20 &                               8626.88 $\pm$ 250.21 &                               8214.43 $\pm$ 185.22 &                               8164.12 $\pm$ 126.91 &                               9560.68 $\pm$ 629.12 &                               7634.24 $\pm$ 267.89 \\
SAT12-INDU             &                               6227.99 $\pm$ 714.29 &                  \underline{4511.68} $\pm$ 76.33 &                              4945.79 $\pm$ 228.37 &                \textbf{4428.49} $\pm$ 142.91 &                             9741.33 $\pm$ 3131.30 &                        11083.35 $\pm$ 274.76 &                             10586.81 $\pm$ 417.60 &                         11709.07 $\pm$ 47.37 &                              10733.23 $\pm$ 400.92 &                               8703.58 $\pm$ 987.22 &                               9159.07 $\pm$ 827.64 &                               11568.96 $\pm$ 45.53 &                               4755.52 $\pm$ 206.95 \\
SAT12-RAND             &                               5489.81 $\pm$ 649.90 &                    \textbf{4008.79} $\pm$ 206.59 &                              4523.33 $\pm$ 170.56 &             \underline{4157.19} $\pm$ 277.10 &                             4541.08 $\pm$ 1857.92 &                        11024.73 $\pm$ 212.35 &                              11118.21 $\pm$ 91.24 &                         8270.41 $\pm$ 249.25 &                              10699.76 $\pm$ 153.19 &                               9015.20 $\pm$ 189.65 &                               10432.59 $\pm$ 97.18 &                               8109.35 $\pm$ 140.40 &                               5023.73 $\pm$ 174.68 \\
SAT15-INDU             &                               7885.50 $\pm$ 583.23 &                    \textbf{7700.27} $\pm$ 310.65 &                  \underline{7856.08} $\pm$ 522.84 &                         7879.61 $\pm$ 546.81 &                             9501.14 $\pm$ 3536.88 &                       23644.43 $\pm$ 2846.82 &                              9691.58 $\pm$ 716.04 &                       11203.29 $\pm$ 3070.49 &                             12188.32 $\pm$ 3435.99 &                              10349.12 $\pm$ 897.36 &                               9599.63 $\pm$ 676.35 &                             10656.72 $\pm$ 3172.70 &                               8220.22 $\pm$ 525.13 \\
SAT18-EXP              &                  \underline{24942.85} $\pm$ 847.88 &                            25201.41 $\pm$ 681.42 &                    \textbf{24906.56} $\pm$ 540.36 &                       25015.28 $\pm$ 1031.38 &                            25310.89 $\pm$ 1669.61 &                       33234.71 $\pm$ 5109.02 &                             26762.26 $\pm$ 760.54 &                       29573.37 $\pm$ 2210.54 &                              27153.72 $\pm$ 924.36 &                              26709.51 $\pm$ 679.87 &                             26361.79 $\pm$ 1018.39 &                             30237.60 $\pm$ 1497.64 &                              25272.35 $\pm$ 881.19 \\
TSP-LION2015           &                      \textbf{1160.50} $\pm$ 396.73 &                             1226.11 $\pm$ 309.42 &                              1411.06 $\pm$ 329.16 &             \underline{1213.43} $\pm$ 370.00 &                             8324.41 $\pm$ 4816.93 &                        8436.53 $\pm$ 4575.03 &                              4810.21 $\pm$ 212.49 &                         9991.79 $\pm$ 294.01 &                              9168.26 $\pm$ 2612.01 &                               6255.37 $\pm$ 194.34 &                               4588.53 $\pm$ 191.98 &                               9422.52 $\pm$ 655.46 &                               1634.79 $\pm$ 112.29 \\ \midrule
avgrank                &                                           3.074074 &                                         3.074074 &                                          \underline{2.925926} &                                     \textbf{2.555556} &                                               8.0 &                                    11.518519 &                                         10.740741 &                                     9.481481 &                                          10.592593 &                                           8.074074 &                                           8.518519 &                                           8.777778 &                                           3.666667 \\
\bottomrule
\end{tabular}
    }
    }
    \caption{Average PAR10 scores (averaged over 10 seeds) and the corresponding standard deviation of all discussed approach variants and the Degroote approach.}
    \label{tab:full_results}
\end{table}

In order to represent the performance of our approaches in a more detailed way than it is the case in Table \ref{tab:full_results}, we have plotted in Figure \ref{fig:app_cumulative_regret_1} the averaged cumulative PAR10 regret curves (regret wrt. the oracle) of the best Thompson, the best LinUCB and the Degroote approach along with their standard deviation.
Here, the cumulative regret up to time $T$ of an approach $s$ is defined as $\sum_{t=1}^T l(i_t,s(h_t, i_t)) - \sum_{t=1}^T l(i_t,s^*(h_t, i_t)),$ where $s^*$ is the oracle and $l$ as in \eqref{eq:loss_at_t} with $\mathcal{P}(C)=10C$.
It is not difficult to see that LinUCB cannot compete with the other approaches in many cases and also features a much larger standard deviation than the others. However, in several cases such as Figures \ref{fig:app_cumulative_regret_qbf-2014}, \ref{fig:app_cumulative_regret_sat11-hand}, or  \ref{fig:app_cumulative_regret_sat18-exp}, the differences become much more subtle.
Comparing the Thompson variant with the Degroote approach, we see that the former is at least competitive with the latter on almost all scenarios, while being even better on some scenarios (e.g. Fig. \ref{fig:app_cumulative_regret_bnsl-2016}, \ref{fig:app_cumulative_regret_maxsat15-pms-indu}, \ref{fig:app_cumulative_regret_proteus-2014}, \ref{fig:app_cumulative_regret_qbf-2014}).
Of course, there are also a few scenarios where the Degroote approach performs better  (e.g. Fig. \ref{fig:app_cumulative_regret_qbf-2011}).
%
%
\begin{figure}[htb]
	\centering
\begin{subfigure}{0.3\textwidth}
	 \includegraphics[width=\linewidth]{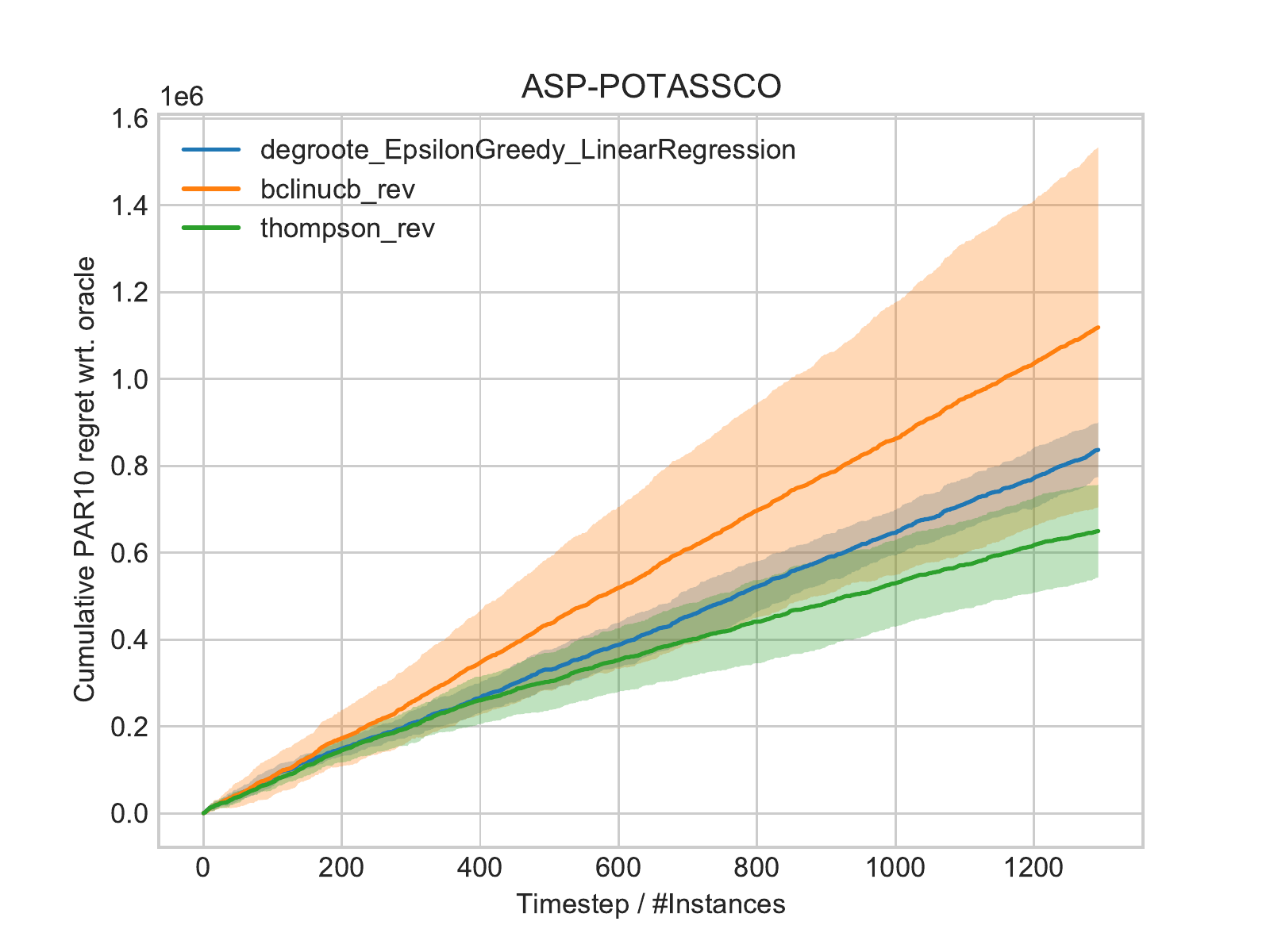}
	 \caption{}
	 \label{fig:app_cumulative_regret_asp-potassco}
\end{subfigure}
\begin{subfigure}{0.3\textwidth}
	 \includegraphics[width=\linewidth]{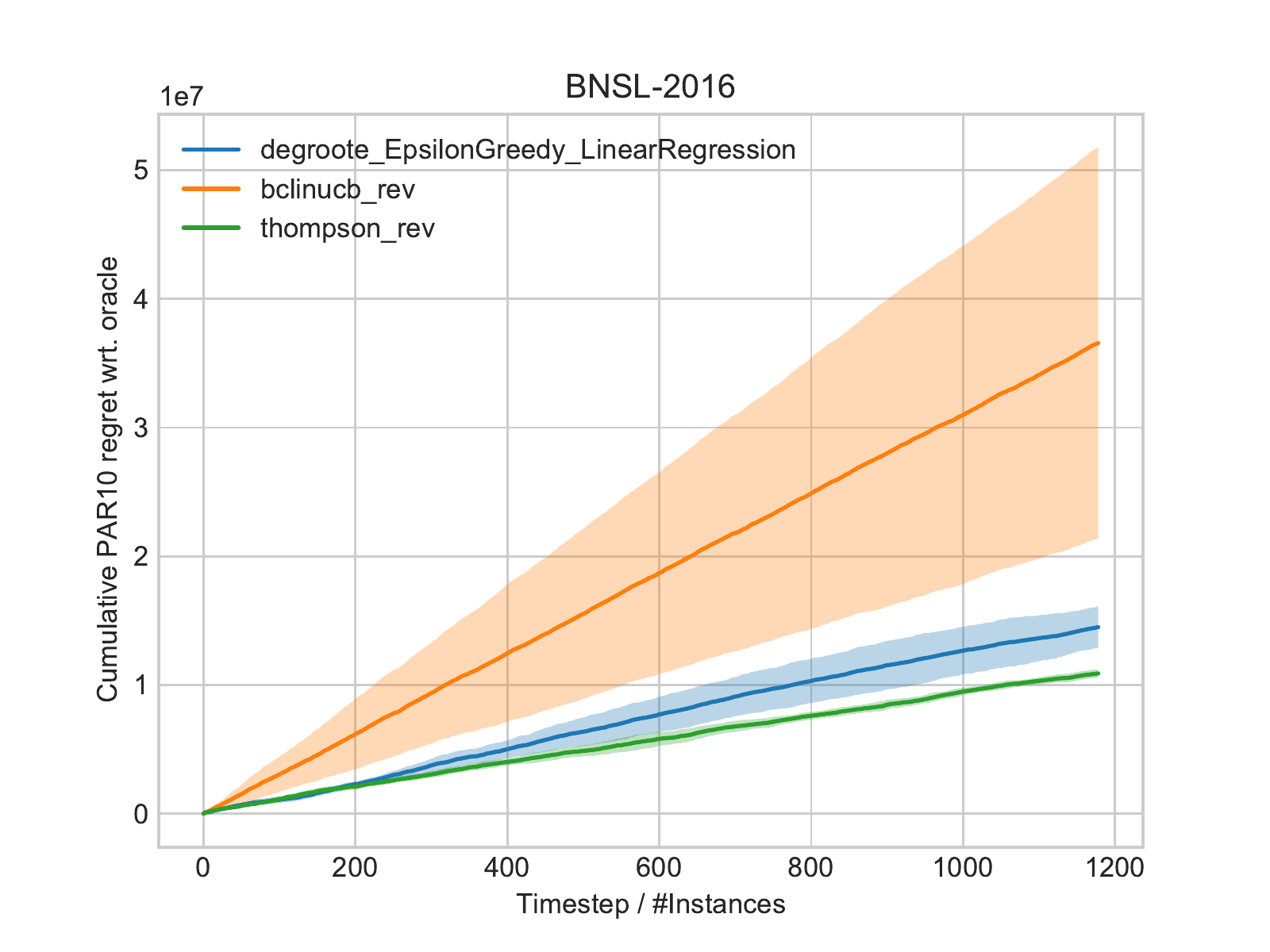}
	 \caption{}
	 \label{fig:app_cumulative_regret_bnsl-2016}
\end{subfigure}
\begin{subfigure}{0.3\textwidth}
	 \includegraphics[width=\linewidth]{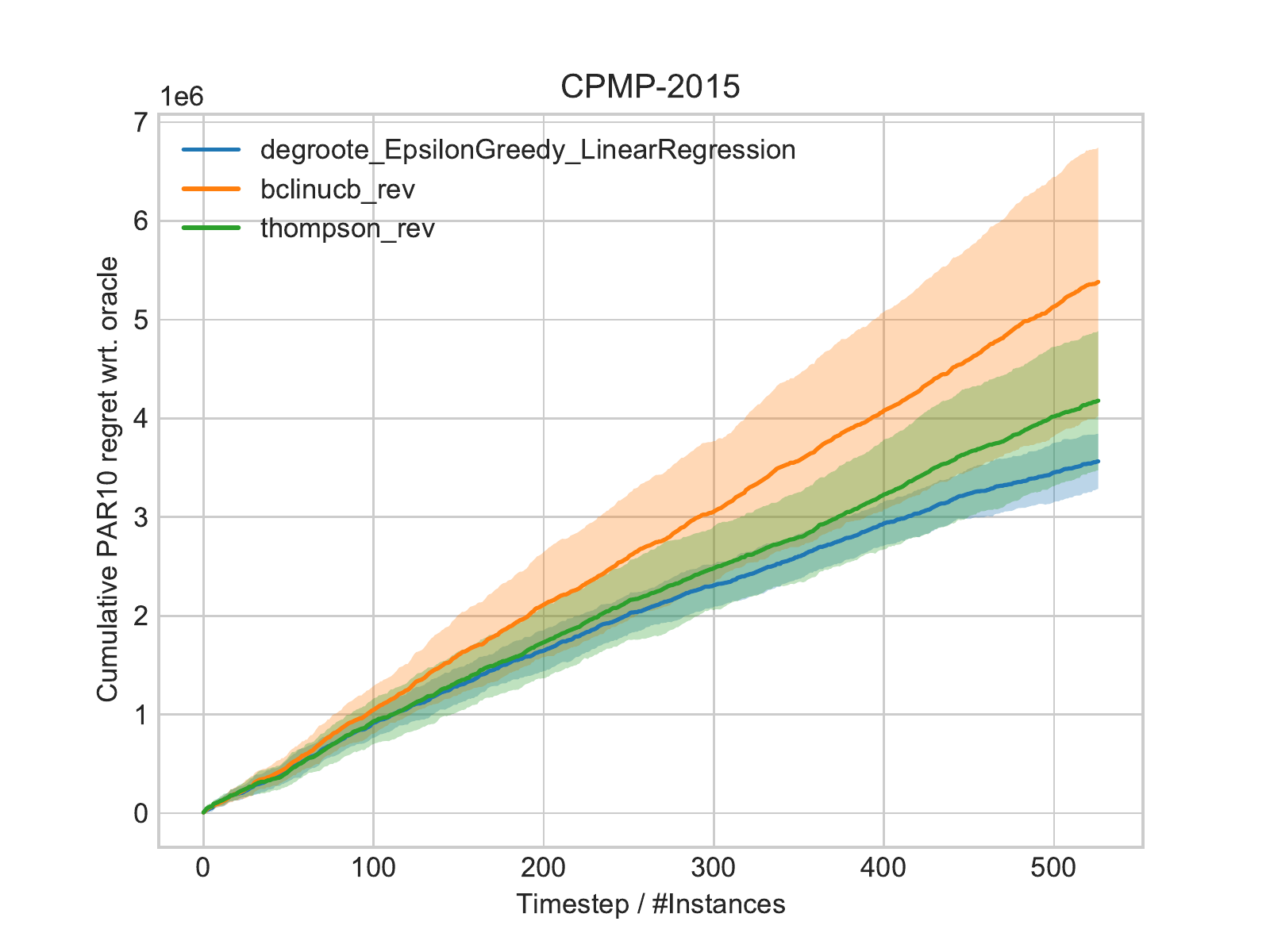}
	 \caption{}
	 \label{fig:app_cumulative_regret_cpmp-2015}
\end{subfigure}
\begin{subfigure}{0.3\textwidth}
	 \includegraphics[width=\linewidth]{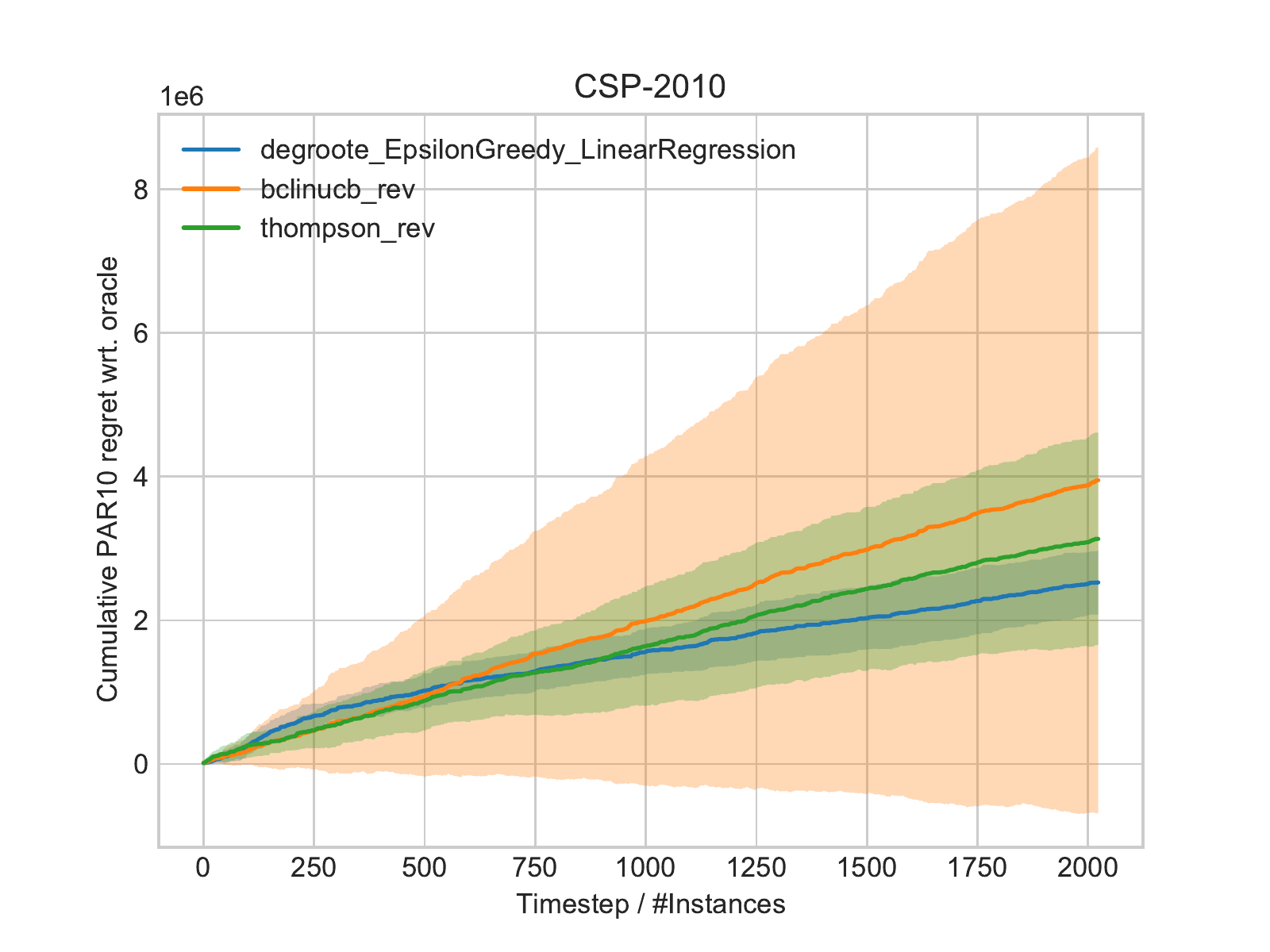}
	 \caption{}
	 \label{fig:app_cumulative_regret_csp-2010}
\end{subfigure}
\begin{subfigure}{0.3\textwidth}
	 \includegraphics[width=\linewidth]{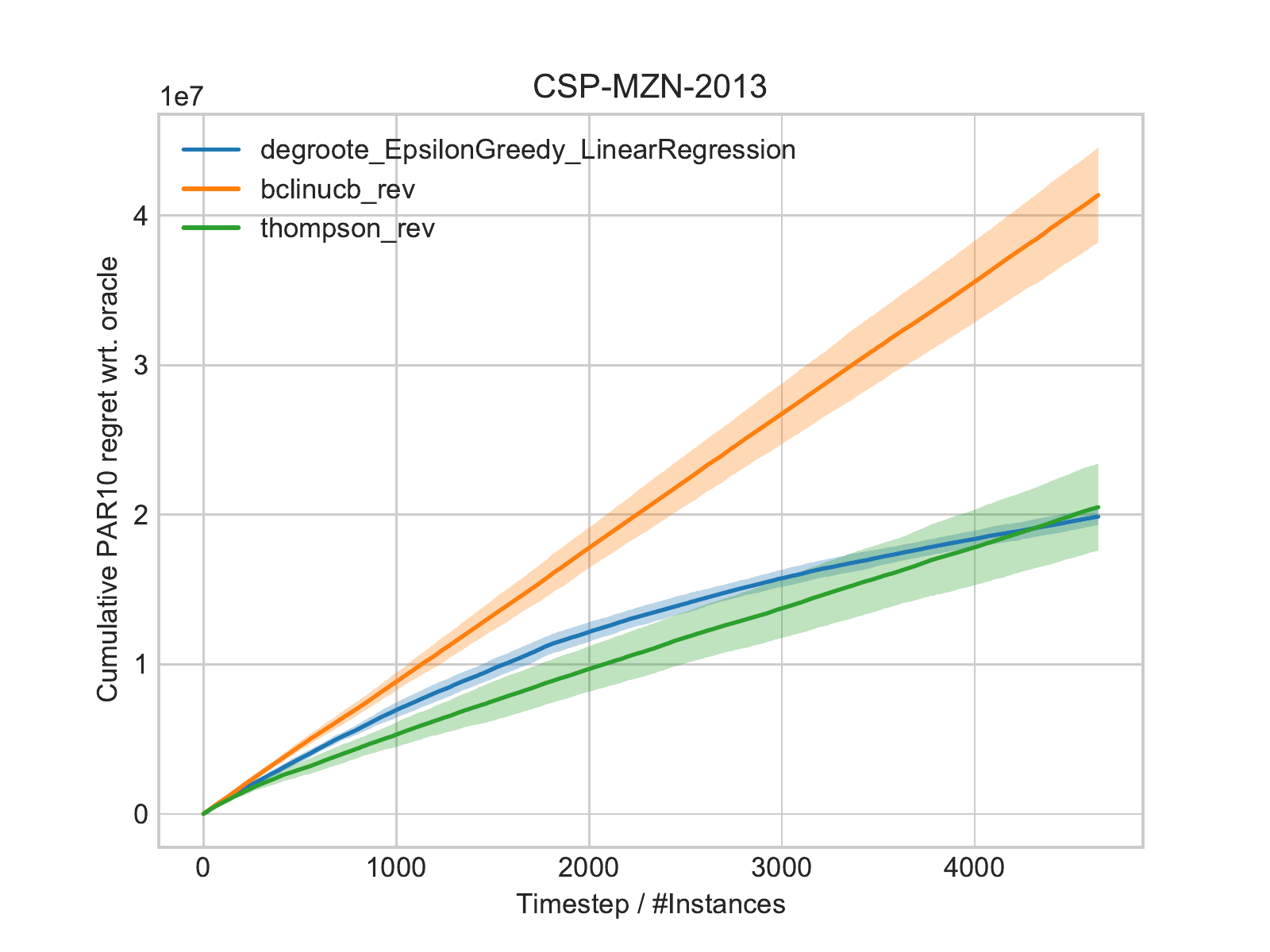}
	 \caption{}
	 \label{fig:app_cumulative_regret_csp-mzn-2013}
\end{subfigure}
\begin{subfigure}{0.3\textwidth}
	 \includegraphics[width=\linewidth]{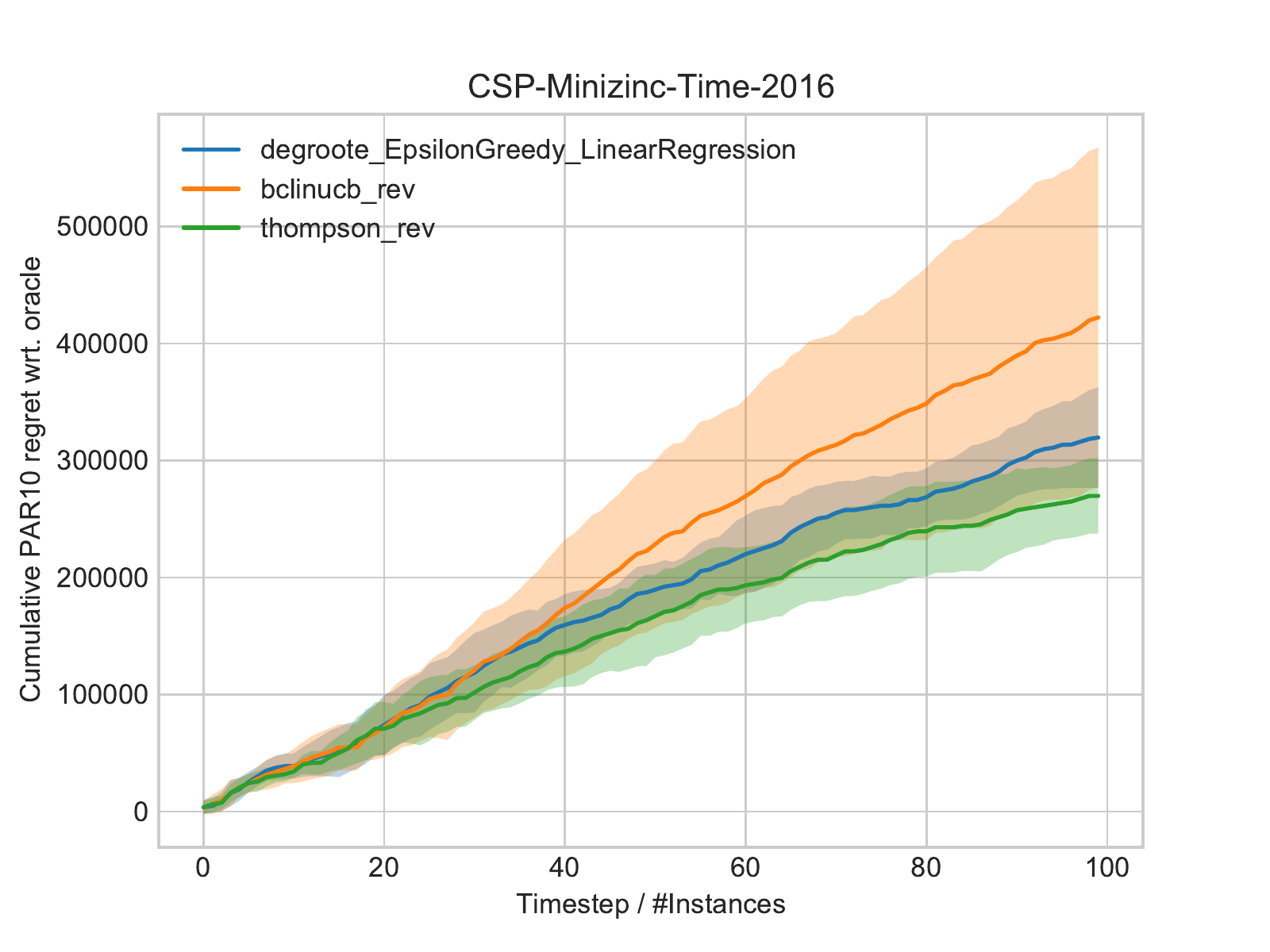}
	 \caption{}
	 \label{fig:app_cumulative_regret_csp-minizinc-time-2016}
\end{subfigure}
\begin{subfigure}{0.3\textwidth}
	 \includegraphics[width=\linewidth]{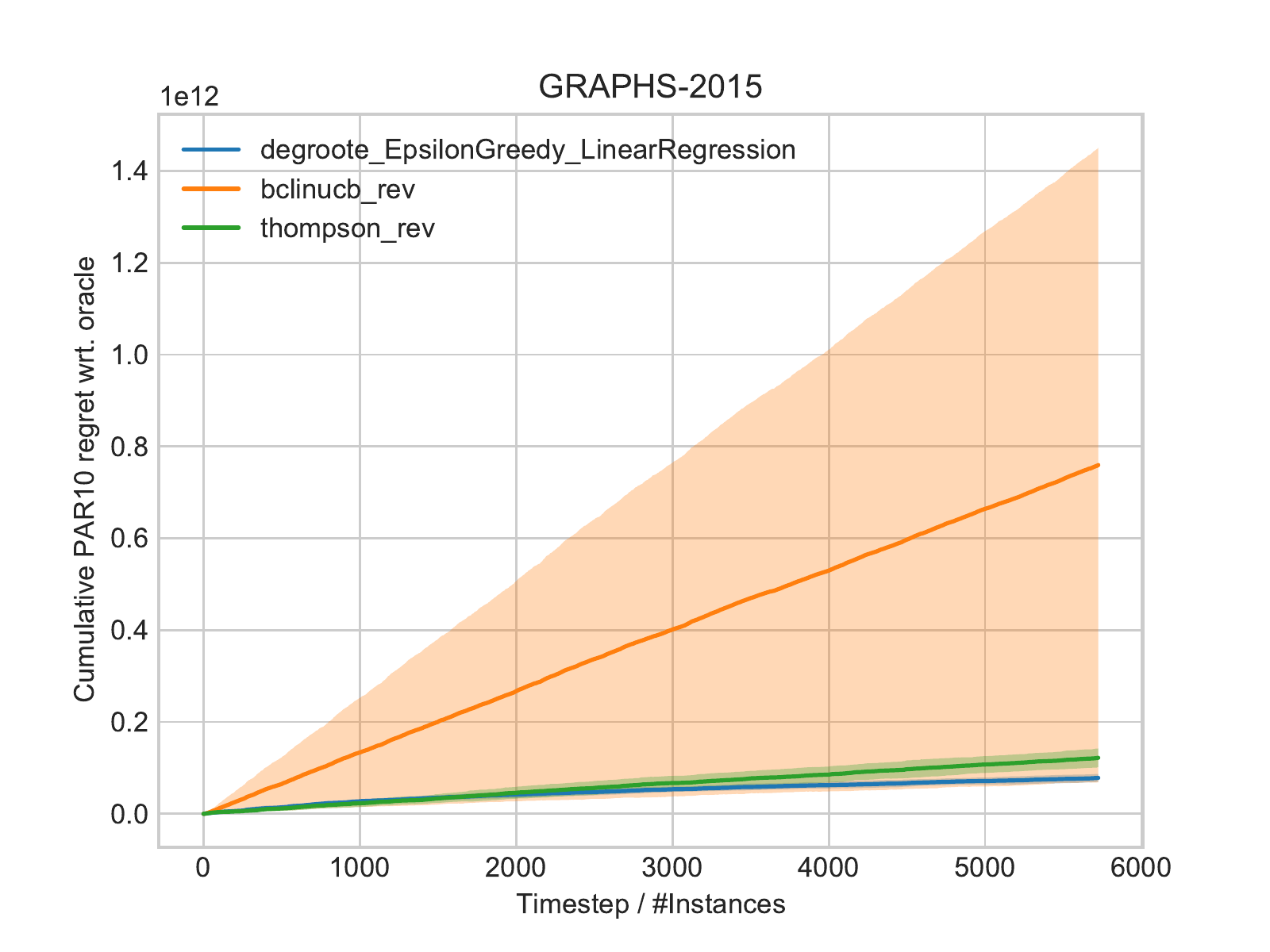}
	 \caption{}
	 \label{fig:app_cumulative_regret_graphs-2015}
\end{subfigure}
\begin{subfigure}{0.3\textwidth}
	 \includegraphics[width=\linewidth]{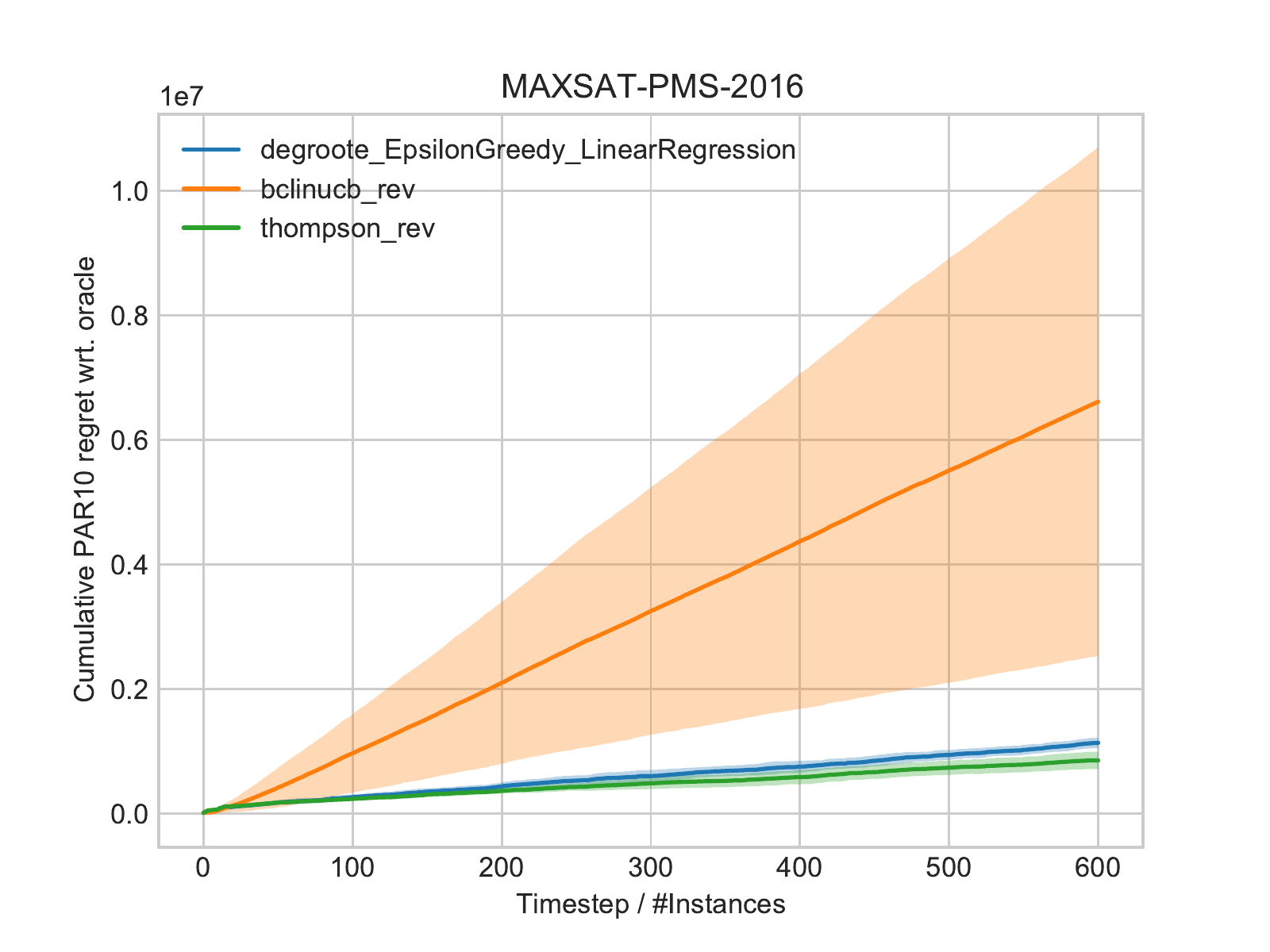}
	 \caption{}
	 \label{fig:app_cumulative_regret_maxsat-pms-2016}
\end{subfigure}
\begin{subfigure}{0.3\textwidth}
	 \includegraphics[width=\linewidth]{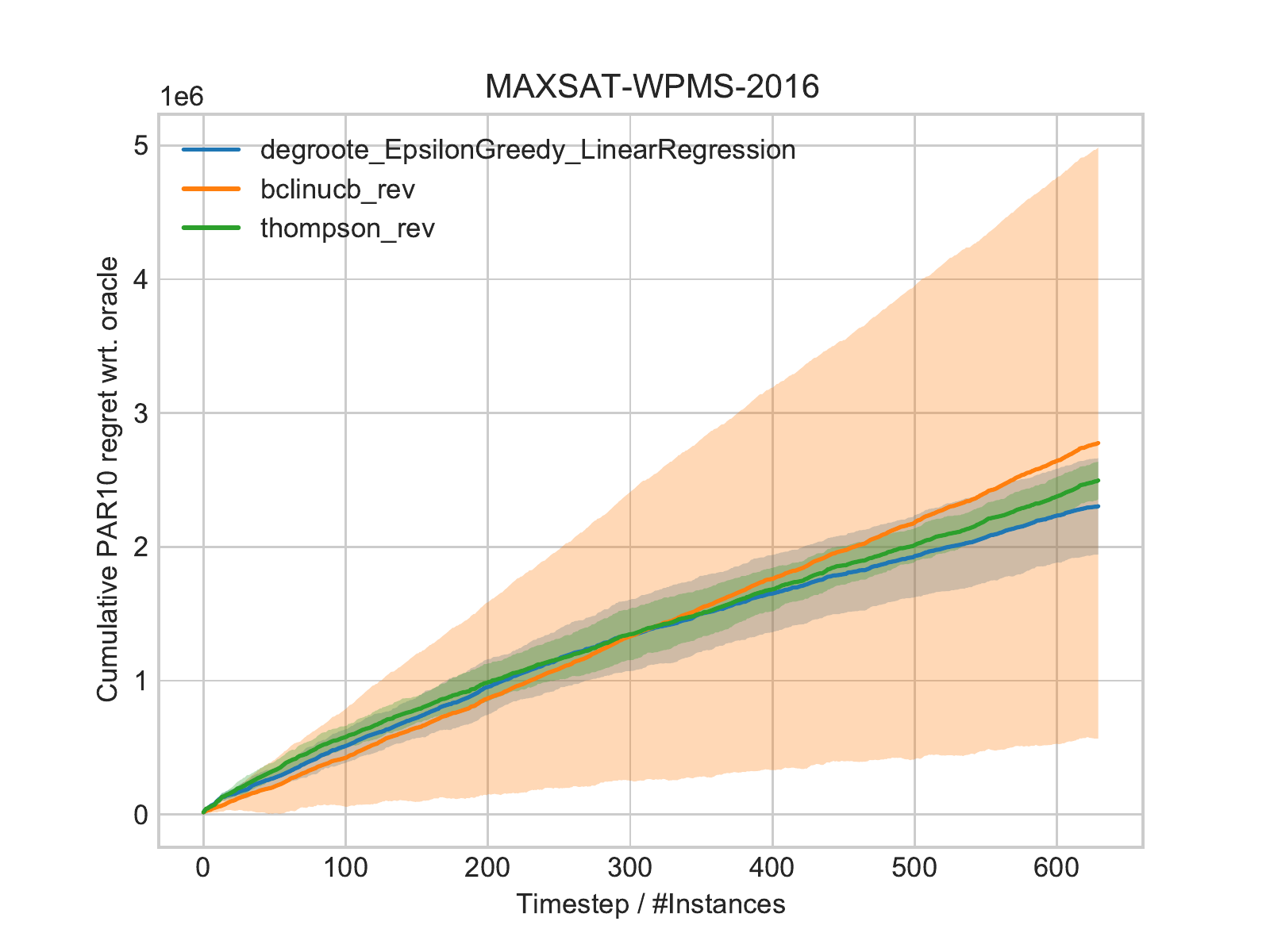}
	 \caption{}
	 \label{fig:app_cumulative_regret_maxsat-wpms-2016}
\end{subfigure}
\begin{subfigure}{0.3\textwidth}
	 \includegraphics[width=\linewidth]{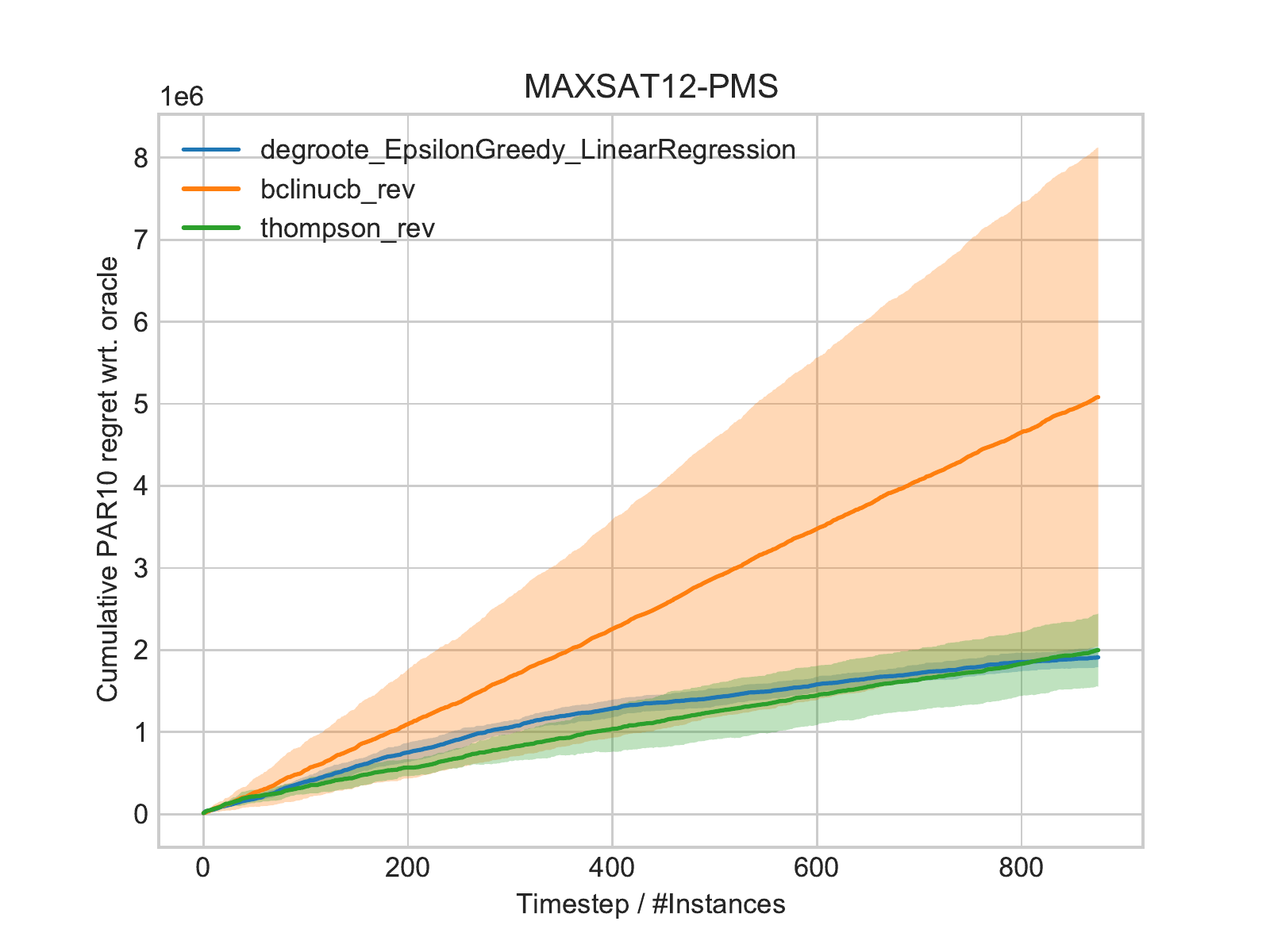}
	 \caption{}
	 \label{fig:app_cumulative_regret_maxsat12-pms}
\end{subfigure}
\begin{subfigure}{0.3\textwidth}
	 \includegraphics[width=\linewidth]{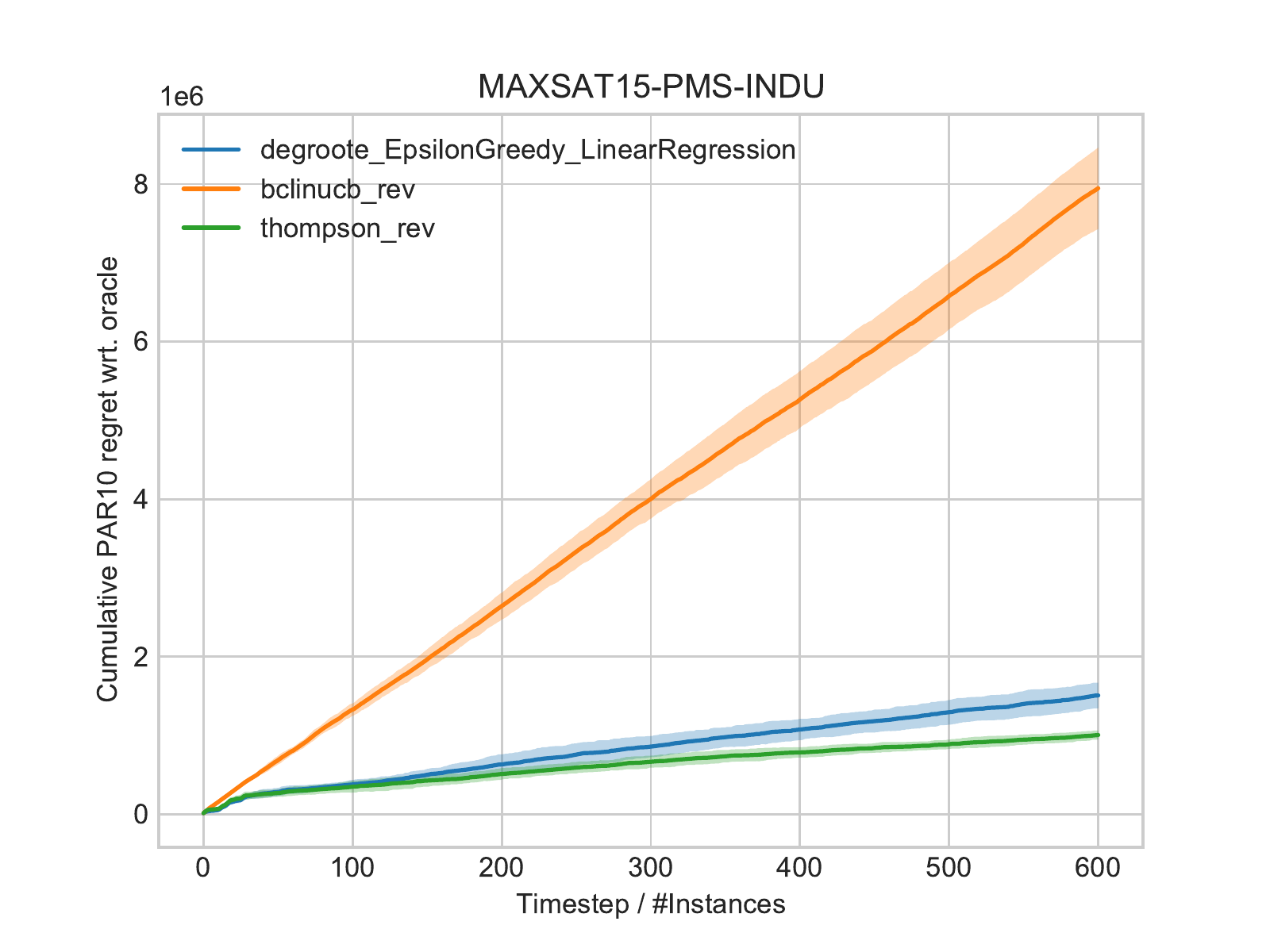}
	 \caption{}
	 \label{fig:app_cumulative_regret_maxsat15-pms-indu}
\end{subfigure}
\begin{subfigure}{0.3\textwidth}
	 \includegraphics[width=\linewidth]{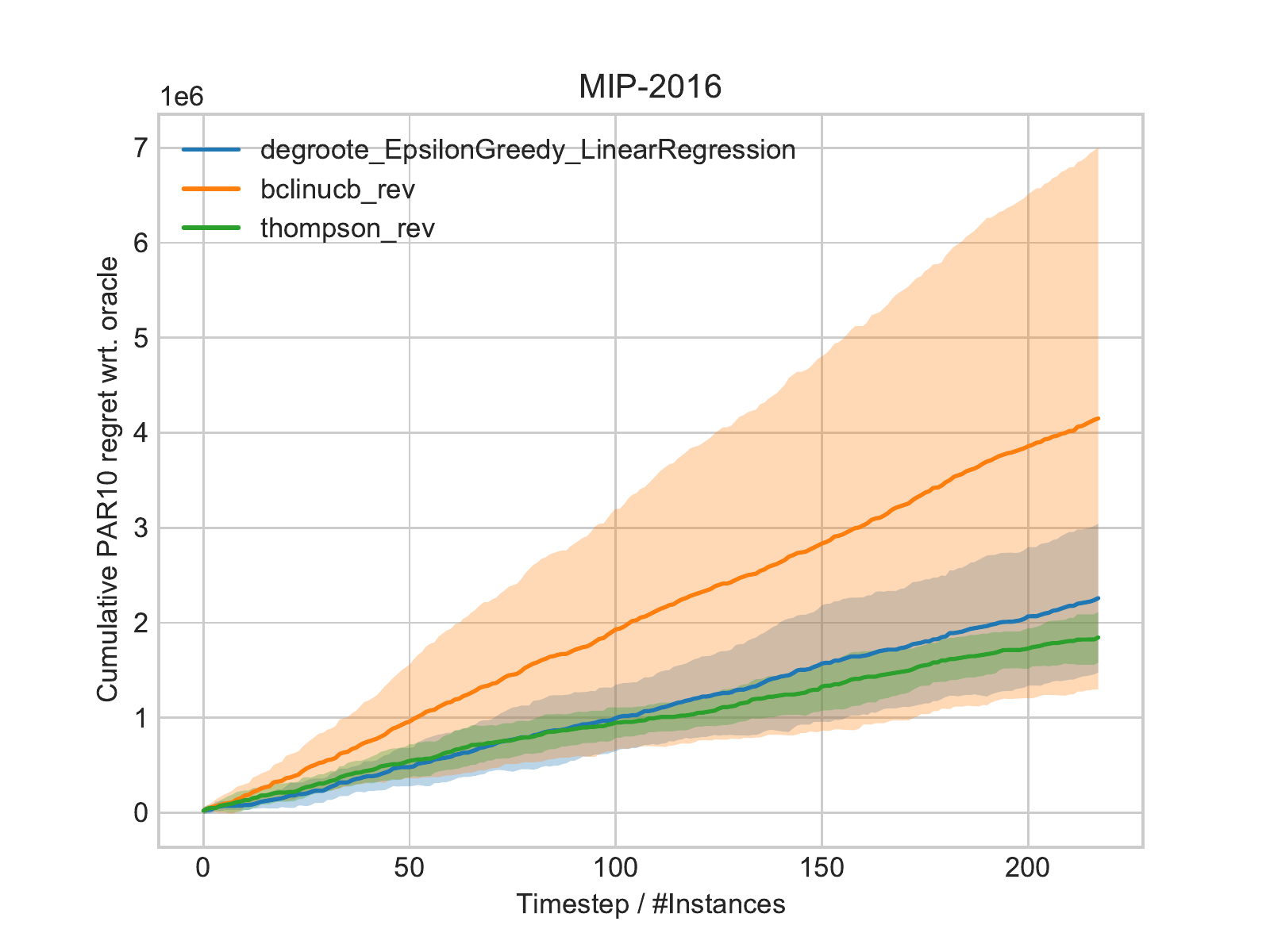}
	 \caption{}
	 \label{fig:app_cumulative_regret_mip-2016}
\end{subfigure}
\begin{subfigure}{0.3\textwidth}
	 \includegraphics[width=\linewidth]{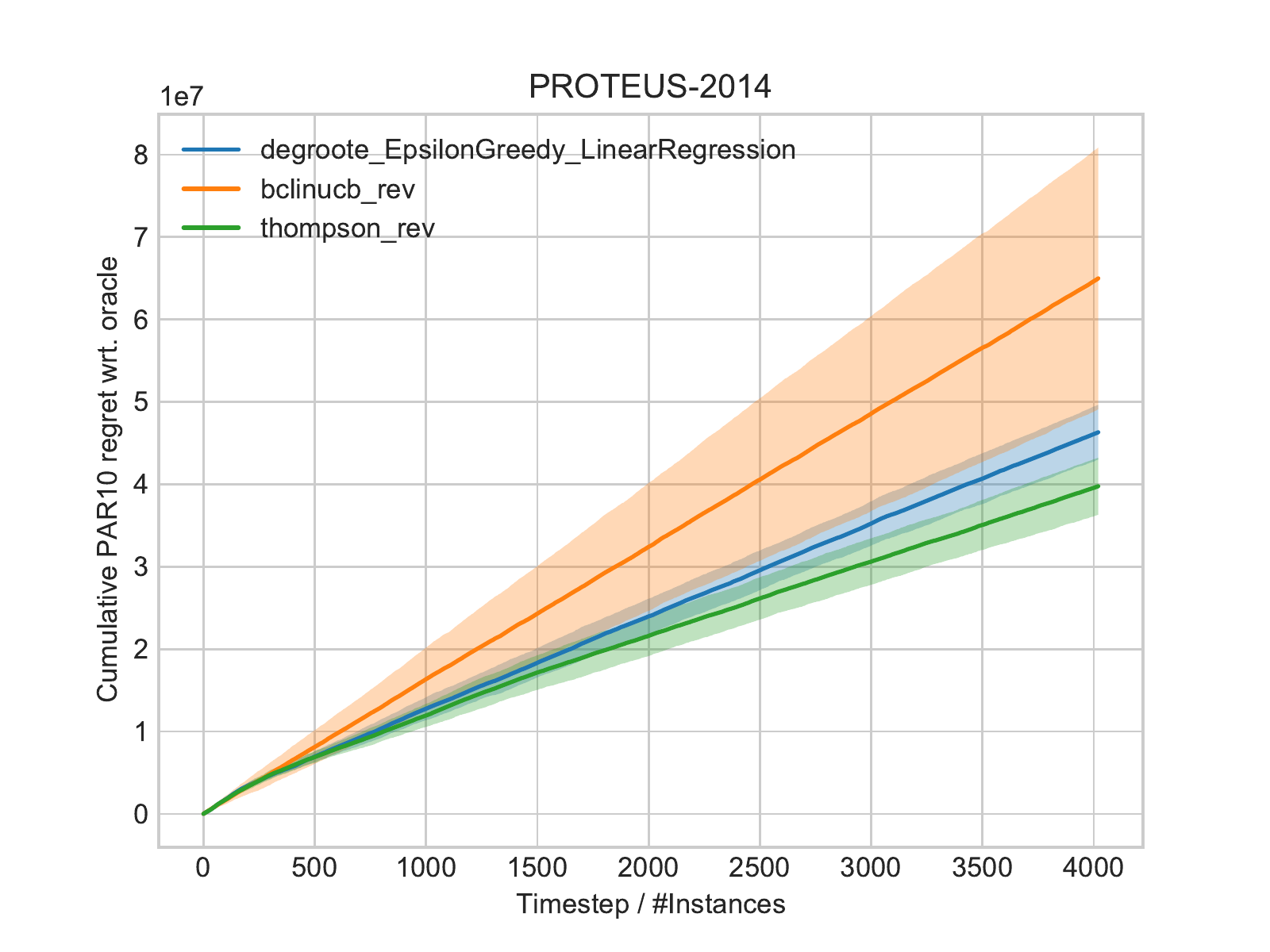}
	 \caption{}
	 \label{fig:app_cumulative_regret_proteus-2014}
\end{subfigure}
\begin{subfigure}{0.3\textwidth}
	 \includegraphics[width=\linewidth]{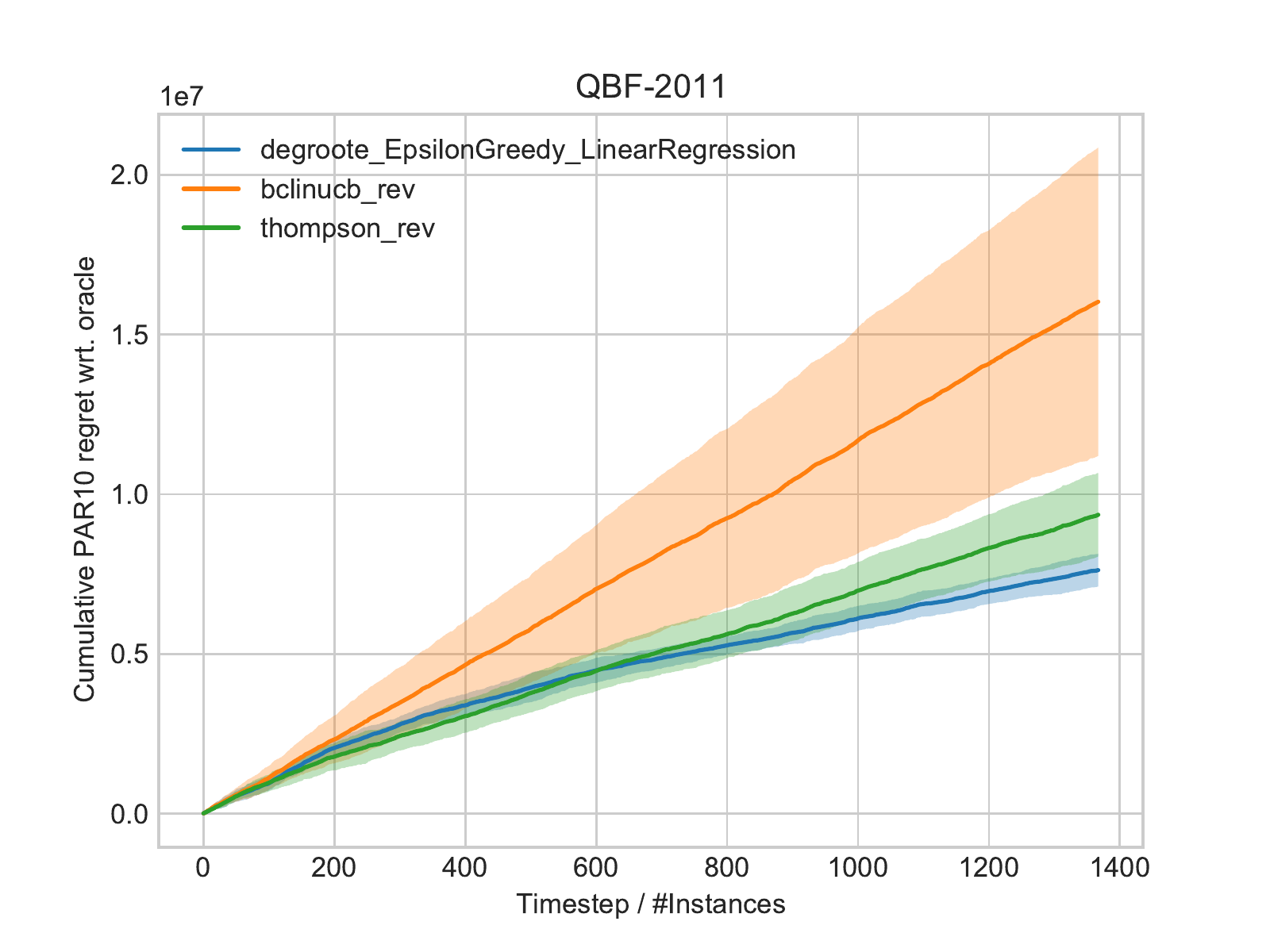}
	 \caption{}
	 \label{fig:app_cumulative_regret_qbf-2011}
\end{subfigure}
\begin{subfigure}{0.3\textwidth}
	 \includegraphics[width=\linewidth]{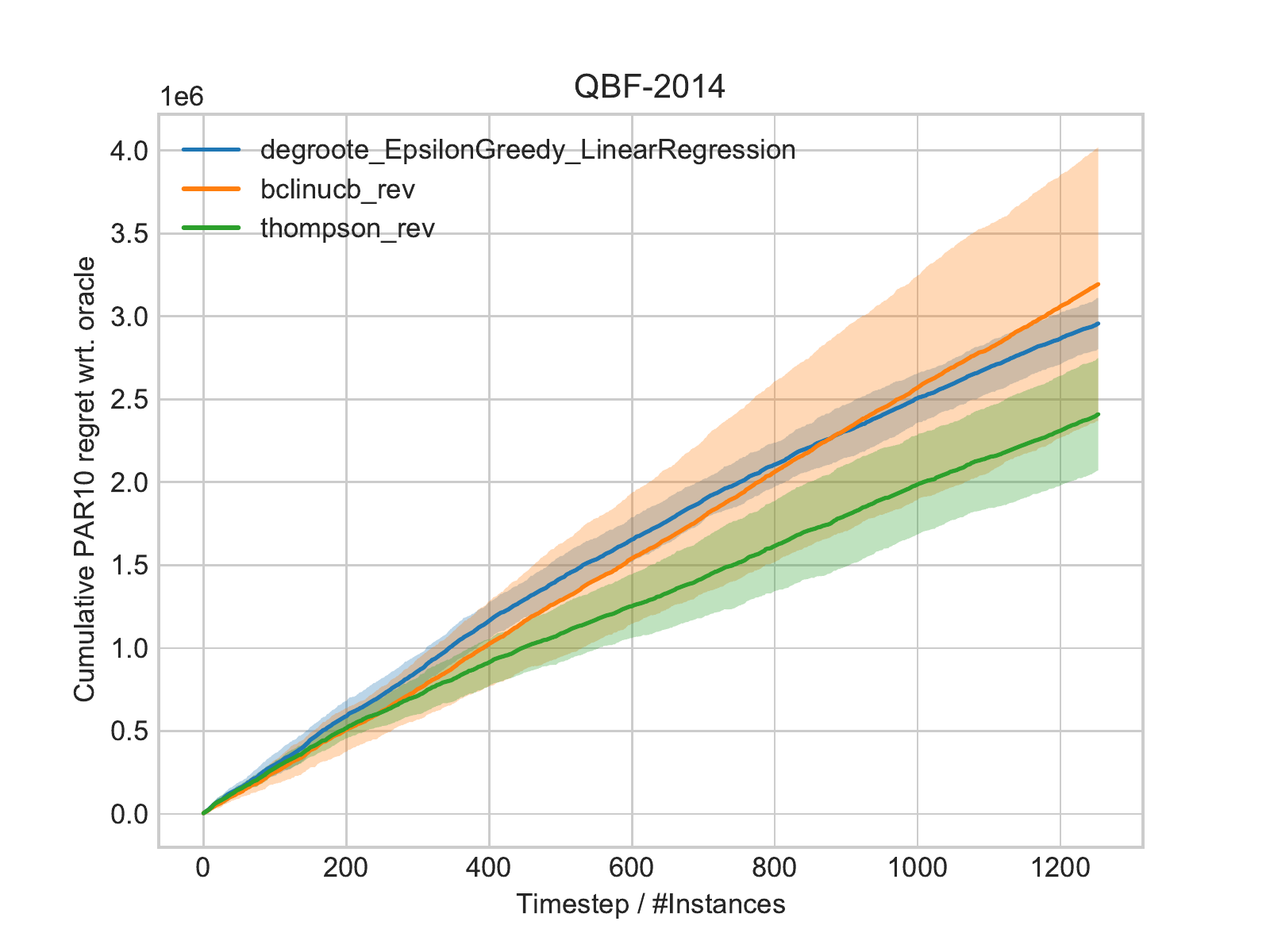}
	 \caption{}
	 \label{fig:app_cumulative_regret_qbf-2014}
\end{subfigure}
\caption{Cumulative PAR10 regret wrt. oracle.}
\label{fig:app_cumulative_regret}
\end{figure}
\begin{figure}[htb]
	\ContinuedFloat
	\centering
\begin{subfigure}{0.3\textwidth}
	 \includegraphics[width=\linewidth]{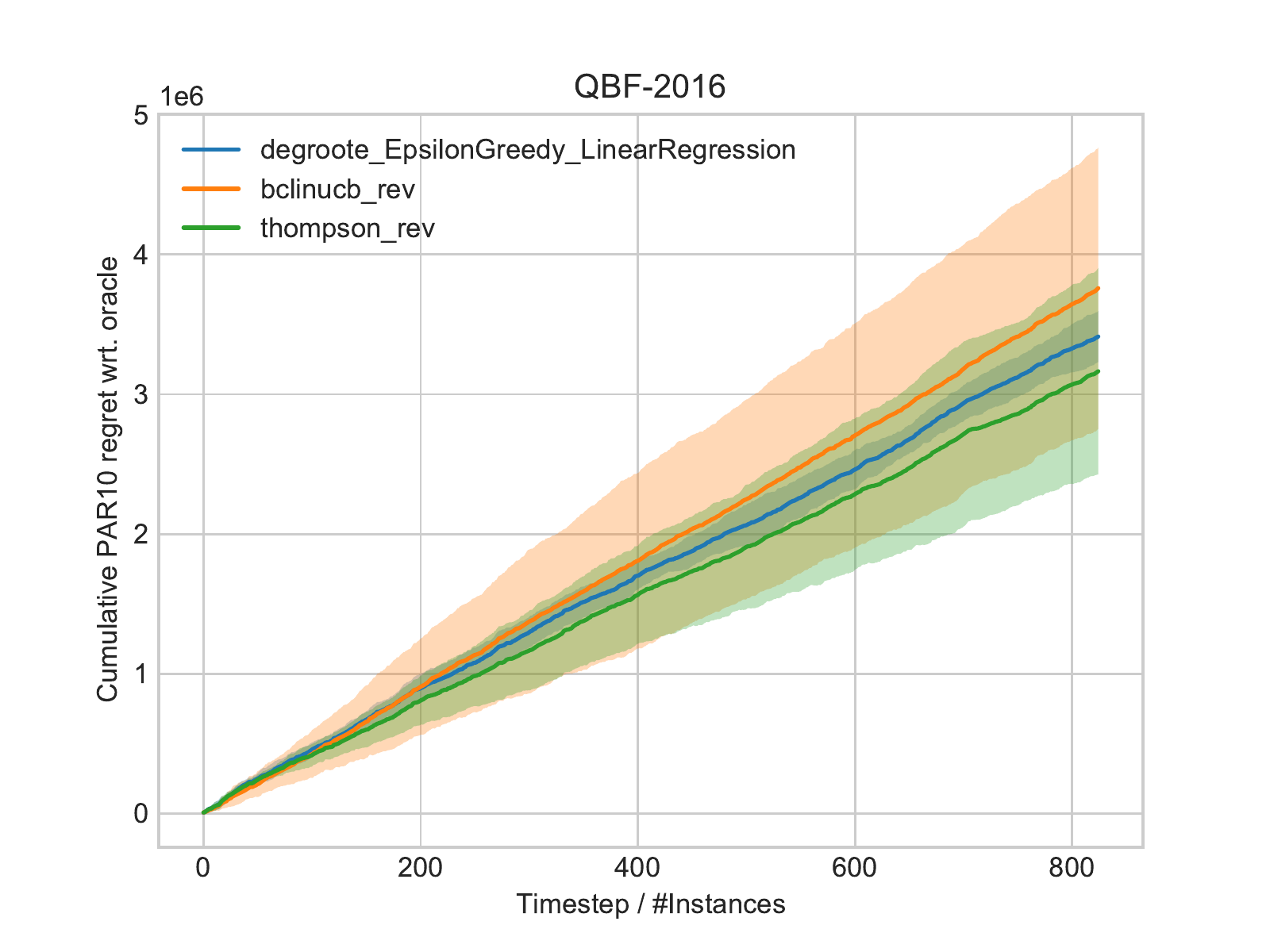}
	 \caption{}
	 \label{fig:app_cumulative_regret_qbf-2016}
\end{subfigure}
\begin{subfigure}{0.3\textwidth}
	 \includegraphics[width=\linewidth]{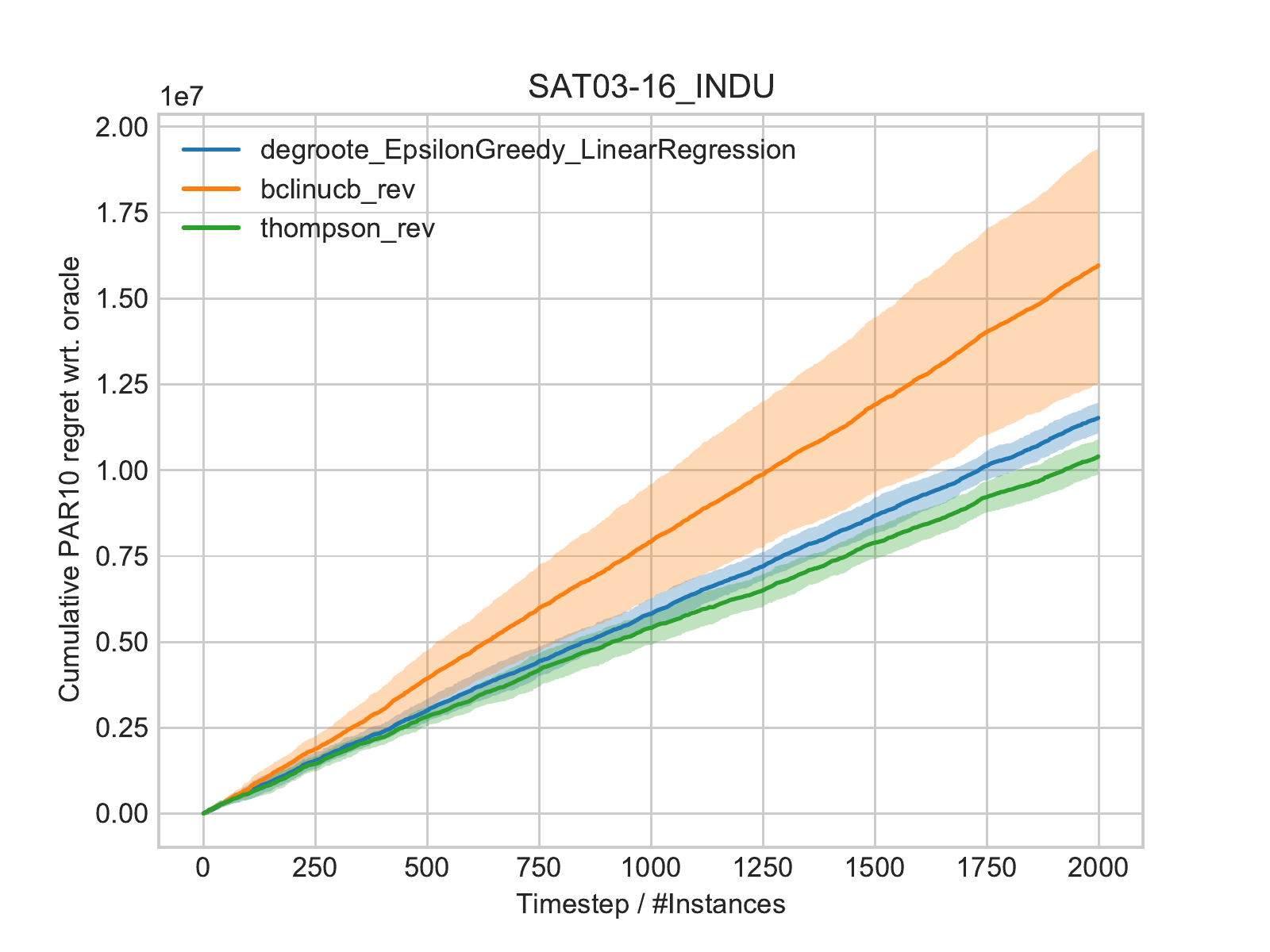}
	 \caption{}
	 \label{fig:app_cumulative_regret_sat03-16_indu}
\end{subfigure}
\begin{subfigure}{0.3\textwidth}
	 \includegraphics[width=\linewidth]{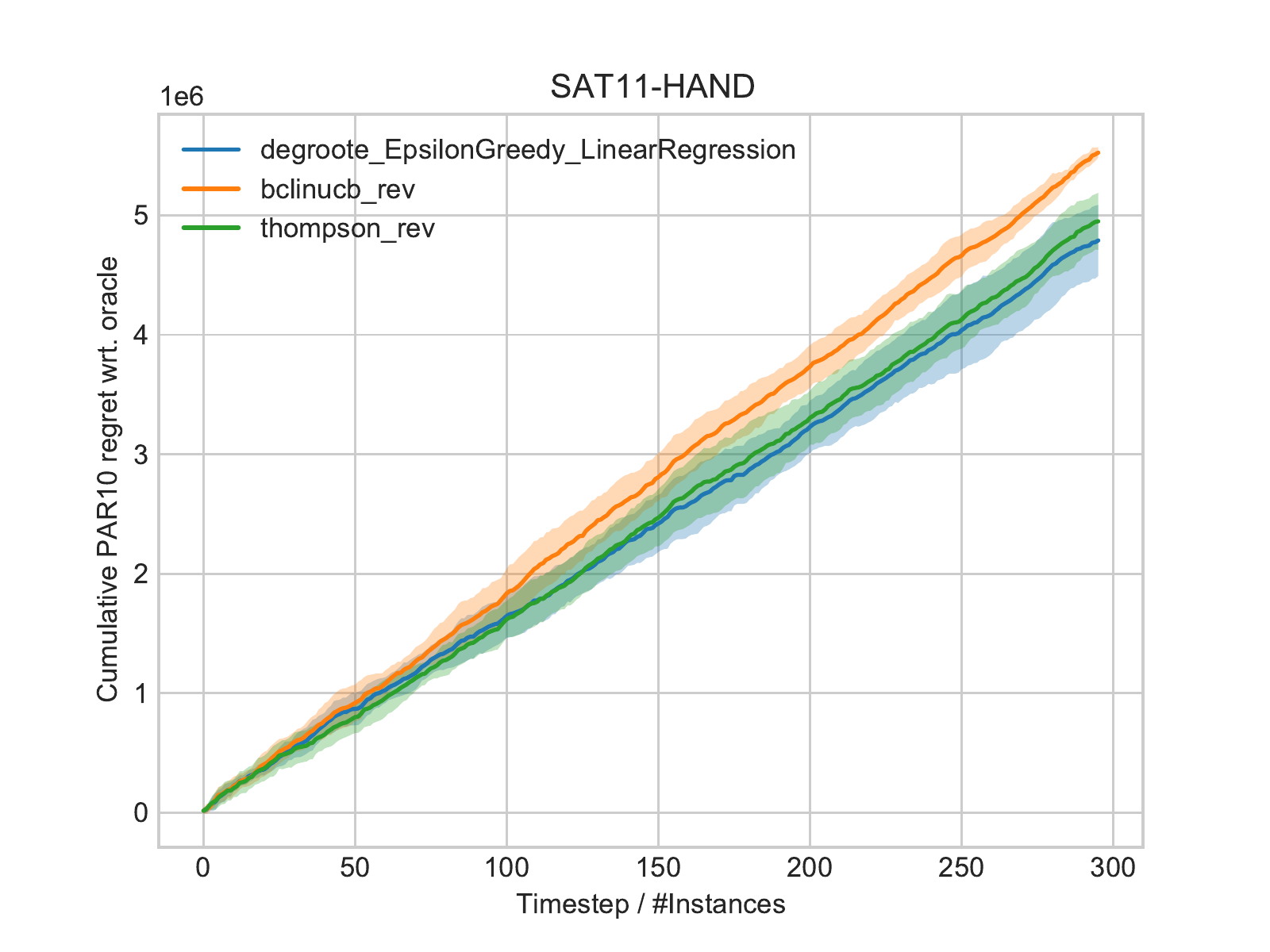}
	 \caption{}
	 \label{fig:app_cumulative_regret_sat11-hand}
\end{subfigure}
\begin{subfigure}{0.3\textwidth}
	 \includegraphics[width=\linewidth]{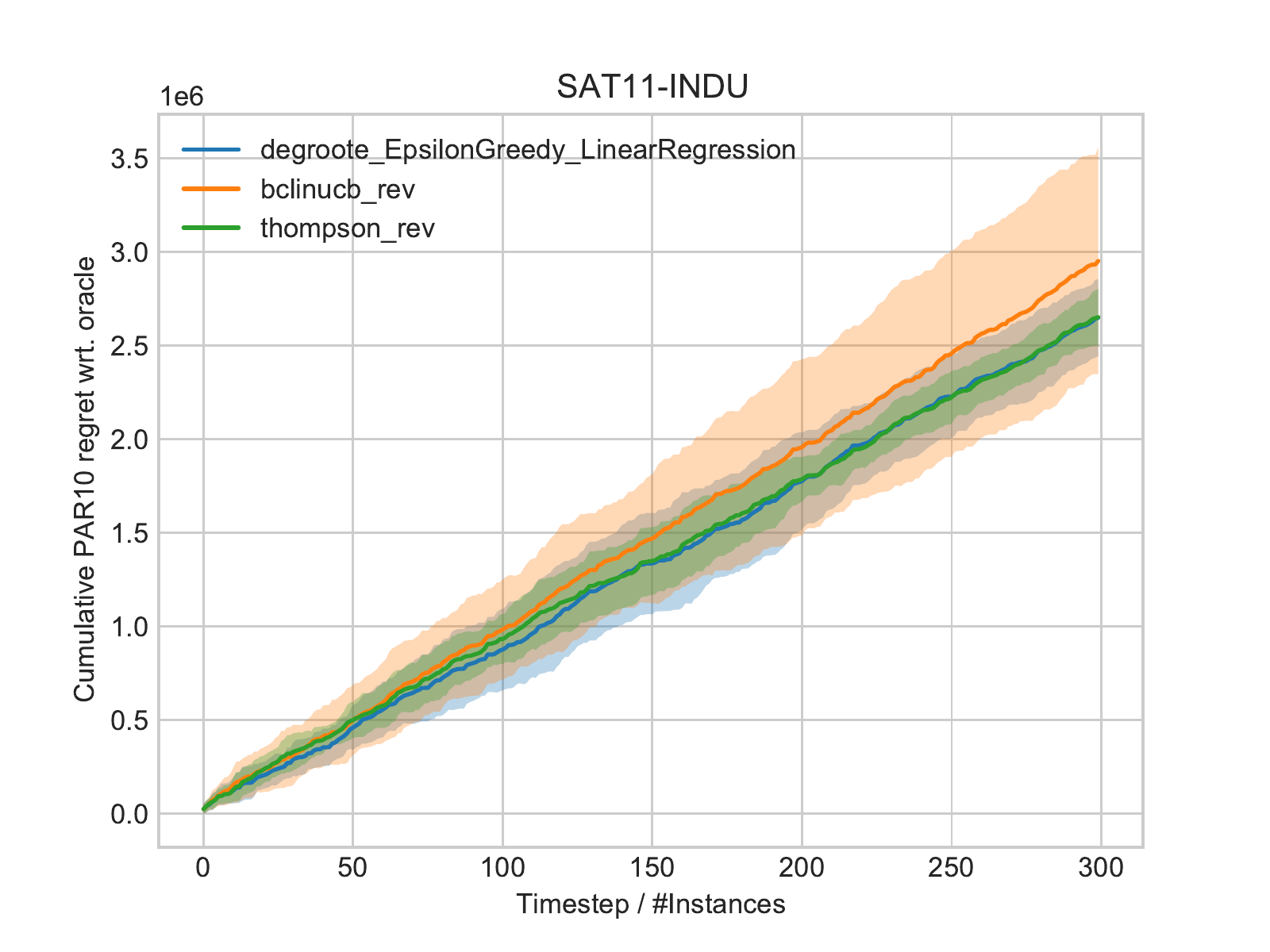}
	 \caption{}
	 \label{fig:app_cumulative_regret_sat11-indu}
\end{subfigure}
\begin{subfigure}{0.3\textwidth}
	 \includegraphics[width=\linewidth]{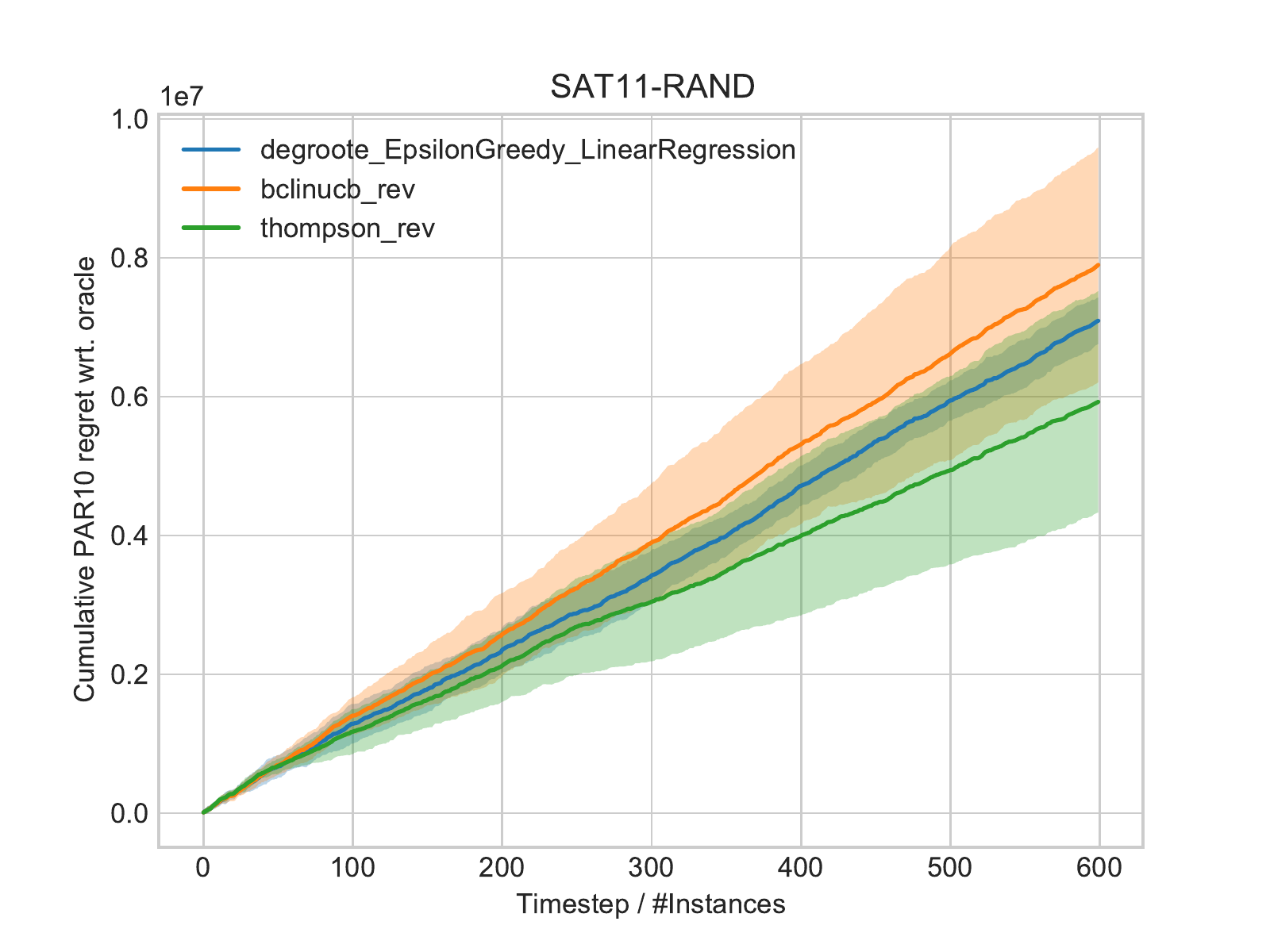}
	 \caption{}
	 \label{fig:app_cumulative_regret_sat11-rand}
\end{subfigure}
\begin{subfigure}{0.3\textwidth}
	 \includegraphics[width=\linewidth]{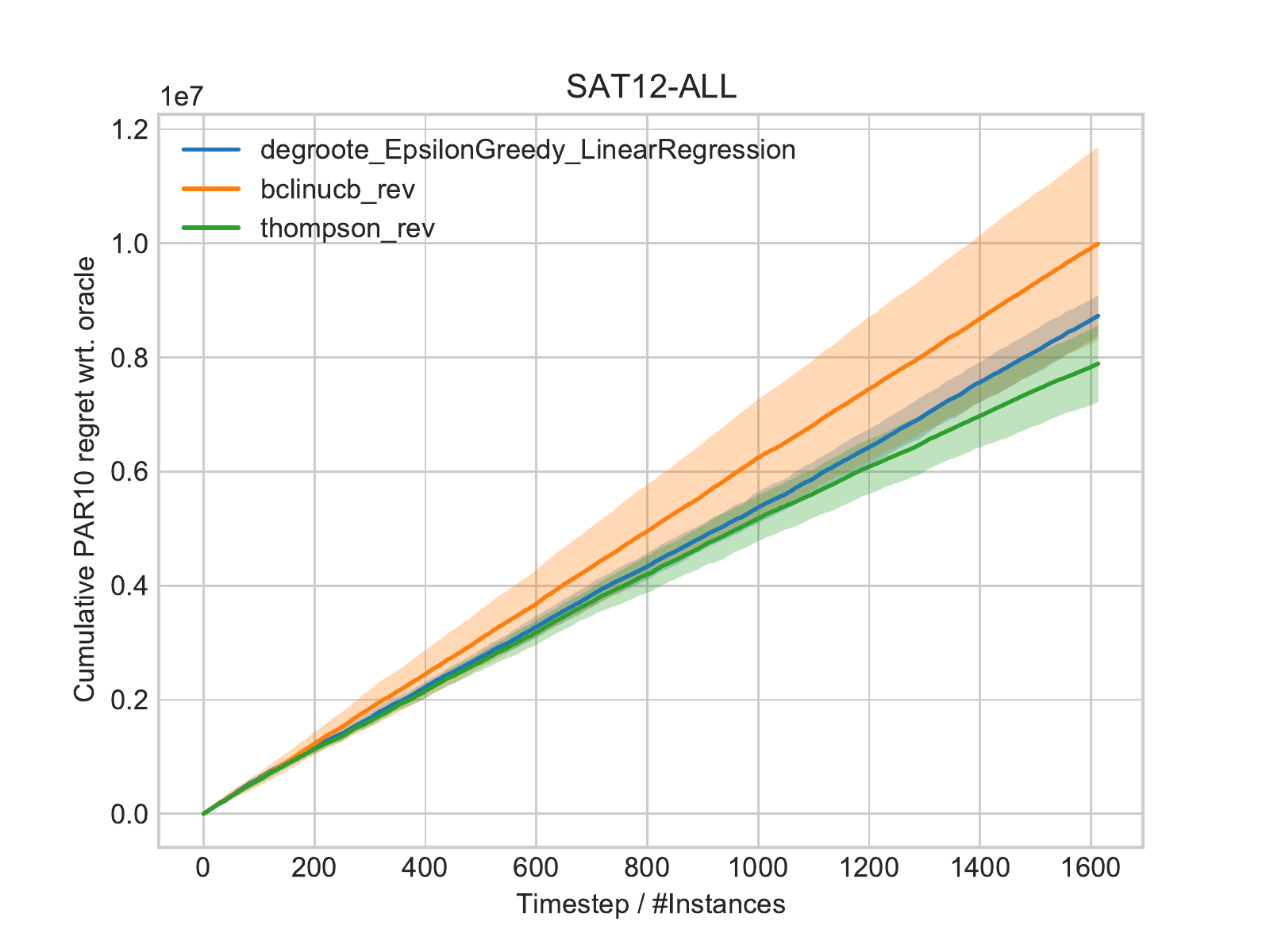}
	 \caption{}
	 \label{fig:app_cumulative_regret_sat12-all}
\end{subfigure}
\begin{subfigure}{0.3\textwidth}
	 \includegraphics[width=\linewidth]{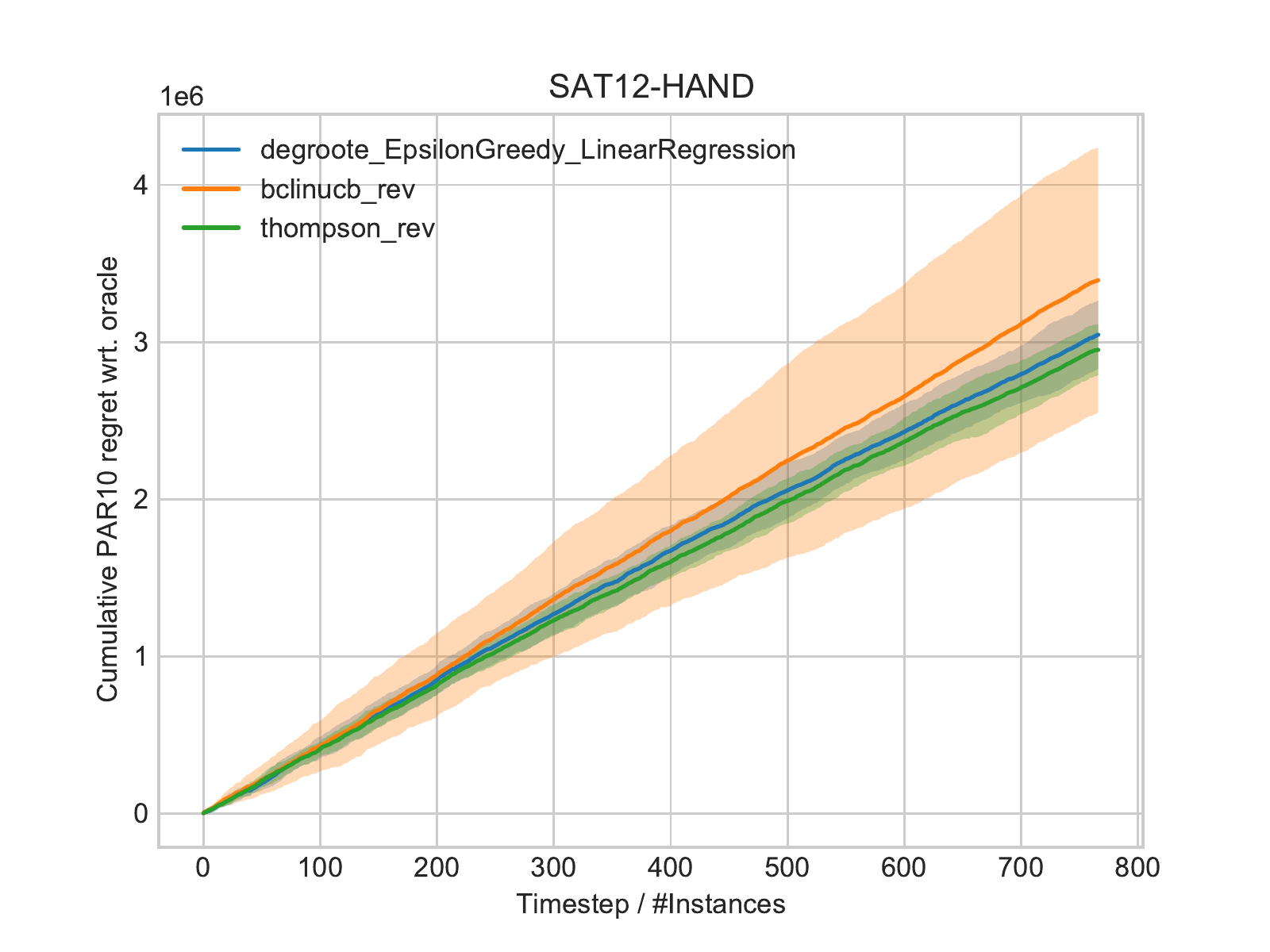}
	 \caption{}
	 \label{fig:app_cumulative_regret_sat12-hand}
\end{subfigure}
\begin{subfigure}{0.3\textwidth}
	 \includegraphics[width=\linewidth]{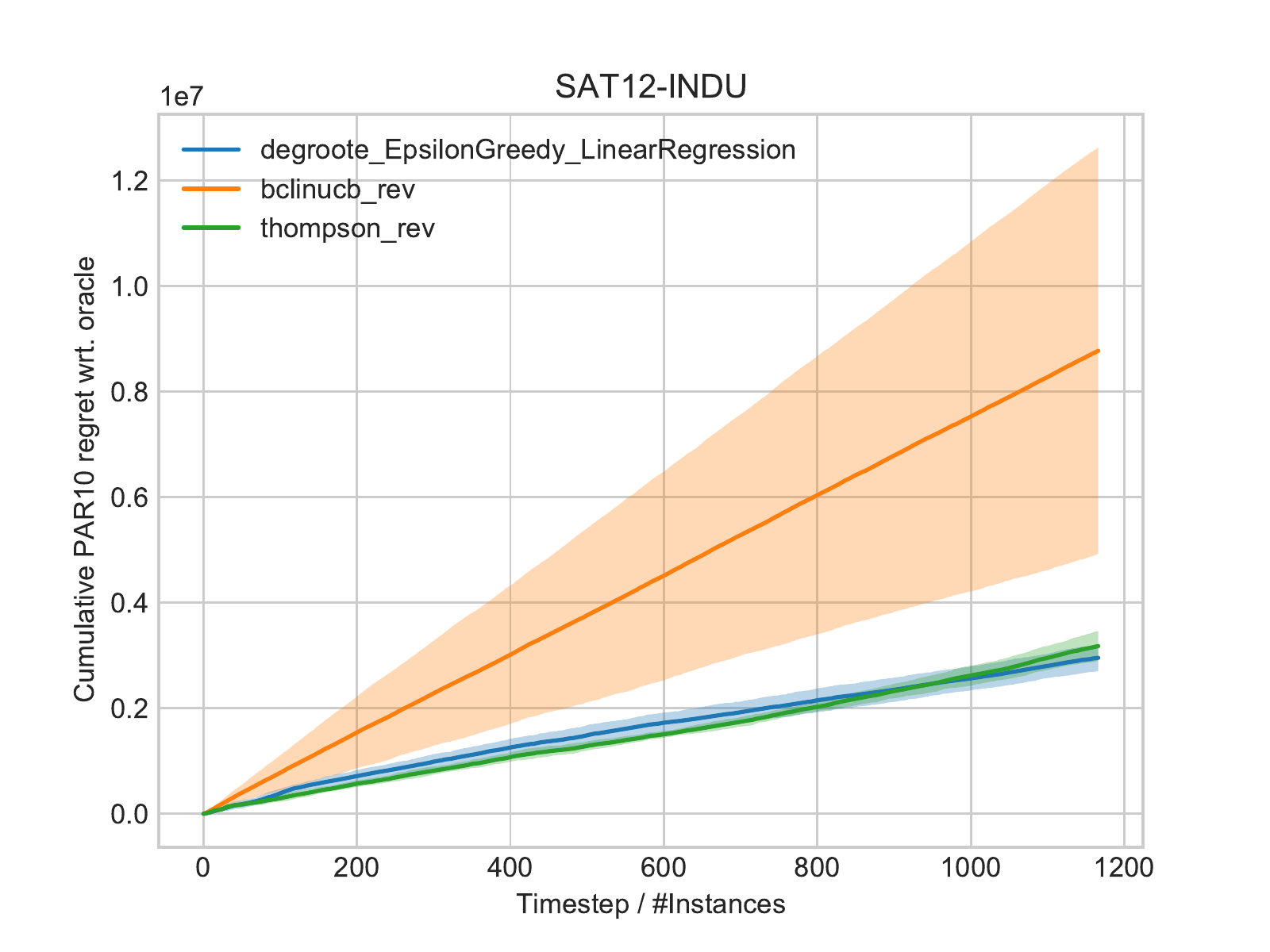}
	 \caption{}
	 \label{fig:app_cumulative_regret_sat12-indu}
\end{subfigure}
\begin{subfigure}{0.3\textwidth}
	 \includegraphics[width=\linewidth]{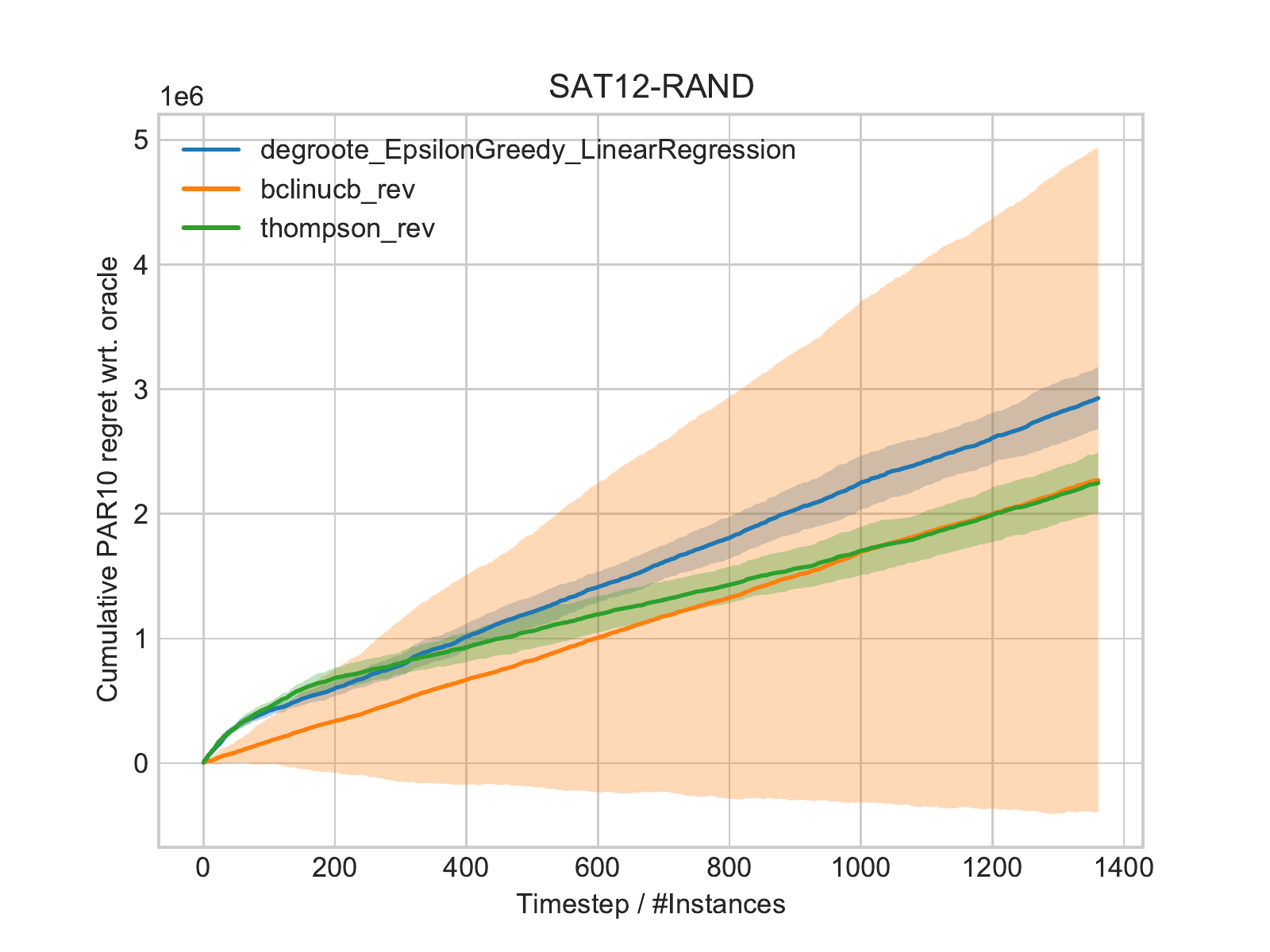}
	 \caption{}
	 \label{fig:app_cumulative_regret_sat12-rand}
\end{subfigure}
\begin{subfigure}{0.3\textwidth}
	 \includegraphics[width=\linewidth]{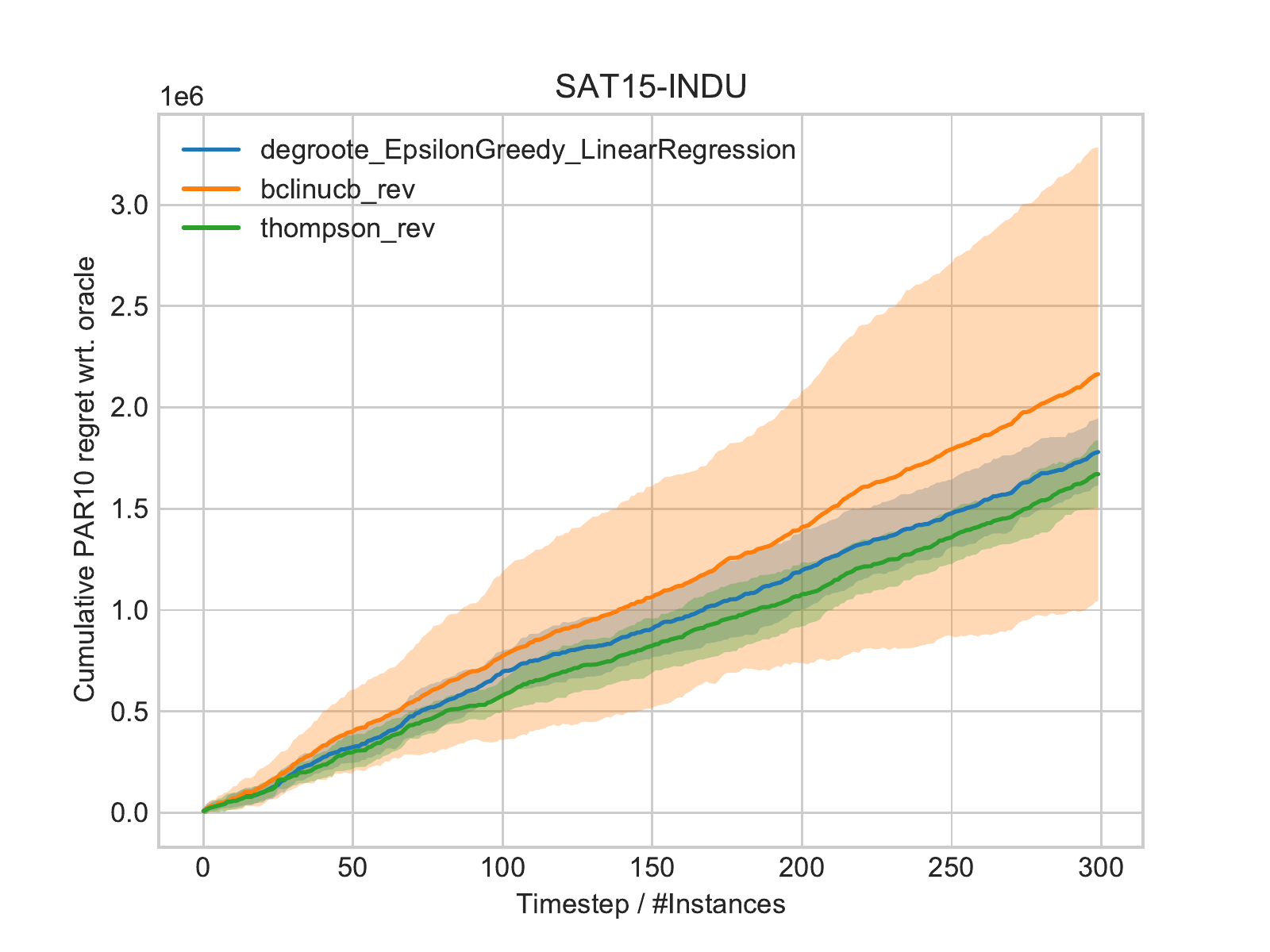}
	 \caption{}
	 \label{fig:app_cumulative_regret_sat15-indu}
\end{subfigure}
\begin{subfigure}{0.3\textwidth}
	 \includegraphics[width=\linewidth]{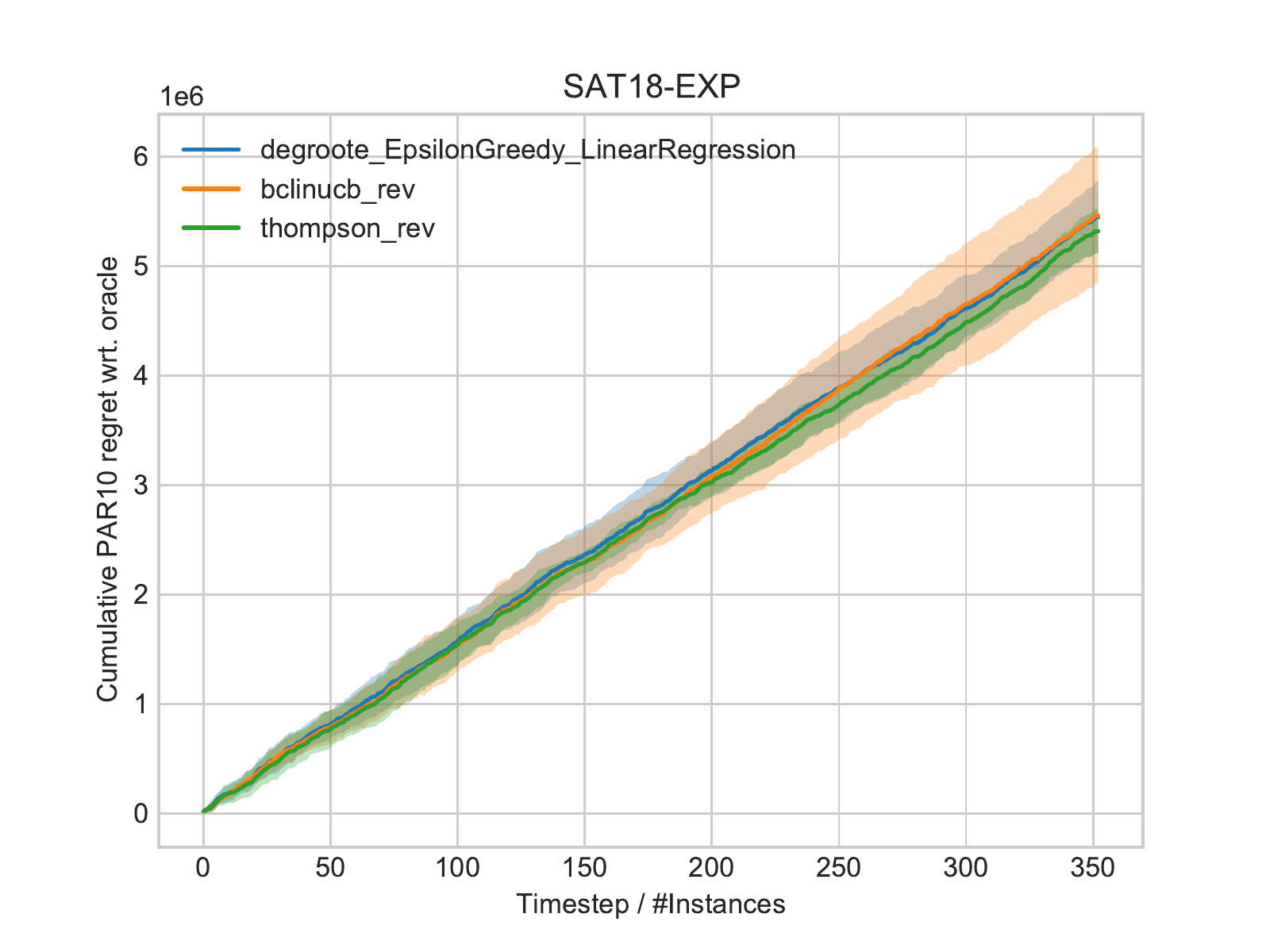}
	 \caption{}
	 \label{fig:app_cumulative_regret_sat18-exp}
\end{subfigure}
\begin{subfigure}{0.3\textwidth}
	 \includegraphics[width=\linewidth]{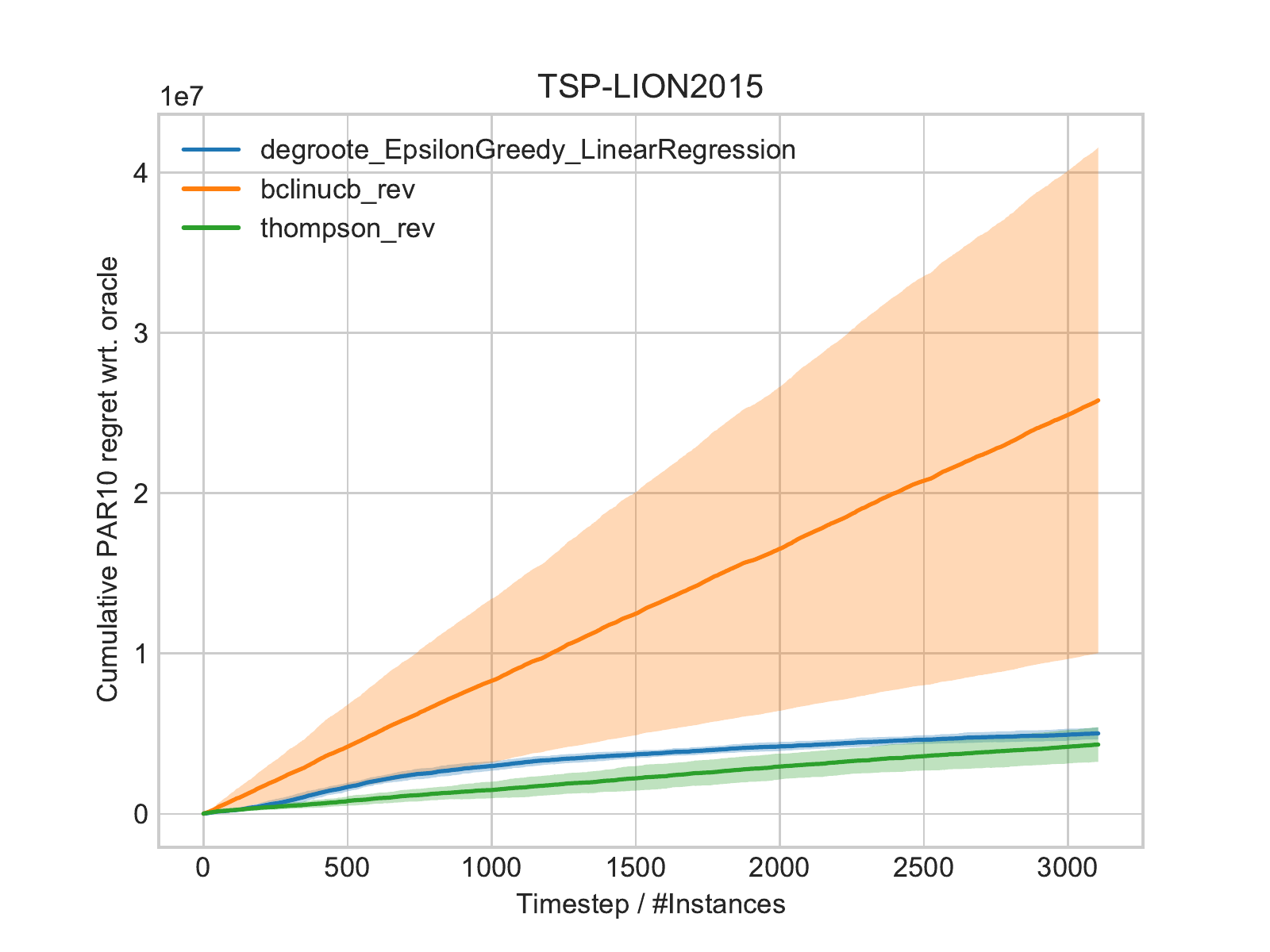}
	 \caption{}
	 \label{fig:app_cumulative_regret_tsp-lion2015}
\end{subfigure}
\caption{(Cont.) Cumulative PAR10 regret wrt. oracle.}
\label{fig:app_cumulative_regret_1}
\end{figure}

\end{document}